\documentclass{article}

\usepackage{microtype}
\usepackage{graphicx}
\usepackage{booktabs} 
\usepackage{tikz}
\usepackage{hyperref,graphicx,fullpage,amsmath,amsfonts,subcaption,amssymb,bm,url,breakurl,epsfig,epsf,color,MnSymbol,mathbbol,fmtcount,semtrans,caption,multirow,comment}

\usepackage{color}
\usepackage[utf8]{inputenc} 
\usepackage[T1]{fontenc}    
\usepackage{hyperref}       
\usepackage{url}            
\usepackage{booktabs}       
\usepackage{amsfonts}       
\usepackage{nicefrac}       
\usepackage{microtype}      
\usepackage{hyperref,graphicx,amsmath,amsfonts,amssymb,bm,url,breakurl,epsfig,epsf,color,MnSymbol,mathbbol,fmtcount,semtrans,caption,subcaption,multirow,comment, boldline}
\usepackage[noend]{algpseudocode}
\usepackage{wrapfig}
\usepackage{amssymb}

\usepackage[utf8]{inputenc} 
\usepackage[T1]{fontenc}    
\usepackage{url}            
\usepackage{booktabs}       
\usepackage{amsfonts}       
\usepackage{nicefrac}       
\usepackage{microtype}      

\makeatletter
\providecommand*{\boxast}{%
  \mathbin{
    \mathpalette\@boxit{*}%
  }%
}
\newcommand*{\@boxit}[2]{%
  \sbox0{$\m@th#1\Box$}%
  \ifx#1\displaystyle \ht0=\dimexpr\ht0+.05ex\relax \fi
  \ifx#1\textstyle \ht0=\dimexpr\ht0+.05ex\relax \fi
  \ifx#1\scriptstyle \ht0=\dimexpr\ht0+.04ex\relax \fi
  \ifx#1\scriptscriptstyle \ht0=\dimexpr\ht0+.065ex\relax \fi
  \sbox2{$#1\vcenter{}$}
  \rlap{%
    \hbox to \wd0{%
      \hfill
      \raisebox{%
        \dimexpr.5\dimexpr\ht0+\dp0\relax-\ht2\relax
      }{$\m@th#1#2$}%
      \hfill
    }%
  }%
  \Box
}
\makeatother

  \makeatletter
\def\BState{\State\hskip-\ALG@thistlm}
\makeatother

  \usepackage{mathtools}

\usepackage{titlesec}

\usepackage{tikz}
\usepackage{pgfplots}
\usetikzlibrary{pgfplots.groupplots}

\titleformat{\paragraph}
{\normalfont\normalsize\bfseries}{\theparagraph}{1em}{}
\titlespacing*{\paragraph}
{0pt}{3.25ex plus 1ex minus .2ex}{1.5ex plus .2ex}

\usepackage{movie15}

\usepackage{caption}
\usepackage[bottom,hang,flushmargin]{footmisc} 

\newcommand{\tsn}[1]{{\left\vert\kern-0.25ex\left\vert\kern-0.25ex\left\vert #1 
    \right\vert\kern-0.25ex\right\vert\kern-0.25ex\right\vert}}

\definecolor{darkred}{RGB}{150,0,0}
\definecolor{darkgreen}{RGB}{0,150,0}
\definecolor{darkblue}{RGB}{0,0,200}
\hypersetup{colorlinks=true, linkcolor=darkred, citecolor=darkgreen, urlcolor=darkblue}

\newtheorem{theorem}{Theorem}[section]

\newtheorem{assumption}{Assumption}

\newtheorem{lemma}[theorem]{Lemma}
\newtheorem{corollary}[theorem]{Corollary}

\newtheorem{definition}[theorem]{Definition}

\newcommand{\eps}{\varepsilon}

\newcommand{\bn}{\alpha}
\newcommand{\distas}{\overset{\text{i.i.d.}}{\sim}}

\newcommand{\beq}{\begin{equation}}

\newcommand{\eeq}{\end{equation}}

\newcommand{\vrn}{\nu}

\newcommand{\nn}{\nonumber}
\newcommand{\la}{\lambda}
\newcommand{\K}{{K}}
\newcommand{\A}{{\mtx{A}}}

\newcommand{\Ub}{{\mtx{U}}}

\newcommand{\V}{{\mtx{V}}}

\newcommand{\B}{{{\mtx{B}}}}

\newcommand{\Sb}{{{\mtx{S}}}}

\newcommand{\diag}[1]{\text{diag}(#1)}

\newcommand{\Lc}{{\cal{L}}}

\newcommand{\Jc}{{\cal{J}}}
\newcommand{\Jcb}{{\overline{\cal{J}}}}
\newcommand{\Dc}{{\cal{D}}}

\newcommand{\Pb}{{\mtx{P}}}

\newcommand{\Cb}{{\mtx{C}}}

\newcommand{\La}{{\boldsymbol{{\Lambda}}}}

\newcommand{\bSi}{{\boldsymbol{{\Sigma}}}}

\newcommand{\bmu}{{\boldsymbol{{\mu}}}}
\newcommand{\Db}{{\mtx{D}}}

\newcommand{\onebb}{{\mathbf{1}}}
\newcommand{\Iden}{{\mtx{I}}}
\newcommand{\M}{{\mtx{M}}}

\newcommand{\smn}[1]{{s_{\min}(#1)}}
\newcommand{\smx}[1]{{s_{\max}(#1)}}

\newcommand{\tn}[1]{\|{#1}\|_{\ell_2}}

\newcommand{\tin}[1]{\|{#1}\|_{\ell_\infty}}
\newcommand{\trow}[1]{\|{#1}\|_{2,\infty}}

\newcommand{\tf}[1]{\|{#1}\|_{F}}

\newcommand{\upp}{{\cal{B}}_{\alpha,\Gamma}}
\newcommand{\upz}{{\cal{B}}_{\alpha_0,\Gamma}}
\newcommand{\dpz}{{\cal{D}}_{\alpha_0,\Gamma}}
\newcommand{\dpp}{{\cal{D}}_{\alpha,\Gamma}}

\newcommand{\Rc}{\mathcal{I}}

\newcommand{\Rcb}{\mathcal{N}}

\newcommand{\bteta}{\boldsymbol{\theta}}
\newcommand{\bbteta}{\widetilde{\boldsymbol{\theta}}}
\newcommand{\brteta}{\overline{\boldsymbol{\theta}}}

\newcommand{\Sc}{\mathcal{S}}

\newcommand{\Nn}{\mathcal{N}}

\newcommand{\vb}{\vct{v}}
\newcommand{\Jb}{\mtx{J}}

\newcommand{\Ic}{{\mathcal{I}}}

\newcommand{\w}{\vct{w}}
\newcommand{\err}[1]{\text{Err}_{\Dc}(#1)}

\newcommand{\ab}{\vct{a}}
\newcommand{\bb}{\vct{b}}
\newcommand{\ub}{{\vct{u}}}

\newcommand{\g}{{\vct{g}}}

\newcommand{\Fc}{\mathcal{F}}

\newcommand{\mat}[1]{{\text{mat}{#1}}}

\newcommand{\opnorm}[1]{\left\|#1\right\|}
\newcommand{\fronorm}[1]{\left\|#1\right\|_{F}}

\newcommand{\twonorm}[1]{\left\|#1\right\|_{\ell_2}}

\newcommand{\infnorm}[1]{\left\|#1\right\|_{\ell_\infty}}

\newcommand{\abs}[1]{\left|#1\right|}

\newcommand{\x}{\vct{x}}
\newcommand{\rb}{\vct{r}}
\newcommand{\rbb}{\vct{\widetilde{r}}}
\newcommand{\y}{\vct{y}}

\newcommand{\W}{\mtx{W}}
\newcommand{\Wc}{{\cal{W}}}
\newcommand{\Vc}{{\cal{V}}}

\definecolor{emmanuel}{RGB}{255,127,0}

\newcommand{\R}{\mathbb{R}}
\newcommand{\Pro}{\mathbb{P}}

\newcommand{\E}{\operatorname{\mathbb{E}}}

\newcommand{\grad}[1]{{\nabla\Lc(#1)}}

\newcommand{\e}{\mathrm{e}}
\newcommand{\eb}{\vct{e}}

\newcommand{\vct}[1]{\bm{#1}}
\newcommand{\mtx}[1]{\bm{#1}}

\newcommand{\X}{{\mtx{X}}}
\newcommand{\Y}{{\mtx{Y}}}
\newcommand{\Vb}{{\mtx{V}}}

\newcommand{\calF}{\mathcal{I}}
\newcommand{\calS}{\mathcal{N}}

\numberwithin{equation}{section} 

\def \endprf{\hfill {\vrule height6pt width6pt depth0pt}\medskip}
\newenvironment{proof}{\noindent {\bf Proof} }{\endprf\par}


\newcommand*\samethanks[1][\value{footnote}]{\footnotemark[#1]}

\begin{document}
\title{Overparameterization without Overfitting: Jacobian-based Generalization Guarantees for Neural Networks}
\title{Generalization Guarantees for Neural Networks via Low-Dimensional Representation of Data}
\title{Overparameterization without Overfitting:\\Harnessing the Jacobian Structure for Generalization Guarantees}
\title{Harnessing the Jacobian Structure\\ for Neural Network Generalization}
\title{Harnessing Low-rank Structure in the Jaocbian\\ for Neural Network Generalization}
\title{Generalization Guarantees for Neural Networks\\via Harnessing the Low-rank Structure of the Jacobian}
\author{Samet Oymak\thanks{{Department of Electrical and Computer Engineering, University of California, Riverside, CA}}\quad  \quad Zalan Fabian\thanks{{Ming Hsieh Department of Electrical Engineering, University of Southern California, Los Angeles, CA}} $^\alpha$\quad  \quad Mingchen Li\thanks{{Department of Computer Science and Engineering, University of California, Riverside, CA}} $^\alpha$  \quad \quad Mahdi Soltanolkotabi\samethanks[2]}
\maketitle
\begin{abstract}
Modern neural network architectures often generalize well despite containing many more parameters than the size of the training dataset. This paper explores the generalization capabilities of neural networks trained via gradient descent. We develop a data-dependent optimization and generalization theory which leverages the low-rank structure of the Jacobian matrix associated with the network. Our results help demystify why training and generalization is easier on clean and structured datasets and harder on noisy and unstructured datasets as well as how the network size affects the evolution of the train and test errors during training. Specifically, we use a control knob to split the Jacobian spectum into ``information" and ``nuisance" spaces associated with the large and small singular values. We show that over the information space learning is fast and one can quickly train a model with zero training loss that can also generalize well. Over the nuisance space training is slower and early stopping can help with generalization at the expense of some bias. We also show that the overall generalization capability of the network is controlled by how well the label vector is aligned with the information space. A key feature of our results is that even constant width neural nets can provably generalize for sufficiently nice datasets. We conduct various numerical experiments on deep networks that corroborate our theoretical findings and demonstrate that: (i) the Jacobian of typical neural networks exhibit low-rank structure with a few large singular values and many small ones leading to a low-dimensional information space, (ii) over the information space learning is fast and most of the label vector falls on this space, and (iii) label noise falls on the nuisance space and impedes optimization/generalization.
\end{abstract}\vspace{-0pt}
{\let\thefootnote\relax\footnotetext{~~~\hspace{1pt}$\alpha$\hspace{1pt}Equal contribution.}}
\section{Introduction}
\subsection{Motivation and contributions}
Deep neural networks (DNN) are ubiquitous in a growing number of domains ranging from computer vision to healthcare. State-of-the-art DNN models are typically overparameterized and contain more parameters than the size of the training dataset. It is well understood that in this overparameterized regime, DNNs are highly expressive and have the capacity to (over)fit arbitrary training datasets including pure noise \cite{zhang2016understanding}. Mysteriously however neural network models trained via simple algorithms such as (stochastic) gradient descent continue to predict well or \emph{generalize} on yet unseen test data. In this paper we wish to take a step towards demystifying this phenomenon and help explain why neural nets can overfit to noise yet have the ability to generalize when real data sets are used for training. In particular we explore the generalization dynamics of neural nets trained via gradient descent. Using the Jacobian mapping associated to the neural network we characterize directions where learning is fast and generalizable versus directions where learning is slow and leads to overfitting. The main contributions of this work are as follows.

\noindent $\bullet$ {\bf{Leveraging dataset structure:}} We develop new optimization and generalization results that can harness the low-rank representation of semantically meaningful datasets via the Jacobian mapping of the neural net. This sheds light as to why training and generalization is easier using datasets where the features and labels are semantically linked versus others where there is no meaningful relationship between the features and labels (even when the same network is used for training).\\
\noindent $\bullet$ {\bf{Bias--variance tradeoffs:}} We develop a bias--variance theory based on the Jacobian which decouples the learning process into {\em{information}} and {\em{nuisance}} spaces. We show that gradient descent almost perfectly interpolates the data over the information space (incurring only a small bias). In contrast, optimization over the nuisance space is slow and results in overfitting due to higher variance.\\ 
\noindent $\bullet$ {\bf{Network size vs prediction bias:}} We obtain data-dependent tradeoffs between the network size and prediction bias. Specifically, we show that larger networks result in smaller prediction bias, but small networks can still generalize well, especially when the dataset is sufficiently structured, but typically incur a larger bias. This is in stark contrast to recent literature on optimization and generalization of neural networks \cite{arora2019fine,du2018gradient,allen2018convergence,cao2019generalization,ma2019comparative,allen2018learning,brutzkus2017sgd} where guarantees only hold for very wide networks with the width of the network growing inversely proportional to the distance between the input samples or class margins or related notions. See Section \ref{priorart} for further detail.\\
\noindent $\bullet$ {\bf{Pretrained models:}} In our framework we do not require the initialization to be random and our results continue to apply even with arbitrary initialization. Therefore, our results may shed light on the generalization capabilities of networks initialized with pre-trained models such as those commonly used in meta/transfer learning.\\
\subsection{Model and training}
In our theoretical analysis we focus on neural networks consisting of one hidden layer with $d$ input features, $k$ hidden neurons and $\K$ outputs as depicted in Figure \ref{neuralnet}. We use $\W\in\R^{k\times d}$ and $\Vb\in\R^{\K\times k}$ to denote the input-to-hidden and hidden-to-output weights. The overall input-output relationship of the neural network in this case is a function $f(\cdot;\W):\R^d\rightarrow\R^\K$ that maps an input vector $\vct{x}\in\R^d$ into an output of size $\K$ via
\begin{align}
\label{neural net func}
\vct{x}\mapsto f(\vct{x};\mtx{W}):=\Vb\phi(\mtx{W}\vct{x}).
\end{align}
Given a dataset consisting of $n$ feature/label pairs $(\vct{x}_i,\vct{y}_i)$ with $\vct{x}_i\in\R^d$ representing the features and $\y_i\in\R^\K$ the associated labels representing one of $\K$ classes with one-hot encoding (i.e.~$\vct{y}_i\in\{\vct{e}_1,\vct{e}_2,\ldots,\vct{e}_K\}$ where $\vct{e}_\ell\in\R^K$ has all zero entries except for the $\ell$th entry which is equal to one). To learn this dataset, we fix the output layer and train over $\mtx{W}$ via\footnote{For clarity of exposition, we focus only on optimizing over the input layer. However, as shown in the supplementary material, the technical approach is quite general and applies to arbitrary multiclass nonlinear least-squares problems. In particular, the proofs are stated so as to apply (or easily extend) to one-hidden layer networks where both layers are used for training. These results when combined can be used to prove variations of Theorems \ref{gen main} and \ref{nn deter gen} when both layers are trained.}

\begin{figure}[t!]
\centering
\includegraphics[scale=1]{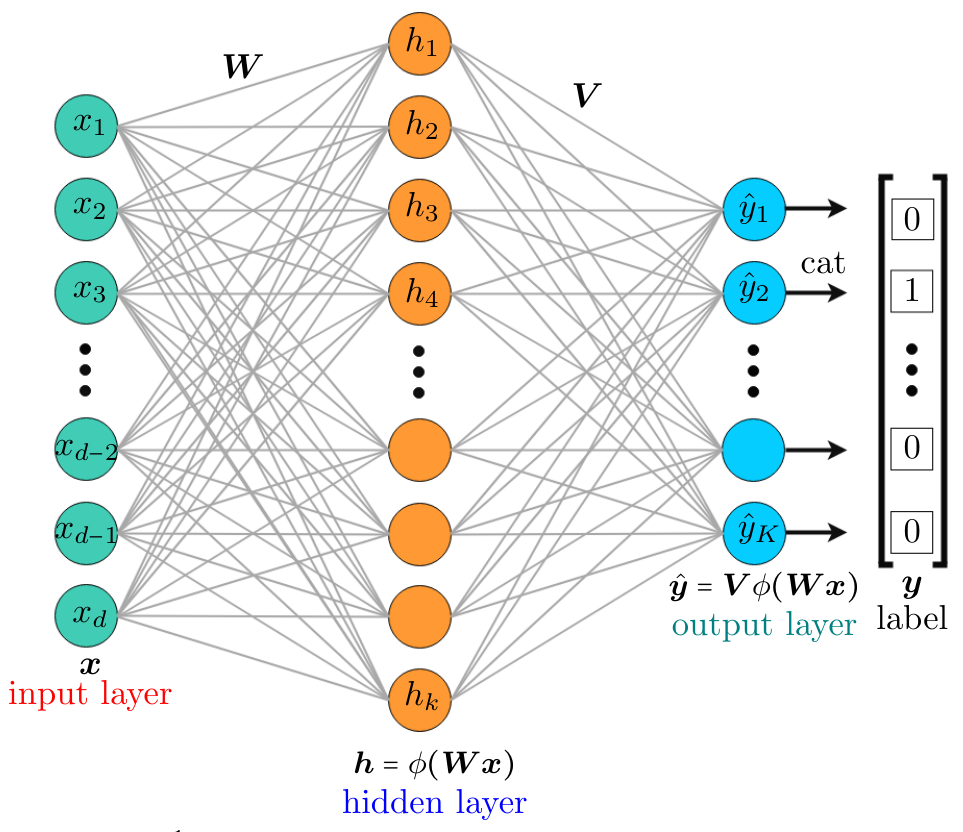}
\caption{Illustration of a one-hidden layer neural net with $d$ inputs, $k$ hidden units and $\K$ outputs along with a one-hot encoded label.}
\label{neuralnet}
\end{figure}

\begin{align}
\label{neuralopt}
\underset{\mtx{W}\in\R^{k\times d}}{\min}\text{ }\mathcal{L}(\mtx{W}):=\frac{1}{2}\sum_{i=1}^n \twonorm{\Vb\phi\left(\mtx{W}\vct{x}_i\right)-\y_i}^2.
\end{align}
It will be convenient to concatenate the labels and prediction vectors as follows
\begin{align}
\label{concat}
\y=\begin{bmatrix}\y_1\\\vdots\\\y_n\end{bmatrix}\in\R^{nK}\quad\text{and}\quad f(\W)=\begin{bmatrix}\Vb f(\x_1;\W)\\\vdots\\\Vb f(\x_n;\W)\end{bmatrix}\in\R^{nK}.
\end{align}
Using this shorthand we can rewrite the loss \eqref{neuralopt} as
\begin{align}
\label{neuralopt2}
\underset{\mtx{W}\in\R^{k\times d}}{\min}\text{ }\mathcal{L}(\mtx{W}):=\frac{1}{2}\twonorm{f(\W)-\vct{y}}^2.
\end{align}
To optimize this loss starting from an initialization $\W_0$ we run gradient descent iterations of the form
\begin{align}
\W_{\tau+1}=\W_\tau-\eta \grad{\W_\tau},\label{grad dec me}
\end{align}
with a step size $\eta$. In this paper we wish to explore the theoretical properties of the model found by such iterative updates with an emphasis on the generalization ability. 

\section{Components of a Jacobian-based theory of generalization}
\subsection{Prelude: fitting a linear model}\label{linmodel} 
To gain better insights into what governs the generalization capability of gradient based iterations let us consider the simple problem of fitting a linear model via gradient descent. This model maps an input/feature vector $\vct{x}\in\R^d$ into a one-dimensional output/label via $\vct{x}\mapsto f(\vct{x},\w):=\w^T\x$. We wish to fit a model of this form to $n$ training data consisting of input/label pairs $\{(\vct{x}_i,y_i)\}_{i=1}^n\in\R^d\times \R$. Aggregating this training data as rows of a feature matrix $\X\in\R^{n\times d}$ and label vector $\y\in\R^n$, the training problem takes the form 
\begin{align} \label{LSt} 
\underset{}{}\mathcal{L}(\w)=\frac{1}{2}\twonorm{\X\w-\y}^2.
\end{align} 
We focus on an overparameterized model where there are fewer training data than the number of parameters i.e.~$n\le d$. We assume the feature matrix can be decomposed into the form $\mtx{X}=\overline{\X}+\mtx{Z}$ where $\overline{\X}$ is low-rank (i.e. rank$(\overline{\X})=r<<n$) with singular value decomposition $\overline{\X}=\vct{U}\mtx{\Sigma}\mtx{V}^T$ with $\mtx{U}\in\R^{n\times r}$, $\mtx{\Sigma}\in\R^{r\times r}$, $\mtx{V}\in\R^{d\times r}$, and $\mtx{Z}\in\R^{n\times d}$ is a matrix with i.i.d.~$\mathcal{N}(0,\sigma_{x}^2/n)$ entries. We shall also assume the labels are equal to $\y=\overline{\y}+\vct{z}$ with $\overline{\vct{y}}=\overline{\X}\w^*$ for some $\w^*\in\text{Range}(\mtx{V})$ and $\vct{z}\in\R^n$ a Gaussian random vector with i.i.d.~$\mathcal{N}(0,\sigma_{y}^2/n)$ entries. One can think of this as a linear regression model where the features and labels are corrupted with Gaussian noise. The goal of course is to learn a model which fits to the clean uncorrupted data and not the corruption. In this case the population loss (i.e.~test error) takes the form
\begin{figure*}[t!] 
\centering \begin{subfigure}[b]{0.5\textwidth} 
\centering 
\includegraphics[scale=1.4]{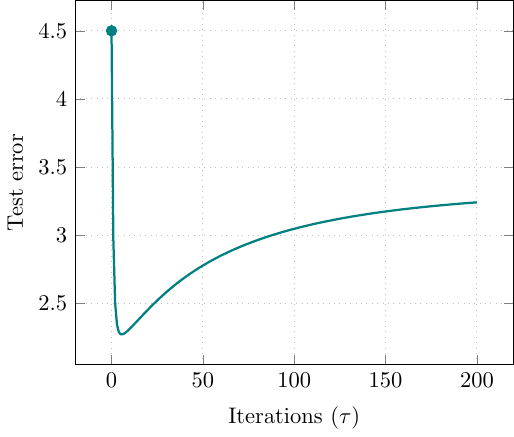} 
\caption{Total test error} 
\label{synthfigtest} 
\end{subfigure}~
\begin{subfigure}[b]{0.5\textwidth}
\centering
\includegraphics[scale=1.4]{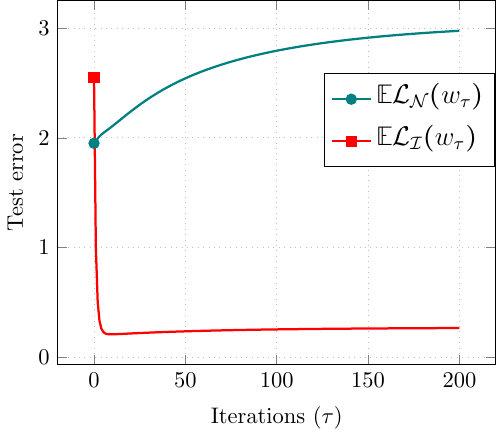}
\caption{Test error along information and nuisance spaces} 
\label{synthfigtest2} 
\end{subfigure} 
\caption{Plots of the (a) total test error and (b) the test error components for the model in Section \ref{linmodel}. The test error decreases rapidly over the information subspace but slowly increases over the nuisance subspace.}
\label{synthfig}
\end{figure*}
\begin{align*}
\E\big[\mathcal{L}(\w)\big]=\frac{1}{2}\twonorm{\overline{\mtx{X}}\w-\overline{\vct{y}}}^2+\frac{1}{2}\sigma_x^2\twonorm{\w}^2+\frac{1}{2}\sigma_y^2,
\end{align*}
Now let us consider gradient descent iterations with a step size $\eta$ which take the form 
\begin{align} 
\label{LSGD}
\w_{\tau+1}=\w_\tau-\eta \nabla\mathcal{L}(\w_\tau)=\left(\mtx{I}-\eta\X^T\X\right)\w_\tau+\eta\X^T\y. \end{align}
To gain further insight into the generalization capabilities of the gradient descent iterations we shall consider an instance of this problem where the subspaces $\mtx{U}$ and $\mtx{V}$ are chosen uniformly at random, $\mtx{\Sigma}=\mtx{I}_r$ with $n=200$, $d=500$, $r=5$, and $\sigma_x =0.2, ~\sigma_y=2$. In Figure \ref{synthfigtest} we plot the population loss evaluated at different iterations. We observe an interesting phenomenon, in the first few iterations the test error goes down quickly but it then slowly increases. To better understand this behavior we decompose the population loss into two parts by tracking the projection of the misfit $\X\w-\y$ on the column space of the uncorrupted portion of the input data $(\mtx{U})$ and its complement. That is, 
\begin{align*} 
\E\mathcal{L}(\w)=\E\mathcal{L}_{\calF}(\w)+\E\mathcal{L}_{\calS}(\w). 
\end{align*}
where
\begin{align*} \E\mathcal{L}_{\calF}(\w):=&\E\Big[\twonorm{\Pi_{\calF}\left(\X\w-\y\right)}^2\Big]=\twonorm{\overline{\X}\w-\overline{\y}}^2+\frac{r}{2n}\sigma_x^2\twonorm{\w}^2+\frac{r}{2n}\sigma_y^2,\\ 
\E\mathcal{L}_{\calS}(\w):=&\E\Big[\twonorm{\Pi_{\calS}\left(\X\w-\y\right)}^2\Big]=\frac{1}{2}\left(1-\frac{r}{n}\right)\left(\sigma_x^2\twonorm{\w}^2+\sigma_y^2\right),
\end{align*} 
with $\Pi_{\calF}=\mtx{U}\mtx{U}^T$ and $\Pi_{\calS}=\mtx{I}-\mtx{U}\mtx{U}^T$. In Figure \ref{synthfigtest2} we plot these two components. This plot clearly shows that $\E\mathcal{L}_{\calF}(\w)$ goes down quickly while $\E\mathcal{L}_{\calS}(\w)$ slowly increases with their sum creating the dip in the test error. Since $\mtx{U}$ is a basis for the range of the uncorrupted portion of the features ($\overline{\mtx{X}}$) one can think of span$(\mtx{U})$ as the ``information" subspace and $\E\mathcal{L}_{\calF}(\w)$ as the test error on this information subspace. Similarly, one can think of the complement of this subspace as the ``nuisance" subspace and $\E\mathcal{L}_{\calS}(\w)$ as the test error on this nuisance subspace. Therefore, one can interpret Figure \ref{synthfigtest} as the test error decreasing rapidly in the first few iterations over the information subspace but slowly increasing due to the contributions of the nuisance subspace. 

To help demystify this behavior note that using the gradient descent updates from \eqref{LSGD} the update in terms of the misfit/residual $\vct{r}_\tau=\X\w_\tau-\y$ takes the form 
\begin{align*} 
\vct{r}_{\tau+1}=\left(\mtx{I}-\eta\X\X^T\right)\vct{r}_\tau=\left(\mtx{I}-\eta\overline{\X}\text{ }\overline{\X}^T\right)\left(\overline{\X}\w_\tau-\overline{y}\right)+noise
\end{align*} 
Based on the form of this update when the information subspace is closely aligned with the prominent singular vectors of $\X$ the test error on the information subspace ($\E\mathcal{L}_{\calF}(\w)\approx \twonorm{\overline{\X}\w_\tau-\overline{\y}}^2$) quickly decreases in the first few iterations. However, the further we iterate the parts of the residual aligned with the less prominent eigen-directions of $\X$ (which correspond to the nuisance subspace) slowly pick up more energy contributing to a larger total test error.

\subsection{Information and nuisance spaces of the Jacobian}
In this section we build upon the intuition gained from the linear case to develop a better understanding of generalization dynamics for nonlinear data fitting problems. As in the linear case, in order to understand the generalization capabilities of models trained via gradient descent we need to develop better insights into the form of the gradient updates and how it affects the training dynamics. To this aim let us aggregate the weights at each iteration into one large vector $\vct{w}_\tau:=$vect$(\W_\tau)\in\R^{kd}$, define the misfit/residual vector $\rb(\w):=f(\w)-\y$ and note that the gradient updates take the form 
\begin{align*}
\w_{\tau+1}=\w_{\tau}-\eta\nabla\mathcal{L}(\w_\tau)\quad\text{where}\quad\nabla\mathcal{L}(\w)=\nabla \mathcal{L}(\w)=\mathcal{J}^T(\w)\vct{r}(\w).
\end{align*}
Here, $\mathcal{J}(\w)\in\R^{nK\times kd}$ denotes the Jacobian mapping associated with $f$ defined as $\mathcal{J}(\w)=\frac{\partial f(\w)}{\partial \w}$.
%
%
%
Due to the form of the gradient updates the dynamics of training is dictated by the spectrum of the Jacobian matrix as well as the interaction between the residual vector and the Jacobian. If the residual vector is very well aligned with the singular vectors associated with the top singular values of $\mathcal{J}(\w)$, the gradient update significantly reduces the misfit allowing substantial reduction in the train error. In a similar fashion we will also show that if the labels $\vct{y}$ are well-aligned with the prominent directions of the Jacobian the test error of the trained network will be low. Thus to provide a more precise understanding of the training dynamics and generalization capabilities of neural networks it is crucial to develop a better understanding of the interaction between the Jacobian and the misfit and label vectors. To capture these interactions we require a few definitions.
\begin{figure}[t!]\hspace{-30pt}
\begin{centering}
\begin{subfigure}[t]{3.1in}
\centering
\begin{tikzpicture}
\node at (0,0) {\includegraphics[height=0.6\linewidth,width=0.72\linewidth]{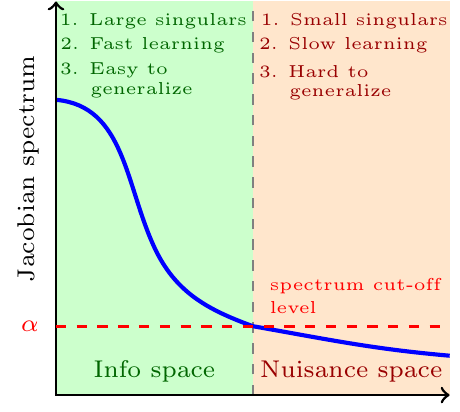}};
\end{tikzpicture}
\caption{Depiction via the Jacobian spectrum}
\label{fig3a}
\end{subfigure}
\end{centering}\hspace{-20pt}
\begin{centering}
\begin{subfigure}[t]{3in}
\centering
\begin{tikzpicture}
\node at (0,0) {\includegraphics[height=0.6\linewidth,width=1.02\linewidth]{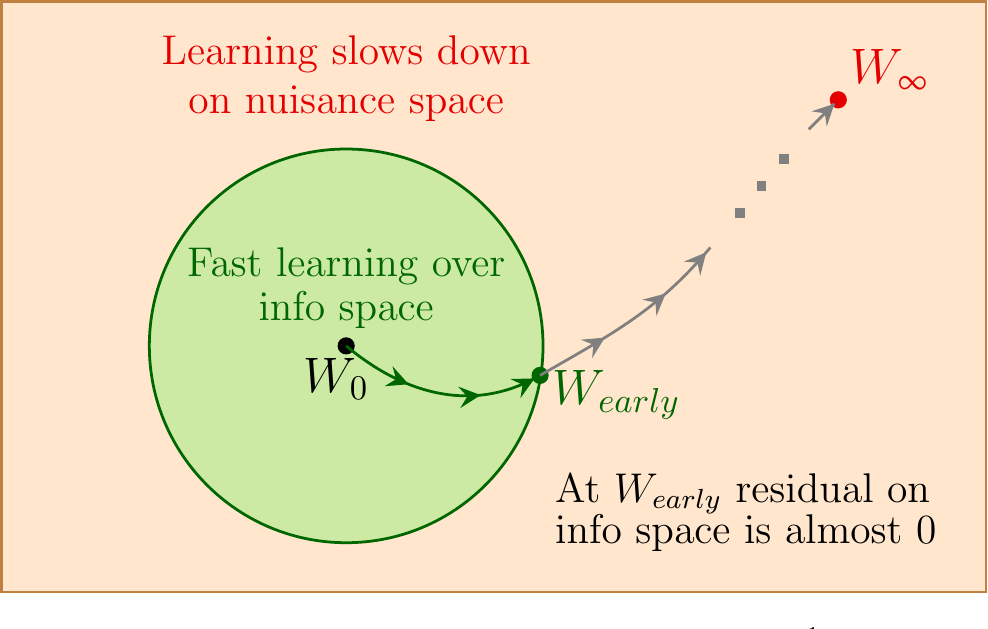}};
\end{tikzpicture}
\caption{Depiction in parameter space}\label{fig3b}
\end{subfigure}
\end{centering}\caption{Depiction of the training and generalization dynamics of gradient methods based on the information and nuisance spaces associated with the neural net Jacobian.}\label{fig3}
\end{figure}
\begin{definition}[Information \& Nuisance Spaces]\label{tjac neural} Consider a matrix $\mtx{J}\in\R^{nK\times p}$ with singular value decomposition given by
\[
\mtx{J}=\sum_{s=1}^{n\K}\lambda_s\ub_s\vb_s^T=\mtx{U}\text{diag}\left(\lambda_1, \lambda_2,\ldots,\lambda_{nK}\right)\mtx{V}^T,
\]
with $\lambda_1\ge \lambda_2\ge \ldots \ge\lambda_{nK}$ denoting the singular values of $\mtx{J}$ in decreasing order and $\{\vct{u}_s\}_{s=1}^{nK}\in\R^{nK}$ and $\{\vct{v}_s\}_{s=1}^{nK}\in\R^{p}$ the corresponding left and right singular vectors forming the orthonormal basis matrices $\mtx{U}\in\R^{nK\times nK}$ and $\mtx{V}\in\R^{p\times nK}$. For a spectrum cutoff $\alpha$ obeying $0\le \alpha\le \lambda_1$ let $r:=r(\alpha)$ denote the index of the smallest singular value above the threshold $\alpha$.

We define the information and nuisance spaces associated with $\mtx{J}$ as $\calF:=\text{span}(\{\ub_s\}_{s=1}^r)$ and $\calS:=\text{span}(\{\ub_s\}_{s=r+1}^{\K n})$. We also define the truncated Jacobian
\begin{align*}
\mtx{J}_{\calF}=\begin{bmatrix}\vct{u}_1 & \vct{u}_2 & \ldots & \vct{u}_r\end{bmatrix}\text{diag}\left(\lambda_1,\lambda_2,\ldots,\lambda_r\right)\begin{bmatrix}\vct{v}_1 & \vct{v}_2 & \ldots & \vct{v}_r\end{bmatrix}^T
\end{align*}
which is the part of the reference Jacobian that acts on the information space $\calF$.
\end{definition}
In this paper we shall use either the expected value of the Jacobian at the random initialization or the Jacobian at one of the iterates to define the matrix $\mtx{J}$ and the corresponding information/nuisance spaces. More, specifically we will set $\mtx{J}$ to either $\mtx{J}=\left(\E[\mathcal{J}(\W_0)\mathcal{J}^T(\W_0)]\right)^{1/2}$ or $\mtx{J}=\mathcal{J}(\W_\tau)$. Therefore, one can effectively think of the information space as the span of the prominent singular vectors of the Jacobian and the nuisance space as its complement. In particular, as we demonstrate in Section \ref{numeric sec} the Jacobian mapping associated with neural networks exhibit low-rank structure with a few large singular values and many small ones leading to natural choices for the cut-off value $\alpha$ as well as the information and nuisance spaces. Furthermore, we demonstrate both (empirically and theoretically) that learning is fast over the information space leading to a significant reduction in both train/test accuracy in the early stages of training. However, after a certain number of iterations learning shifts to the nuisance space and reduction in the training error significantly slows down. Furthermore, subsequent iterations in this stage lead to a slight increase in test error. We provide a cartoon depiction of this behavior in Figure \ref{fig3}. 

\section{Main results}
Our main results establish multi-class generalization bounds for neural networks trained via gradient descent. First, we will focus on networks where both layers are randomly initialized. Next we will provide guarantees for arbitrary initialization with the goal of characterizing the generalization ability of subsequent iterative updates for a given (possibly pre-trained) network in terms of its Jacobian mapping. In this paper we focus on activations $\phi$ which are smooth and have bounded first and second order derivatives. This would for instance apply to the softplus activation $\phi(z)=\log\left(1+e^z\right)$. We note that utilizing a proof technique developed in \cite{onehidden} for going from smooth to ReLU activations it is possible to extend our results to ReLU activations with proper modifications. We avoid doing this in the current paper for clarity of exposition. Before we begin discussing our main results we discuss some notation used throughout the paper. For a matrix $\mtx{X}\in\R^{n\times d}$ we use $\smn{\mtx{X}}$ and $\smx{\mtx{X}}=\opnorm{\mtx{X}}$ to denote the minimum and maximum singular value of $\mtx{X}$. For two matrices $\mtx{A}$ and $\mtx{B}$ we use $\mtx{A}\odot B$ and $\mtx{A}\otimes B$ to denote their Hadamard and Kronecker products, respectively. For a PSD matrix $\mtx{A}\in\R^{n\times n}$ with eigenvalue decomposition $\mtx{A}=\sum_{i=1}^n\lambda_i\vct{u}_i\vct{u}_i^T$, the square root matrix is defined as $\mtx{A}^{1/2}:=\sum_{i=1}^n\sqrt{\lambda_i}\vct{u}_i\vct{u}_i^T$. We also use $\A^\dagger$ to denote the pseudo-inverse of $\A$. In this paper we mostly focus on label vectors $\y$ which are one-hot encoded i.e.~all entries are zero except one of them. For a subspace $\mathcal{S}\subset \R^n$ and point $\x\in\R^n$, $\Pi_{\mathcal{S}}(\vct{x})$ denotes the projection of $\x$ onto $\mathcal{S}$. Finally, before stating our results we need to provide a quantifiable measure of performance for a trained model. Given a sample $(\x,\y)\in\R^{d}\times \R^\K$ from a distribution $\Dc$, the classification error of the network $\W$ with respect to $\Dc$ is defined as
\begin{align}
\err{\W}=\Pro\Big\{\arg\max_{1\leq \ell\leq K}\y_\ell\neq \arg\max_{1\leq \ell\leq K}f_\ell(\x;\W)\Big\}.\label{class error}
\end{align}

\subsection{Results for random initialization}
To explore the generalization of randomly initialized networks, we utilize the neural tangent kernel.
\begin{definition}[Multiclass Neural Tangent Kernel (M-NTK) \cite{jacot2018neural}]\label{nneig} Let $\vct{w}\in\R^d$ be a vector with $\mathcal{N}(\vct{0},\mtx{I}_d)$ distribution. Consider a set of $n$ input data points $\vct{x}_1,\vct{x}_2,\ldots,\vct{x}_n\in\R^d$ aggregated into the rows of a data matrix $\X\in\R^{n\times d}$. Associated to the activation $\phi$ and the input data matrix $\X$ we define the multiclass kernel matrix as
\begin{align*}
\mtx{\Sigma}(\X):=\Iden_{\K}\otimes \E\Big[\left(\phi'\left(\X\w\right)\phi'\left(\X\w\right)^T\right)\odot\left(\X\X^T\right)\Big],
\end{align*}
\end{definition}
where $\Iden_K$ is the identity matrix of size $K$. Here, the $\ell$ th diagonal block of $\mtx{\Sigma}(\X)$ corresponds to the kernel matrix associated with the $\ell$ th network output for $1\le\ell\leq \K$. This kernel is intimately related to the multiclass Jacobian mapping. In particular, suppose the initial input weights $\W_0$ are distributed i.i.d.~ $\Nn(0,1)$ and the output layer $\Vb$ has i.i.d.~zero-mean entries with $\vrn^2/K$ variance. Then $\E[\Jc(\W_0)\Jc(\W_0)^T]=\vrn^2\mtx{\Sigma}(\X)$. We use the square root of this multiclass kernel matrix (i.e.~$\mtx{\Sigma}(\X)^{1/2}$) to define the information and nuisance spaces for our random initialization result. 
\vspace{-2pt}\begin{theorem}  \label{gen main}
Let $\zeta, \Gamma, \bar{\alpha}$ be scalars obeying $\zeta\le 1/2$, $\Gamma\ge 1$, and $\bar{\alpha}\ge 0$ which determine the overall precision, cut-off and learning duration, respectively.\footnote{Note that this theorem and its conclusions hold for any choice of these parameters in the specified range.} Consider a training data set $\{(\x_i,\y_i)\}_{i=1}^n\in\R^d\times \R^K$ generated i.i.d.~according to a distribution $\Dc$ where the input samples have unit Euclidean norm and the concatenated label vector obeys $\tn{\y}=\sqrt{n}$ (e.g.~one-hot encoding). Consider a neural net with $k$ hidden nodes as described in \eqref{neural net func} parameterized by $\W$ where the activation function $\phi$ obeys $\abs{\phi'(z)}, \abs{\phi''(z)}\le B$. Let $\W_0$ be the initial weight matrix with i.i.d.~$\Nn(0,1)$ entries. Fix a precision level $\zeta$ and set $\vrn=\zeta/(50B\sqrt{\log(2\K)})$.
Also assume the output layer $\Vb$ has i.i.d.~Rademacher entries scaled by $\frac{\vrn}{\sqrt{k\K}}$. Furthermore, set $\mtx{J}:=(\bSi(\X))^{1/2}$ and define the information $\calF$ and nuisance $\calS$ spaces and the truncated Jacobian $\mtx{J}_{\calF}$ associated with the Jacobian $\mtx{J}$ based on a cut-off spectrum value of $\alpha_0=\bar{\alpha}\sqrt[4]{n}\sqrt{K\opnorm{\X}}B$ per Definition \ref{tjac neural}. Assume
\begin{align}
\label{hidnum}
k\gtrsim \frac{\Gamma^4\log n}{\zeta^4\bar{\alpha}^8}
\end{align}
with $\Gamma\ge 1$. We run gradient descent iterations of the form \eqref{grad dec me} with a learning rate $\eta\le \frac{1}{\vrn^2B^2\opnorm{\X}^2}$. Then, after $T=\frac{\Gamma K}{\eta\vrn^2\alpha_0^2}$ iterations, classification error $\err{\W_{T}}$ is upper bounded by 
\[
\underbrace{\frac{2\tn{ \Pi_{\Rcb}(\y)}}{\sqrt{n}}}_{\text{bias term}}+\underbrace{\frac{12B\sqrt{\K}}{\sqrt{n}}\left(\twonorm{\mtx{J}_{\calF}^\dagger\y}+\frac{\Gamma}{\alpha_0}\tn{ \Pi_{\Rcb}(\y)}\right)}_{\text{variance term}}+12\Big(1+\frac{\Gamma}{\bar{\alpha}\sqrt[4]{n\|\X\|^2}}\Big)\zeta+5\sqrt{\frac{\log(2/\delta)}{n}}+2\e^{-\Gamma},
\]
holds with probability at least $1-(2\K)^{-100}-\delta$.
\end{theorem}
This theorem shows that even networks of moderate width can achieve a small generalization error if (1) the data has low-dimensional representation i.e.~the kernel is approximately low-rank and (2) the inputs and labels are semantically-linked i.e.~the label vector $\y$ mostly lies on the information space.\\
%
%

\noindent $\bullet$ {\bf{Bias--Variance decomposition:}} The generalization error has two core components: bias and variance. The bias component ${\tn{ \Pi_{\Rcb}(\y)}}/{\sqrt{n}}$ arises from the portion of the labels that falls over the nuisance space leading to a nonzero training error. The variance component is proportional to the distance $\tf{\W_T-\W_0}$ and arises from the growing model complexity as gradient descent strays further away from the initialization while fitting the label vector over the information space. If the label vector is aligned with the information space, bias term $\Pi_{\Rcb}(\y)$ will be small. Additionally, if the kernel matrix is approximately low-rank, we can set $\bar{\alpha}$ to ensure small variance even when the width grows at most logarithmically with the size of the training data as required by \eqref{hidnum}. In particular, using $\tn{\mtx{J}_{\calF}^\dagger\y}\leq \tn{\y}/\alpha_0\leq \sqrt{n}/\alpha_0$, the bound simplifies to
\begin{align}
\err{\W_{T}}\leq \frac{2}{\sqrt{n}}\tn{ \Pi_{\Rcb}(\y)}+{\frac{36\Gamma}{\bar{\alpha}\sqrt[4]{n\opnorm{\X}^2}}}+12\zeta+5\sqrt{\frac{\log(2/\delta)}{n}}+2\e^{-\Gamma},\label{simple bound}
\end{align}
which is small as soon as the label vector is well-aligned with the information subspace. We note however that our results continue to apply even when the kernel is not approximately low-rank. In particular, consider the extreme case where we select $\alpha_0=\sqrt{\lambda}:=\sqrt{\lambda_{\min}\left(\mtx{\Sigma}(\X)\right)}$. Then, the information space $\Rc$ spans $\R^{\K n}$ and the bias term disappears ($\tn{ \Pi_{\Rcb}(\y)}=0$) and 
\[
\tn{\mtx{J}_{\calF}^\dagger\y}=\tn{\mtx{J}^\dagger\y}=\sqrt{\y^T\bSi(\X)^{-1}\y}.
\]
In this case our results guarantee that 
\begin{align}
\label{simpgen}
\err{\W_{T}}\lesssim \frac{\sqrt{K}}{\sqrt{n}}\sqrt{\vct{y}^T\mtx{\Sigma}^{-1}(\X)\vct{y}}+\sqrt{\frac{\log(2/\delta)}{n}},
\end{align}
holds as long as $\bSi(\X)$ is invertible and the width of the network obeys
\begin{align}
\label{simwidth}
k\gtrsim \frac{n^2K^4\opnorm{\X}^4 \log n}{\lambda^4}
\end{align}
We note that in this special case our results improve upon the required width in recent literature \cite{arora2019fine}\footnote{Based on our understanding \cite{arora2019fine} requires the number of hidden units to be at least on the order of $k\gtrsim \frac{n^8}{\lambda^6}$. Note that using the fact that $\opnorm{\X}\le \sqrt{n}$ our result reduces the dependence on width by a factor of at least $\frac{n^4}{\lambda^2}$. We note that $\opnorm{\X}$ often scales with $\sqrt{\frac{n}{d}}$ so that the improvement in width is even more pronounced in typical instances.} that focuses on $\K=1$ and a conclusion of the form \eqref{simpgen}. However, as we demonstrate in our numerical experiments in practice $\lambda$ can be rather small or even zero (e.g. see the toy model in Section \ref{mix sec}) so that requirements of the form \eqref{simwidth} may require unrealistically (or even infinitely) wide networks. In contrast, as discussed above by harnessing the low-rank structure of the Jacobian our results show that neural networks generalize well as soon as the width grows at most logarithmically in the size of the training data (even when $\lambda=0$).

\noindent $\bullet$ {\bf{Small width is sufficient for generalization:}} Based on our simulations the M-NTK (or more specifically Jacobian at random initialization) indeed has low-rank structure with a few large eigenvalues and many smaller ones. As a result a typical scaling of the cut-off $\alpha_0$ is so that $\bar{\alpha}$ scales like a constant. In that case our result states that as soon as the number of hidden nodes are moderately large (e.g.~logarithmic in $n$) then good generalization can be achieved. Specifically we can achieve good generalization by using width on the order of $\log n$ and picking small values for $\zeta$ and $\bar{\alpha}$ and large values for $\Gamma$. 

\noindent $\bullet$ {\bf{Network size--Bias tradeoff:}} Based on the requirement \eqref{hidnum} if the network is large (in terms of \# of hidden units $k$), we can choose a small cut-off $\alpha_0$. This in turn allows us to enlargen the information space and reduce the training bias. In summary, as network capacity grows, we can gradually interpolate finer detail and reduce bias. On the other hand, choosing a properly large $\alpha_0$, we can obtain good bounds for even small network sizes $k$ as long as the portion of the labels that fall on the nuisance space is small. This is in stark contrast to related works \cite{arora2019fine,du2018gradient,allen2018convergence,cao2019generalization} where network size grows inversely proportional to the distance between the input samples or other notions of margin.

\noindent $\bullet$ {\bf{Fast convergence:}} We note that by setting learning rate to $\eta=\frac{1}{\vrn^2B^2\opnorm{\X}^2}$, the number of gradient iterations is upper bounded by $\frac{\Gamma }{\bar{\alpha}^2}$. Hence, the training speed is dictated by and is inversely proportional to the the smallest singular value over the information space. Specifically, when the Jacobian is sufficiently low-rank so that we can pick $\bar{\alpha}$ to be a constant, convergence on the information space is rather fast requiring only a {\em{constant number of iterations}} to converge to any fixed constant accuracy. See the proofs for further detail on the optimization dynamics of the training problem (e.g.~results/proofs for linear convergence of the empirical loss).


\subsection{Generalization guarantees with arbitrary initialization}
Our next result provides generalization guarantees from an arbitrary initialization which applies to pre-trained networks (e.g.~those that arise in transfer learning applications) as well as intermediate gradient iterates as the weights evolve. This result has a similar flavor to Theorem \ref{gen main} with the key difference that the information and nuisance spaces are defined with respect to any arbitrary initial Jacobian. This shows that if a pre-trained model\footnote{e.g. obtained by training with data in a related problem as is common in transfer learning.} provides a better low-rank representation of the data in terms of its Jacobian, it is more likely to generalize well. Furthermore, given its deterministic nature the theorem can be applied at any iteration, implying that if the Jacobians of any of the iterates provides a better low-rank representation of the data then one can provide sharper generalization guarantees.  
\begin{theorem} \label{nn deter gen} Let $\zeta, \Gamma, \bar{\alpha}$ be scalars obeying $\zeta\le 1/2$, $\Gamma\ge 1$, and $\bar{\alpha}\ge 0$ which determine the overall precision, cut-off and learning duration, respectively.\footnote{Note that this theorem and its conclusions hold for any choice of these parameters in the specified range.} Consider a training data set $\{(\x_i,\y_i)\}_{i=1}^n\in\R^d\times \R^K$ generated i.i.d.~according to a distribution $\Dc$ where the input samples have unit Euclidean norm. Also consider a neural net with $k$ hidden nodes as described in \eqref{neural net func} parameterized by $\W$ where the activation function $\phi$ obeys $\abs{\phi'(z)}, \abs{\phi''(z)}\le B$. Let $\W_0$ be the initial weight matrix with i.i.d.~$\Nn(0,1)$ entries. Also assume the output matrix has bounded entries obeying $\infnorm{\mtx{V}}\le \frac{\vrn}{\sqrt{kK}}$. Furthermore, set $\mtx{J}:=\mathcal{J}(\W_0)$ and define the information $\calF$ and nuisance $\calS$ subspaces and the truncated Jacobian $\mtx{J}_{\calF}$ associated with the reference/initial Jacobian $\mtx{J}$ based on a cut-off spectrum value $\alpha=\vrn B \bar{\alpha} \sqrt[4]{n}\sqrt{\opnorm{\X}}$. Also define the initial residual $\rb_0=f(\W_0)-\y\in\R^{nK}$ and pick $C_r>0$ so that $\frac{\tn{\rb_0}}{\sqrt{n}}\le C_r$. Suppose number of hidden nodes $k$ obeys
\begin{align}
 k\gtrsim\frac{C_r^2\Gamma^4}{\bar{\alpha}^8\vrn^2\zeta^2},\label{k bound det gen}
\end{align}
with $\Gamma\ge 1$ and tolerance level $\zeta$.
Run gradient descent updates \eqref{grad dec me} with learning rate $\eta\le \frac{1}{\vrn^2B^2\opnorm{\X}^2}$. Then, after $T=\frac{\Gamma }{\eta\alpha^2}$ iterations, with probability at least $1-\delta$, the generalization error obeys
\begin{align*}
\err{\W_{T}}&\leq \underbrace{\frac{2\tn{ \Pi_{\mathcal{N}}(\rb_0)}}{\sqrt{n}}}_{bias~term}+\underbrace{\frac{12\vrn B }{\sqrt{n}} \left(\twonorm{\mtx{J}_{\calF}^\dagger\rb_0}+\frac{\Gamma}{\alpha}\tn{ \Pi_{\Rcb}(\rb_0)}\right)}_{variance~term}+5\sqrt{\frac{\log(2/\delta)}{n}}+2C_r(\e^{-\Gamma}+\zeta).
\end{align*}
\end{theorem}
As with the random initialization result, this theorem shows that as long as the initial residual is sufficiently correlated with the information space, then high accuracy can be achieved for neural networks with moderate width. As with its randomized counter part this result also allows us to study various tradeoffs between bias-variance and network size-bias. Crucially however this result does not rely on random initialization. The reason this is particularly important is two fold. First, in many scenarios neural networks are not initialized at random. For instance, in transfer learning the network is pre-trained via data from a different domain. Second, as we demonstrate in Section \ref{numeric sec} as the iterates progress the Jacobian mapping seems to develop more favorable properties with the labels/initial residuals becoming more correlated with the information space of the Jacobian. As mentioned earlier, due its deterministic nature the theorem above applies in both of these scenarios. In particular, if a pre-trained model provides a better low-rank representation of the data in terms of its Jacobian, it is more likely to generalize well. Furthermore, given its deterministic nature the theorem can be applied at any iteration by setting $\vct{\theta}_0=\vct{\theta}_\tau$, implying that if the Jacobians of any of the iterates provides a better low-rank representation of the data then one can provide sharper generalization guarantees. Our numerical experiments demonstrate that the Jacobian of the neural network seems to adapt to the dataset over time with a more substantial amount of the labels lying on the information space. While we have not formally proven such an adaptation behavior in this paper, we hope to develop rigorous theory demonstrating this adaptation in our future work.Such a result when combined with our arbitrary initialization guarantee above can potentially provide significantly tighter generalization bounds. This is particularly important in light of a few recent literature \cite{chizat2018note, ghorbani2019linearized, yehudai2019power} suggesting a significant gap between generalization capabilities of kernel methods/linearized neural nets when compared with neural nets operating beyond a linear or NTK learning regime (e.g.~mean field regime). As a result we view our deterministic result as a first step towards moving beyond the NTK regime.


\subsection{Case Study: Gaussian mixture model}\label{mix sec}
To illustrate a concrete example, we consider a distribution based on a Gaussian mixture model consisting of $K$ classes where each class consists of $C$ clusters. 

\begin{definition} [Gaussian mixture model]\label{GMM} Consider a data set of size $n$ consisting of input/label pairs $\{(\vct{x}_i,\y_i)\}_{i=1}^n\in\R^d\times \R^K$. We assume this data set consists of $K$ classes each comprising of $C$ clusters with a total of $KC$ clusters. We use the class/cluster pair to index the clusters with $(\ell,\widetilde{\ell})$ denoting the $\widetilde{\ell}$th cluster from the $\ell$th class. We assume the data set in cluster $(\ell,\widetilde{\ell})$ is centered around a cluster center $\vct{\mu}_{\ell,\widetilde{\ell}}\in\R^d$ with unit Euclidian norm. We assume the data set is generated i.i.d.~with the cluster membership assigned uniformlych of the clusters with probability $\frac{1}{KC}$\footnote{This assumption is for simplicity of exposition. Our results (with proper modification) apply to any discrete probability distribution over the clusters.} and the input data points associated with the cluster indexed by $(\ell,\widetilde{\ell})$ are generated i.i.d.~according to $\mathcal{N}\left(\vct{\mu}_{\ell,\widetilde{\ell}},\frac{\sigma^2}{d}\mtx{I}_d\right)$ with the corresponding label set to the one-hot encoded vector associated with class $\ell$ i.e.~$\vct{e}_\ell$. We note that in this model the cluster indexed by $(\ell,\widetilde{\ell})$ contains $\widetilde{n}_{\ell,\widetilde{\ell}}$ data points satisfying $\E[\widetilde{n}_{\ell,\widetilde{\ell}}]=\widetilde{n}=\frac{n}{KC}$.
\end{definition}
\begin{figure}
	\centering
	\begin{tikzpicture}
	\node at (0,0) {\includegraphics[scale=1]{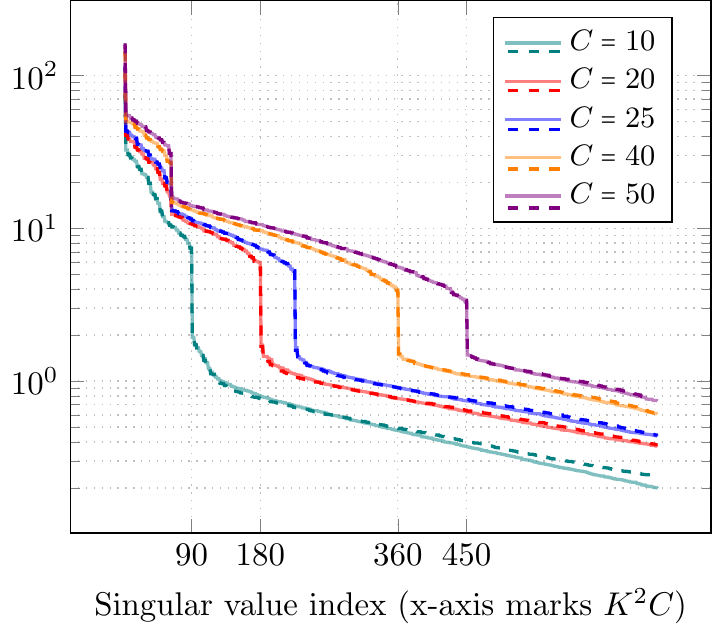}};
	\end{tikzpicture}
	\caption{The singular values of the normalized Jacobian spectrum $\sqrt{\frac{KC}{n}}\Jc(\W_0)$ of a one-hidden layer neural network with $K=3$ outputs. Here, the data set is generated according to the Gaussian mixture model in Definition \ref{GMM} with $K=3$ classes and $\sigma=0.1$. We pick the cluster center so that the distance between any two is at least $0.5$. We consider two cases: $n=30C$ (solid line) and $n=60C$ (dashed line). These plots demonstrate that the top $KC$ singular values grow with the square root of the size of the data set ($\sqrt{n}$).}
	\label{neuralnet}
	\label{synthetic_fig}
\end{figure}

This distribution is an ideal candidate to demonstrate why the Jacobian of the network exhibits low-rank or bimodal structure. Let us consider the extreme case $\sigma=0$ where we have a discrete input distribution over the cluster centers. In this scenario, we can show that the multi-class Jacobian matrix is at most rank 
\[
\K^2 C=\text{\#~of~output~nodes}~\times~\text{\# of distinct inputs}.
\]
as there are (i) only $\K C$ distinct input vectors and (ii) $\K$ output nodes. We can thus set the information space to be the top $\K^2 C$ eigenvectors of the multiclass kernel matrix $\bSi(\X)$. As formalized in the appendix, it can be shown that
\begin{itemize}
\item The singular values of the information space grow proportionally with $n/KC$.
\item The concatenated label vector $\y$ perfectly lies on the information space.
\end{itemize}
In Figure \ref{synthetic_fig} we numerically verify that the approximate rank and singular values of the Jacobian indeed scale as above even when $\sigma>0$. The following informal theorem leverages these observations to establish a generalization bound for this mixture model. This informal statement is for exposition purposes. See Theorem \ref{app thm} in Appendix \ref{Jacclust} for a more detailed result capturing the exact dependencies (e.g.~$\zeta,B,\log n$). In this theorem we use $\gtrsim$ to denote inequality up to constant/logarithmic factors. 
\begin{theorem} [Generalization for Gaussian Mixture Models-simplified] Consider a data set of size $n$ consisting of input/label pairs $\{(\vct{x}_i,\y_i)\}_{i=1}^n\in\R^d\times \R^K$ generated according to a Gaussian mixture model per Definition \ref{GMM} with the standard deviation obeying $\sigma\lesssim\frac{K}{n}$. Let $\M=[\bmu_{1,1}~\dots~\bmu_{\K,C}]^T$ be the matrix obtained by aggregating all the cluster centers as rows and let $\vct{g}\in\R^d$ be a Gaussian random vector distributed as $\Nn(0,\Iden_d)$. Also let $\bSi(\M)\in\R^{KC\times KC}$ be the M-NTK associated with the cluster centers $\mtx{M}$ per Definition \ref{nneig}. Furthermore, set $\la_{\mtx{M}}=\la_{\min}(\bSi(\M))$, and assume $\la_{\mtx{M}}>0$. Also, assume the number of hidden nodes obeys
\[
k\gtrsim \frac{\Gamma^4K^8C^4}{\la_{\mtx{M}}^4}.
\]
Then, after running gradient descent for $T=\frac{2\Gamma K^2C}{\la_{\mtx{M}}}$ iterations, the model obeys
\[
\err{\W_{T}}\lesssim \Gamma \sqrt{\frac{K^2C}{n\la_{\mtx{M}}}},
\]
with high probability.
\end{theorem}
We note that $\la_{\mtx{M}}$ captures how diverse the cluster centers are. In this sense $\la_{\mtx{M}}>0$ intuitively means that neural network, specifically the neural tangent kernel, is sufficiently expressive to interpolate the cluster centers. In fact when the cluster centers are in generic position $\la_{\mtx{M}}$ scales like a constant \cite{onehidden}. This theorem focuses on the regime where the noise level $\sigma$ is small. In this case we show that one can achieve good generalization as soon as the number of data points scale with the square of the number classes times the total number of cluster (i.e.~$n\gtrsim K^2C$) which is the effective rank of the M-NTK matrix. We note that this result follows from our main result with random initialization by setting the cutoff level at $\alpha_0^2\sim\frac{\la_{\M}n}{KC}$. This demonstrates that in this model $\bar{\alpha}$ does indeed scale as a constant. Finally, the required network width is independent of $n$ and only depends on $K$ and $C$ specifically we require $k\gtrsim K^8C^4$. This is in stark contrast with \cite{arora2019fine} in the binary case. To the best of understanding \cite{arora2019fine} requires $k\gtrsim \frac{n^8}{\la_{\X}^6}$ which depends on $n$ (in lieu of $K$ and $C$) and the minimum eigenvalue $\la_{\X}$ of the NTK matrix $\bSi(\X)$ (rather than $\la_{\M}$). Furthermore, in this case as $\sigma\rightarrow 0$, $\bSi(\X)$ becomes rank deficient and $\la_{\X}\rightarrow 0$ so that the required width of \cite{arora2019fine} grows to infinity.




\subsection{Prior Art}\vspace{-0pt}\label{priorart}
Neural networks have impressive generalization abilities even when they are trained with more parameters than the size of the dataset \cite{zhang2016understanding}. Thus, optimization and generalization properties of neural networks have been the topic of many recent works \cite{zhang2016understanding}. Below we discuss related work on classical learning theory as well as optimization and implicit bias.

\noindent{\bf{Statistical learning theory:}} Statistical properties of neural networks have been studied since 1990's \cite{anthony2009neural, bartlett1999almost,bartlett1998sample}. With the success of deep networks, there is a renewed interest in understanding capacity of the neural networks under different norm constraints or network architectures \cite{dziugaite2017computing,arora2018stronger,neyshabur2017exploring,golowich2017size}. \cite{Bartlett:2017aa,neyshabur2017pac} established tight sample complexity results for deep networks based on the product of appropriately normalized spectral norms. See also \cite{nagarajan2019deterministic} for improvements via leveraging various properties of the inter-layer Jacobian and \cite{long2019size} for results with convolutional networks. Related, \cite{arora2018stronger} leverages compression techniques for constructing tighter bounds. \cite{yin2018rademacher} jointly studies statistical learning and adversarial robustness. These interesting results, provide generalization guarantees for the optimal solution to the empirical risk minimizer. In contrast, we focus on analyzing the generalization dynamics of gradient descent iterations. 

\noindent{\bf{Properties of gradient descent:}} There is a growing understanding that solutions found by first-order methods such as gradient descent have often favorable properties. Generalization properties of stochastic gradient descent is extensively studied empirically \cite{keskar2016large,hardt2015train,sagun2017empirical,chaudhari2016entropy,hoffer2017train,goel2017learning,goel2018learning}. For linearly separable datasets, \cite{soudry2018implicit,gunasekar2018implicit,brutzkus2017sgd,ji2018gradient,ji2018risk} show that first-order methods find solutions that generalize well without an explicit regularization for logistic regression. An interesting line of work establish connection between kernel methods and neural networks and study the generalization abilities of kernel methods when the model interpolates the training data \cite{dou2019training,belkin2018reconciling,Belkin:2018aa, belkin2019two, Liang:2018aa, Belkin:2018ab}. \cite{chizat2018global,song2018mean,mei2018mean,Sirignano:2018aa, Rotskoff:2018aa} relate the distribution of the network weights to Wasserstein gradient flows using mean field analysis. This literature is focused on asymptotic characterizations rather than finite-size networks.

\noindent{\bf{Global convergence and generalization of neural nets:}} Closer to this work, recent literature \cite{cao2019generalization,arora2019fine,ma2019comparative,allen2018learning} provides generalization bounds for overparameterized networks trained via gradient descent. Also see \cite{li2018visualizing,Huang2019UnderstandingGT} for interesting visualization of the optimization and generalization landscape. Similar to Theorem \ref{gen main}, \cite{arora2019fine} uses the NTK to provide generalization gurantees. \cite{li2019gradient} leverages low-rank Jacobian structure to establish robustness to label noise. These works build on global convergence results of randomly initialized neural networks \cite{du2018gradient, du2018gradient2, allen2018convergence,chizat2018note,zhang2019training,nitanda2019refined,Oymak:2018aa,zou2018stochastic} which study the gradient descent trajectory via comparisons to a a linearized Neural Tangent Kernel (NTK) learning problem. These results however typically require unrealistically wide networks for optimization where the width grows poly-inversely proportional to the distance between the input samples. Example distance measures are class margin for logistic loss and minimum eigenvalue of the kernel matrix for least-squares. Our work circumvents this by allowing a capacity-dependent interpolation. We prove that even rather small networks (e.g.~of constant width) can interpolate the data over a low-dimensional information space without making restrictive assumptions on the input. This approach also leads to faster convergence rates. In terms of generalization, our work has three distinguishing features: (a) bias-variance tradeoffs by identifying information/nuisance spaces, (b) no margin/distance/minimum eigenvalue assumptions on data, (c) the bounds apply to multiclass classification as well as pre-trained networks (Theorem \ref{nn deter gen}).

\section{Numerical experiments}\label{numeric sec}
\textbf{Experimental setup.}  We present experiments supporting our theoretical findings on the CIFAR-10 dataset, which consists of $50k$ training images  and $10k$ test images in $10$ classes. For our experiments, we reduced the number of classes to $3$ (automobile, airplane, bird) and subsampled the training data such that each class is represented by $3333$ images ($9999$ in total). This is due to the fact that calculating the full spectrum of the Jacobian matrix over the entire data set is computationally intensive\footnote{We plan to perform more comprehensive set of experiments by calculating the Jacobian spectrum in a distributed manner.}. For testing, we used all examples of the $3$ classes ($3000$ in total). In all of our experiments we set the information space to be the span of the top 50 singular vectors (out of total dimension of $\K n\approx 30000$).

We demonstrate our results on ResNet20, a state-of-the-art architecture with a fairly low test error on this dataset ($8.75\%$ test error reported on $10$ classes) and relatively few parameters ($0.27M$). In order to be consistent with our theoretical formulation we made the following modifications to the default architecture: (1) we turned off batch normalization and (2) we did not pass the network output through a soft-max function. We trained the network using a least-squares loss with SGD with batch size $128$ and standard data augmentation (e.g.~random crop and flip). We set the initial learning rate to $0.01$ and adjusted the learning rate schedule and number of epochs depending on the particular experiment so as to achieve a good fit to the training data quickly. The figures in this section depict the minimum error over a window consisting of the last 10 epochs for visual clarity. We also conducted two sets of experiments to illustrate the results on uncorrupted and corrupted data.


\noindent\textbf{Experiments without label corruption.} First, we present experiments on the original training data described above with no label corruption. We train the network to fit to the training data by using $400$ epochs and decreasing the learning rate at $260$ and $360$ epochs by a factor of $10$. 
\begin{figure}
\centering
		\includegraphics[scale=0.9]{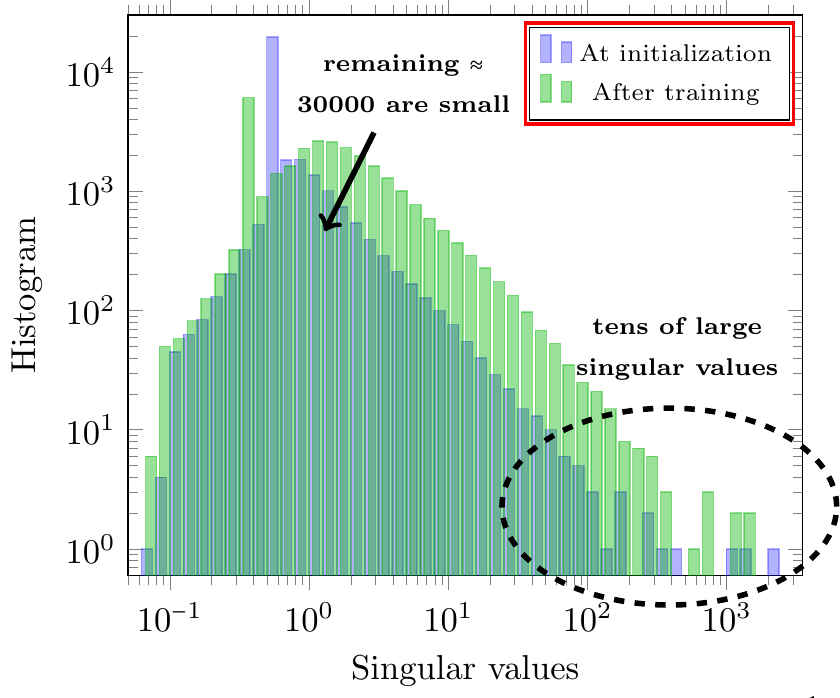}
		\captionof{figure}{Histogram of the singular values of the initial and final Jacobian of the neural network during training.}\label{exp:fig_hist}
		\end{figure}
		
In Figure \ref{exp:fig_hist} we plot the histogram of the eigenvalues of the Jacobian calculated on the training data at initialization and after training. This figure clearly demonstrates that the Jacobian has low-rank structure as there are tens of large singular values with the remaining majority of the spectrum consisting of small singular values. This observation serves as a natural basis for decomposition of the label space into the information space $\calF$ (large singular values, low-dimensional) and nuisance space $\calS$ (small singular values, high-dimensional). 

\begin{table}
\centering
{
		\centering
		\begin{tabular}{|l||c|c|c||c|c|r|}
			\hline
			&$\frac{\twonorm{\Pi_{\calF}( \y)}}{\twonorm{\y}}$ & 
			$\frac{\twonorm{\Pi_{\calS}( \y)}}{\twonorm{\y}}$ & 
			$\frac{\twonorm{\mtx{J}_{\calF}^{\dagger}\y}}{\twonorm{\y}}$ & 
			$\frac{\twonorm{\Pi_{\calF}( \vct{r}_0)}}{\twonorm{\vct{r}_0}}$  & 
			$\frac{\twonorm{\Pi_{\calS}( \vct{r}_0)}}{\twonorm{\vct{r}_0}}$ & 
			$\frac{\twonorm{\mtx{J}_{\calF}^{\dagger}\vct{r}_0}}{\twonorm{\vct{r}_0}}$ \\
			\hlineB{2.5}
			\textbf{$\mtx{J}_{init}$} & 0.724 &  0.690 & $5.44\cdot 10^{-3}$&  0.886  & 0.465& $4.10\cdot 10^{-3}$\\
			\hline
			\textbf{$\mtx{J}_{final}$} & 0.987 &  0.158 &  $3.16\cdot 10^{-3}$&  0.976  & 0.217&$3.43\cdot 10^{-3}$\\
			\hline
		\end{tabular}
	}
	\captionof{table}{Depiction of the alignment of the initial label/residual with the information/nuisance space using uncorrupted data and a Multi-class ResNet20 model trained with SGD.}\label{exp:table_nolc}
\end{table}

Our theory predicts that the sum of $\twonorm{\mtx{J}_{\calF}^{\dagger}\y}$ and $\twonorm{\Pi_{\calS}\left(\y\right)}$ determines the classification error (Theorem \ref{gen main}). Table \ref{exp:table_nolc} collects these values for the initial and final Jacobian. These values demonstrate that the label vector is indeed correlated with the top eigenvectors of both the initial and final Jacobians. An interesting aspect of these results is that this correlation increases from the initial to the final Jacobian so that more of the label energy lies on the information space of the final Jacobian in comparison with the initial Jacobian. Stated differently, we observe a significant adaptation of the Jacobian to the labels after training compared to the initial Jacobian so that our predictions become more and more accurate as the iterates progress. In particular, Table \ref{exp:table_nolc} shows that more of the energy of both labels and initial residual $\rb_0$ lies on the information space of the Jacobian after training. Consequentially, less energy falls on the nuisance space, while $\twonorm{\mtx{J}_{\calF}^{\dagger}\y}$ remains relatively small resulting in better generalization. 

\begin{figure}[t]
	\begin{subfigure}[b]{0.48\textwidth}
	\centering
	\includegraphics[scale=1.2]{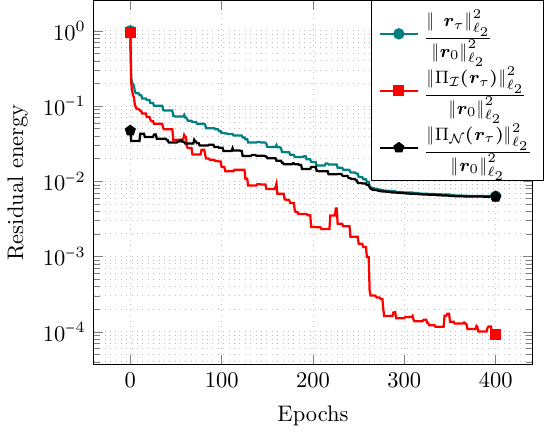}
		\caption{Residual along the information and nuisance spaces of the final Jacobian ${\mtx{J}}_{\text{final}}$ using training data.}\label{exp:fig_nolc_train}
\end{subfigure}~		
		\begin{subfigure}[b]{0.48\textwidth}
		\centering
		\includegraphics[scale=1.2]{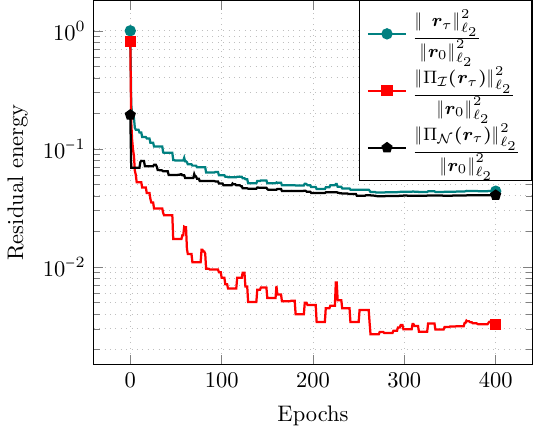}
		\caption{Residual along the information/nuisance spaces of the final Jacobian ${\mtx{J}}_{\text{final}}$ using test data.}      \label{exp:fig_nolc_test}
		\end{subfigure}
		\\
		\begin{center}
		\begin{subfigure}[c]{0.48\textwidth}
		\centering
		\includegraphics[scale=1.2]{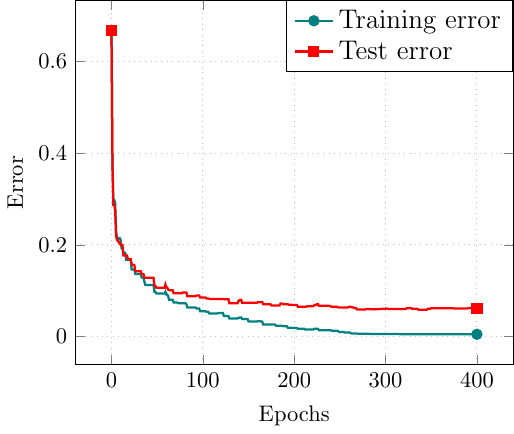}
		\caption{Training and test miss-classification error.}
		\label{exp:fig_nolc_err}
	\end{subfigure}
	\end{center}
	\caption{Evolution of the residual ($\rb_\tau=f(\W_\tau)-\y$) and misclassification error on training and test data without label corruption using SGD.}
	 \label{exp:fig_nolc}
\end{figure}

We also track the projection of the residual $\rb_{\tau}$ on the information and nuisance subspaces throughout training on both training and test data and depict the results in Figures \ref{exp:fig_nolc_train} and \ref{exp:fig_nolc_test}. In agreement with our theory, these plots show that learning on $\calF$ is fast and the residual energy decreases rapidly on this space. On the other hand, residual energy on $\calS$ goes down rather slowly and the decrease in total residual energy is overwhelmingly governed by $\calF$, suggesting that most information relevant to learning lies in this space. We also plot the training and test error in Figure \ref{exp:fig_nolc_err}. We observe that  as learning progresses, the residual on both spaces decrease in tandem with training and test error.

%
%
%
%
%

In our final experiment with uncorrupted data we focus on training the model with an Adam optimizer with a learning rate of $0.001$. We depict the results in Figure \ref{exp:fig_adam}. We observe that due to the built-in learning rate adaptation of Adam, perfect fitting to training data is achieved in fewer iterations compared to SGD. Interestingly, the residual energy on the information space drops significantly faster than in the previous experiment with simple SGD (without Adam). In particular, after $100$ epochs the fraction of the residual on the information space falls below $4 \cdot 10^{-4}$  with Adam ($\twonorm{\Pi_{\mathcal{I}}(\rb_\tau)}^2/\twonorm{\rb_0}^2\le 4 \cdot 10^{-4}$) versus $10^{-2}$ for the SGD on the final Jacobian. This suggests Adam obtains semantically relevant features significantly faster. Moreover, Table \ref{exp:table_nolc_adam} shows that the Jacobian adapts to both the labels and initial residual even faster than SGD on this dataset.

\begin{figure}[t!]
	\begin{subfigure}[t]{0.5\textwidth}
		\centering
		\includegraphics[scale=1.4]{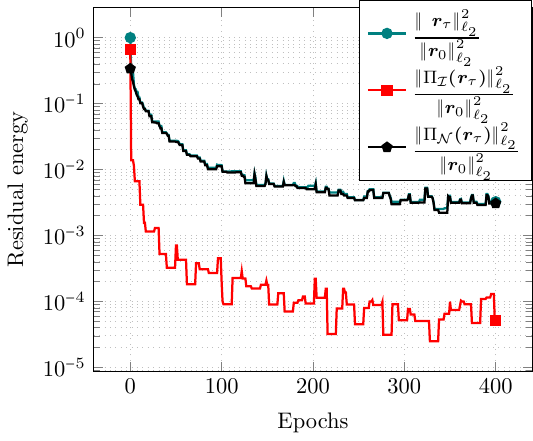}
		\caption{Residual along the information and nuisance spaces of the initial Jacobian ${\mtx{J}}_{\text{init}}$ using training data.}
\label{exp:fig_adam_init}
	\end{subfigure}~~
	\begin{subfigure}[t]{0.5\textwidth}
		\centering
		\includegraphics[scale=1.4]{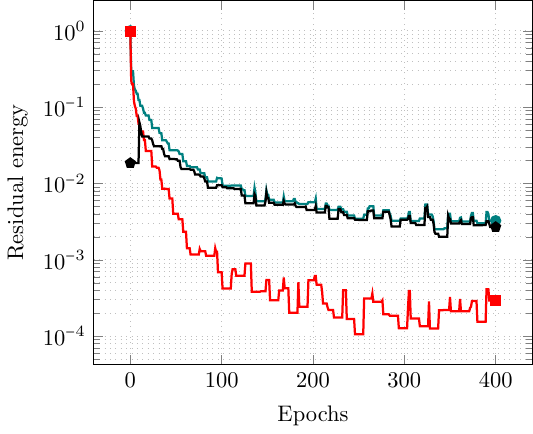}
		\caption{Residual along the information and nuisance spaces of the final Jacobian ${\mtx{J}}_{\text{final}}$ using training data.}\label{exp:fig_adam_final}
	\end{subfigure}\vspace{0.1cm}
	\caption{Evolution of the residual ($\rb_\tau=f(\W_\tau)-\y$) on training data without label corruption using ADAM.}\label{exp:fig_adam}
\end{figure}
\begin{table}
\centering
	{
		\centering
		\begin{tabular}{|l||c|c|c||c|c|r|}
			\hline
			&$\frac{\twonorm{\Pi_{\calF}( \y)}}{\twonorm{\y}}$ & 
			$\frac{\twonorm{\Pi_{\calS}( \y)}}{\twonorm{\y}}$ & 
			$\frac{\twonorm{\mtx{J}_{\calF}^{\dagger}\y}}{\twonorm{\y}}$ & 
			$\frac{\twonorm{\Pi_{\calF}( \vct{r}_0)}}{\twonorm{\vct{r}_0}}$  & 
			$\frac{\twonorm{\Pi_{\calS}( \vct{r}_0)}}{\twonorm{\vct{r}_0}}$ & 
			$\frac{\twonorm{\mtx{J}_{\calF}^{\dagger}\vct{r}_0}}{\twonorm{\vct{r}_0}}$ \\
			\hlineB{2.5}
			\textbf{$\mtx{J}_{init}$} & 0.702 &  0.712 & $5.36\cdot 10^{-3}$&  0.814  & 0.582& $4.43\cdot 10^{-3}$\\
			\hline
			\textbf{$\mtx{J}_{final}$} & 0.997 &  0.078 &  $3.10\cdot 10^{-3}$&  0.991  & 0.136&$3.06\cdot 10^{-3}$\\
			\hline
		\end{tabular}
	}
	\captionof{table}{Depiction of the alignment of the initial label/residual with the information/nuisance space using uncorrupted data and a Multi-class ResNet20 model trained with Adam.}\label{exp:table_nolc_adam}
\end{table}

\noindent\textbf{Experiments with 50\% label corruption.} In our next series of experiments we study the effect of corruption. Specifically, we corrupt $50\%$ of the labels by randomly picking a label from a (strictly) different class. We train the network for $800$ epochs and divide the learning rate by $10$ at $700$ epochs to fit to the training data.

Similar to the uncorrupted case, we track the projection of the residual $\rb_{\tau}$ on the information and nuisance spaces throughout training on both training and test data and depict the results in Figures \ref{exp:fig_lc50_train} and \ref{exp:fig_lc50_test}. We also track the train and test misclassification error in Figure \ref{exp:fig_lc50_err}. From Figure  \ref{exp:fig_lc50_err} it is evident that while the training error steadily decreases, test error exhibits a very different behavior from the uncorrupted experiment. In the first phase, test error drops rapidly as the network learns from information contained in the uncorrupted data, accompanied by a corresponding decrease in residual energy on the information subspace on the training data (Figure \ref{exp:fig_lc50_train}). The lowest test error is observed at $100$ epochs after which a steady increase follows. In the second phase, the network overfits to the corrupted data resulting in larger test error on the uncorrupted test data (Figure \ref{exp:fig_lc50_test}). More importantly, the increase of the test error is due to the nuisance space as the error over information space is stable while it increases over the nuisance space. In particular the residual on $\calS$ slowly increases while residual on $\calF$ drops sharply creating a dip in both test error and total residual energy at approximately $100$ epochs. This phenomenon closely resembles the population loss decomposition of the linear model discussed in Section \ref{linmodel} (see Figure \ref{synthfig}), where we observe a dip in total test error caused by an increasing component along the nuisance space and a simultaneously decreasing component along information space. 



\begin{figure}
\begin{subfigure}[b]{0.5\textwidth}
		\centering
		\includegraphics[scale=1.1]{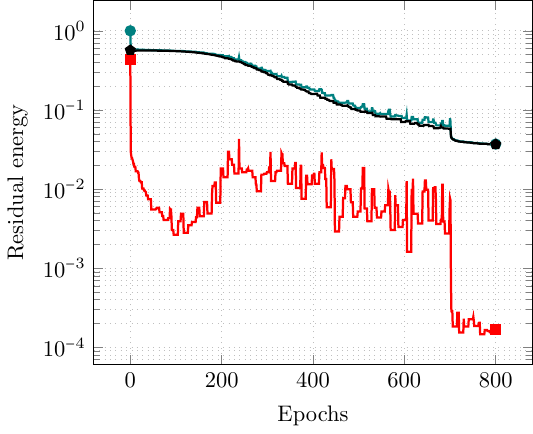}
		\caption{Residual along the info./nuisance spaces of the Jacobian evaluated at 100 epoch (${\mtx{J}}(\w_\tau)$) using training data.}
\label{exp:fig_lc50_train}		
\end{subfigure}~~
\begin{subfigure}[b]{0.5\textwidth}
\centering
\includegraphics[scale=1.1]{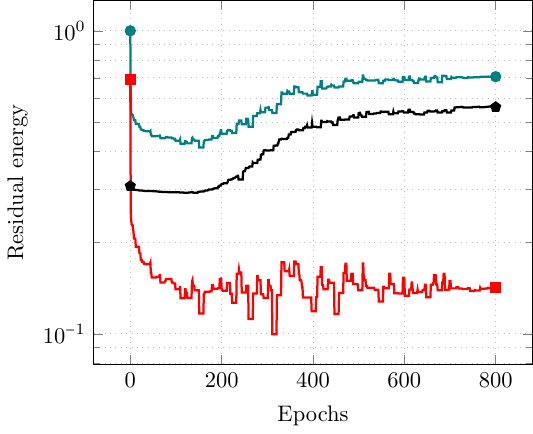}		
		\caption{Residual along the information/nuisance spaces of the Jacobian evaluated at 100 epoch (${\mtx{J}}(\w_\tau)$) using test data.}
\label{exp:fig_lc50_test}
		\end{subfigure}
		\begin{center}
		\begin{subfigure}[b]{0.5\textwidth}
		\centering
		\includegraphics[scale=1.1]{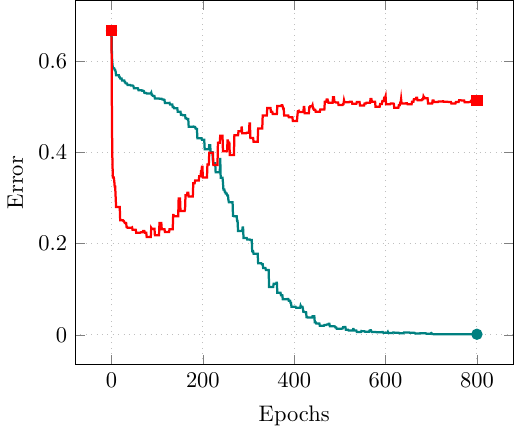}
		\caption{Training and test error}
		\label{exp:fig_lc50_err}
		\end{subfigure}
		\end{center}
		\caption{Evolution of the residual ($\rb_\tau=f(\W_\tau)-\y$) and misclassification error on training and test data with $50\%$ label corruption using SGD.}
		\label{exp:fig_lc50}

\end{figure}


\begin{figure}
		\centering
		\includegraphics[scale=1.4]{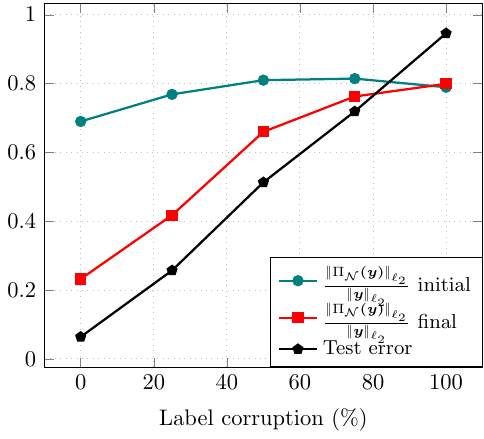}
		\captionof{figure}{Fraction of the energy of the label vector that lies on the nuisance space of the initial Jacobian (cyan with circles) and final Jacobian (red with squares) as we as the test error (black with pentagons)  as a function of the amount of label corruption.}\label{exp:fig_lc_compare}
\end{figure}

In Table \ref{exp:table_lc50} we again depict the fraction of the energy of the labels and the initial residual that lies on the information/nuisance spaces. The Jacobian continues to adapt to the labels/initial residual even in the presence of label corruption, albeit to a smaller degree. We note that due to corruption, labels are less correlated with the information space of the Jacobian and the fraction of the energy on the nuisance space is higher which results in worse generalization (as also predicted by our theory).

\begin{table}[t]
	\centering
	{
		\centering
		\begin{tabular}{|l||c|c|c||c|c|r|}
			\hline
			&$\frac{\twonorm{\Pi_{\calF}( \y)}}{\twonorm{\y}}$ & 
			$\frac{\twonorm{\Pi_{\calS}( \y)}}{\twonorm{\y}}$ & 
			$\frac{\twonorm{\mtx{J}_{\calF}^{\dagger}\y}}{\twonorm{\y}}$ & 
			$\frac{\twonorm{\Pi_{\calF}( \vct{r}_0)}}{\twonorm{\vct{r}_0}}$  & 
			$\frac{\twonorm{\Pi_{\calS}( \vct{r}_0)}}{\twonorm{\vct{r}_0}}$ & 
			$\frac{\twonorm{\mtx{J}_{\calF}^{\dagger}\vct{r}_0}}{\twonorm{\vct{r}_0}}$ \\
			\hlineB{2.5}
			\textbf{$\mtx{J}_{init}$} & 0.587&  0.810 & $1.72\cdot 10^{-3}$&  0.643  & 0.766& $1.98\cdot 10^{-3}$\\
			\hline
			\textbf{$\mtx{J}_{final}$} & 0.751 &  0.660 & $1.87\cdot 10^{-3}$ &  0.763  & 0.646& $1.20\cdot 10^{-3}$\\
			\hline
		\end{tabular}
	}
	\captionof{table}{Depiction of the alignment of the initial label/residual with the information/nuisance space using 50\% label corrupted data and a Multi-class ResNet20 model trained with SGD.}\label{exp:table_lc50}
\end{table}

In order to demonstrate the connection between generalization error and information/nuisance spaces of the Jacobian, we repeat the experiment with $25\%$, $75\%$ and $100\%$ label corruption and depict the results after $800$ epochs in Figure~\ref{exp:fig_lc_compare}. As expected, the test error increases with the corruption level.  Furthermore, the corrupted labels become less correlated with the information space, with more of the label energy falling onto the nuisance space. This is consistent with our theory which predicts worse generalization in this case.


\section{Technical approach and General Theory}
In this section, we outline our approach to proving robustness of over-parameterized neural networks. Towards this goal, we consider a general formulation where we aim to fit a general nonlinear model of the form $\vct{x}\mapsto f(\vct{x};\vct{\theta})$ with $\vct{x}\in\R^d$ denoting the input features, $\vct{\theta}\in\R^p$ denoting the parameters, and $f(\vct{x};\vct{\theta})\in\R^K$ the $K$ outputs of the model denoted by $f_1(\vct{x};\vct{\theta}), f_2(\vct{x};\vct{\theta}), \ldots, f_K(\vct{x};\vct{\theta})$. For instance in the case of neural networks $\vct{\theta}$ represents its weights. Given a data set of $n$ input/label pairs $\{(\vct{x}_i,\vct{y}_i)\}_{i=1}^n\subset \R^d\times \R^K$, we fit to this data by minimizing a nonlinear least-squares loss of the form
\begin{align}
\label{noncompact}
\Lc(\bteta)=\frac{1}{2}\sum_{i=1}^n \twonorm{f(\x_i;\bteta)-\y_i}^2.
\end{align}
To continue let us first aggregate the predictions and labels into larger vectors based on class. In particular define
\begin{align*}
f_\ell(\bteta)=\begin{bmatrix}f_\ell(\x_1;\bteta)\\\vdots\\f_\ell(\x_n;\bteta)\end{bmatrix}\in\R^{n}\quad\text{and}\quad\y^{(\ell)}=\begin{bmatrix}{(\y_1)}_\ell\\\vdots\\ {(\y_n)}_\ell\end{bmatrix}\in\R^{n}\quad\text{for}\quad \ell=1,2,\ldots,K.
\end{align*}
Concatenating these vectors we arrive at
\begin{align}
\label{model}
f(\bteta)=\begin{bmatrix}f_1(\bteta)\\\vdots\\f_\K(\bteta)\end{bmatrix}\in\R^{\K n}\quad\text{and}\quad\y=\begin{bmatrix}\y^{(1)}\\\vdots\\\y^{(\K)}\end{bmatrix}\in\R^{\K n}.
\end{align}
Using the latter we can rewrite the optimization problem \eqref{noncompact} into the more compact form
\begin{align}
\label{compact}
\Lc(\bteta)=\frac{1}{2}\twonorm{f(\bteta)-\y}^2.
\end{align}
To solve this problem we run gradient descent iterations with a learning rate $\eta$ starting from an initial point $\vct{\theta}_0$. These iterations take the form
\begin{align}
\bteta_{\tau+1}=\bteta_\tau-\eta \grad{\bteta_\tau}\quad\text{with}\quad\grad{\bteta}=\mathcal{J}^T(\vct{\theta})\left(f(\bteta)-\y\right).\label{nonlin gd}
\end{align}
As mentioned earlier due to the form of the gradient the convergence/generalization of gradient descent naturally depends on the spectral properties of the Jacobian. To capture these spectral properties we will use a reference Jacobian $\Jb$ (formally defined below) that is close to the Jacobian at initialization $\mathcal{J}(\bteta_0)$. %
\begin{definition} [Reference Jacobian and its SVD]\label{tjac} Consider an initial point $\bteta_0\in\R^p$ and the Jacobian mapping $\mathcal{J}(\bteta_0)\in\R^{\K n\times p}$. For $\eps_0,\beta>0$, we call $\Jb\in\R^{\K n\times \max(\K n,p)}$ an $(\eps_0,\beta)$ reference Jacobian matrix if it obeys the following conditions,
\[
\opnorm{\Jb}\leq\beta,\quad\opnorm{\Jc(\bteta_0)\Jc^T(\bteta_0)-\Jb\Jb^T}\leq \eps_0^2,\quad\text{and}\quad\opnorm{\Jcb(\bteta_0)-\Jb}\leq \eps_0.
\]
where $\Jcb(\bteta_0)\in \R^{\K n\times \max(\K n,p)}$ is a matrix obtained by augmenting $\Jc(\bteta_0)$ with $\max(0,\K n-p)$ zero columns. Furthermore, consider the singular value decomposition of $\Jb$ given by
\begin{align}
\label{eigdecompJ}
\Jb=\Ub\text{diag}(\vct{\la})\Vb^T=\sum_{s=1}^{\K n} \la_s\ub_s\vb_s^T.
\end{align}
where $\vct{\la}\in\R^{\K n}$ are the vector of singular values and $\ub_s\in\R^{Kn}$ and $\vb_s\in\R^p$ are the left/right singular vectors. 
\end{definition}
One natural choice for this reference Jacobian is $\Jb=\Jcb(\bteta_0)$. However, we shall also use other reference Jacobians in our results. We will compare the gradient iterations \eqref{nonlin gd} to the iterations associated with fitting a linearized model around $\bteta_0$ defined as $f_{\text{lin}}(\bbteta)=f(\bteta_0)+\Jb(\bbteta-\brteta_0)$, where $\brteta_0\in\R^{\max\left(Kn,p\right)}$ is obtained from $\bteta_0$ by adding $\max(\K n-p,0)$ zero entries at the end of $\bteta_0$. The optimization problem for fitting the linearized problem has the form
\begin{align}
\label{compactlin}
\mathcal{L}_{lin}(\bteta)=\frac{1}{2}\twonorm{f_{\text{lin}}(\bteta)-\y}^2.
\end{align}
Thus starting from $\bbteta_0=\brteta_0$ the iterates $\bbteta_{\tau}$ on the linearized problem take the form
\begin{align}
\label{lingd}
\bbteta_{\tau+1}&=\bbteta_{\tau}-\eta \nabla\mathcal{L}_{lin}(\bbteta_\tau),\\
&=\bbteta_{\tau}-\eta \Jb^T(f(\bteta_0)+\Jb(\bbteta_\tau-\bteta_0)-\y),\nn\\
&=\bbteta_{\tau}-\eta \Jb^T\Jb(\bbteta_\tau-\brteta_0)-\eta \Jb^T\left(f(\bteta_0)-\y\right).\nn
\end{align}
The iterates based on the linearized problem will provide a useful reference to keep track of the evolution of the original iterates \eqref{nonlin gd}. Specifically we study the evolution of misfit/residuals associated with the two problems
\begin{align}
&\text{Original residual: }\rb_\tau=f(\bteta_\tau)-\y.\label{org res eq}\\
&\text{Linearized residual: }\rbb_\tau=f_{\text{lin}}(\bbteta_\tau)-\y=(\Iden-\eta\Jb\Jb^T)^\tau \rb_0.\label{lin res eq}
\end{align}
To better understand the dynamics of convergence of the linearized iterates next we define two subspaces associated with the reference Jacobian and its spectrum.
\begin{definition}[Information/Nuisance Subspaces]\label{tjac2} Let $\mtx{J}$ denote the reference Jacobian per Definition \ref{tjac} with eigenvalue decomposition $\mtx{J}=\mtx{U}$diag$(\vct{\lambda})\mtx{V}^T$ per \eqref{eigdecompJ}. For a spectrum cutoff $\alpha$ obeying $0\le \alpha\le\lambda_1$ let $r(\alpha)$ denote the index of the smallest singular value above the threshold $\alpha$, that is,
\begin{align*}
r(\alpha)=\min\left(\{s\in\{1,2,\ldots,nK\}\quad \text{such that}\quad\lambda_s\ge \alpha \}\right).
\end{align*}
We define the information and nuisance subspaces associated with $\mtx{J}$ as $\calF:=\text{span}(\{\ub_s\}_{s=1}^r)$ and $\calS:=\text{span}(\{\ub_s\}_{s=r+1}^{\K n})$. We also define the truncated reference Jacobian
\begin{align*}
\mtx{J}_{\calF}=\begin{bmatrix}\vct{u}_1 & \vct{u}_2 & \ldots & \vct{u}_r\end{bmatrix}\text{diag}\left(\lambda_1,\lambda_2,\ldots,\lambda_r\right)\begin{bmatrix}\vct{v}_1 & \vct{v}_2 & \ldots & \vct{v}_r\end{bmatrix}^T
\end{align*}
which is the part of the reference Jacobian that acts on the information subspace $\calF$.
\end{definition}
We will show rigorously that the information and nuisance subspaces associated with the reference Jacobian dictate the directions where learning is fast and generalizable versus the directions where learning is slow and overfitting occurs.
Before we make this precise we list two assumptions that will be utilized in our result.
\begin{assumption}[Bounded spectrum]\label{ass2} For any $\bteta\in\R^p$ the Jacobian mapping associated with the nonlinearity $f:\R^p\mapsto \R^n$ has bounded spectrum, i.e.~$\|\Jc(\bteta)\|\leq \beta$. 
\end{assumption}
\begin{assumption}[Bounded perturbation] \label{ass3} Consider a point $\bteta_0\in\R^p$ and positive scalars $\eps,R>0$. Assume that for any $\bteta$ obeying $\tn{\bteta-\bteta_0}\leq R$, we have
\[
\|\Jc(\bteta)-\Jc(\bteta_0)\|\leq \frac{\eps}{2}.
\]
\end{assumption}
With these assumptions in place we are now ready to discuss our meta theorem that demonstrates that the misfit/residuals associated to the original and linearized iterates do in fact track each other rather closely.
\begin{theorem}[Meta Theorem]\label{many step thm}
Consider a nonlinear least squares problem of the form $\mathcal{L}(\bteta)=\frac{1}{2}\twonorm{f(\bteta)-\y}^2$ with $f:\R^p\mapsto \R^{nK}$ the multi-class nonlinear mapping, $\bteta\in\R^p$ the parameters of the model, and $\y\in\R^{nK}$ the concatenated labels as in \eqref{model}. Let $\brteta$ be zero-padding of $\bteta$ till size $\max(\K n,p)$. Also, consider a point $\vct{\theta}_0\in\R^p$ with $\mtx{J}$ an $(\epsilon_0,\beta)$ reference Jacobian associated with $\mathcal{J}(\vct{\theta}_0)$ per Definition \ref{tjac} and fitting the linearized problem $f_{\text{lin}}(\bbteta)=f(\bteta_0)+\Jb(\bbteta-\brteta_0)$ via the loss $\mathcal{L}_{lin}(\bteta)=\frac{1}{2}\twonorm{f_{\text{lin}}(\bteta)-\y}^2$. Furthermore, define the information $\calF$ and nuisance $\calS$ subspaces and the truncated Jacobian $\mtx{J}_{\calF}$ associated with the reference Jacobian $\mtx{J}$ based on a cut-off spectrum value of $\alpha$ per Definition \ref{tjac2}. Furthermore, assume the Jacobian mapping $\mathcal{J}(\bteta)\in\R^{nK\times p}$ associated with $f$ obeys Assumptions \ref{ass2} and \ref{ass3} for all $\bteta\in\R^p$ obeying
\begin{align}
\label{assumpgen}
\twonorm{\bteta-\bteta_0}\le R:=2\left(\twonorm{\mtx{J}_{\calF}^{\dagger}\vct{r}_0}+\frac{\Gamma}{\alpha}\twonorm{\Pi_{\calS}\left(\vct{r}_0\right)}+\delta\frac{\Gamma}{\alpha}\twonorm{\vct{r}_0}\right),
\end{align}  
around a point $\bteta_0\in\R^p$ for a tolerance level $\delta$ obeying $0< \delta \le 1$ and stopping time $\Gamma$ obeying $\Gamma\ge 1$. Finally, assume the following inequalities hold
\begin{align}
 \eps_0\leq \frac{\min(\delta\alpha,\sqrt{\delta\alpha^3/\Gamma\beta})}{5}\quad\text{and}\quad\eps\leq \frac{\delta\alpha^3}{5\Gamma\beta^2}.\label{cndd}
\end{align}
We run gradient descent iterations of the form $\bteta_{\tau+1}=\bteta_{\tau}-\eta\nabla \mathcal{L}(\bteta_\tau)$ and $\bbteta_{\tau+1}=\bbteta_{\tau}-\eta \nabla\mathcal{L}_{lin}(\bbteta_\tau)$ on the original and linearized problems starting from $\bteta_0$ with step size $\eta$ obeying $\eta\le 1/\beta^2$. Then for all iterates $\tau$ obeying $0\le \tau\le T:=\frac{\Gamma}{\eta\alpha^2}$ the iterates of the original ($\bteta_\tau$) and linearized ($\bbteta_\tau$) problems and the corresponding residuals $\rb_\tau:=f(\bteta_\tau)-\y$ and $\rbb_\tau:=f_{\text{lin}}(\bbteta_\tau)-\y$  closely track each other. That is,
\begin{align}
\label{main res eq1}
&\tn{\rb_\tau-\rbb_\tau}\leq \frac{3}{5} \frac{\delta\alpha}{\beta}\tn{\rb_0}\quad\text{and}\quad\tn{\brteta_\tau-\bbteta_\tau}\leq\delta\frac{\Gamma}{\alpha}\tn{\rb_0}
\end{align}
Furthermore, for all iterates $\tau$ obeying $0\le \tau\le T:=\frac{\Gamma}{\eta\alpha^2}$
\begin{align}
\tn{\bteta_\tau-\bteta_0}\leq \frac{R}{2}=\twonorm{\mtx{J}_{\calF}^{\dagger}\vct{r}_0}+\frac{\Gamma}{\alpha}\twonorm{\Pi_{\calS}\left(\vct{r}_0\right)}+\delta\frac{\Gamma}{\alpha}\twonorm{\vct{r}_0}.\label{main res eq2}
\end{align}
and after $\tau=T$ iteration we have
\begin{align}
\twonorm{\rb_{T}}\leq   e^{-\Gamma}\twonorm{\Pi_{\calF}(\rb_0)}+\twonorm{ \Pi_{\calS}(\rb_0)}+\frac{\delta\alpha}{\beta}\tn{\rb_0}.\label{eqqq2}
\end{align}
\end{theorem}

\section{Proofs}
Before we proceed with the proof let us briefly discuss some notation used throughout. For a matrix $\mtx{W}\in\R^{k\times d}$ we use vect$(\mtx{W})\in\R^{kd}$ to denote a vector obtained by concatenating the rows $\vct{w}_1,\vct{w}_2,\ldots,\vct{w}_k\in\R^d$ of $\mtx{W}$. That is,
$\text{vect}(\W)=\begin{bmatrix}\vct{w}_1^T&\vct{w}_2^T &\ldots &\vct{w}_k^T\end{bmatrix}^T$. Similarly, we use $\mat(\vct{w})\in\R^{k\times d}$ to denote a $k\times d$ matrix obtained by reshaping the vector $\vct{w}\in\R^{kd}$ across its rows. Throughout, for a differentiable function $\phi:\R\mapsto\R$ we use $\phi'$ and $\phi''$ to denote the first and second derivative. 
\subsection{Proofs for General Theory (Proof of Theorem \ref{many step thm})}
In this section we prove our result for general nonlinearities. We begin with a few notations and definitions and preliminary lemmas in Section \ref{prelim5}. Next in Section \ref{key5} we prove some key lemmas regarding the evolution of the linearized residuals $\rbb_\tau$. In Section \ref{radam} we establish some key Rademacher complexity results used in our generalization bounds. Finally, in Section \ref{comp5} we use these results to complete the proof of Theorem \ref{many step thm}.
\subsubsection{Preliminary definitions and lemmas}
\label{prelim5}
Throughout we use 
\begin{align*}
\mtx{U}_{\mathcal{I}}=\begin{bmatrix}\vct{u}_1 & \vct{u}_2 & \ldots & \vct{u}_r\end{bmatrix}\in\R^{nK\times r}
\quad\text{and}\quad
\mtx{U}_{\mathcal{N}}=\begin{bmatrix}\vct{u}_{r+1} & \vct{u}_{r+2} & \ldots & \vct{u}_{nK}
\end{bmatrix}\in\R^{nK\times (nK-r)}. 
\end{align*}
to denote the basis matrices for the information and nuisance subspaces from Definition \ref{tjac2}. Similarly, we define the information and nuisance spectrum as 
\begin{align*}
\vct{\lambda}_{\mathcal{I}}=
\begin{bmatrix}
\lambda_1&
\lambda_2&
\ldots&
\lambda_r
\end{bmatrix}^T
\quad\text{and}\quad
\vct{\lambda}_{\mathcal{N}}=
\begin{bmatrix}
\lambda_{r+1}&
\lambda_{r+2}&
\ldots&
\lambda_{nK}
\end{bmatrix}^T.
\end{align*}
We also define the diagonal matrices 
\begin{align*}
\mtx{\Lambda}=\text{diag}(\vct{\lambda}),\quad \mtx{\Lambda}_{\calF}=\text{diag}(\vct{\lambda}_{\calF}),\quad\text{and}\quad\mtx{\Lambda}_{\calS}=\text{diag}(\vct{\lambda}_{\calS}).
\end{align*}
\begin{definition}[early stopping value and distance]\label{earlyval} Consider Definition \ref{tjac2} and let $\Gamma>0$ be a positive scalar. Associated with the initial residual $\rb_0=f(\bteta_0)-\y$ and the information/nuisance subspaces of the reference Jacobian $\mtx{J}$ (with a cut-off level $\alpha$) we define the $(\alpha,\Gamma)$ early stopping value as
\begin{align}
\upp=&\left(\sum_{s=1}^r  \frac{\alpha^2}{\la_s^2}\left(\langle\vct{u}_s,\vct{r}_0\rangle\right)^2+\Gamma^2\sum_{s=r+1}^{nK}\frac{\la_s^2}{\alpha^2} \left(\langle\vct{u}_s,\vct{r}_0\rangle\right)^2\right)^{1/2}.\label{pinv def2}
\end{align}
We also define the early stopping distance as
\begin{align*}
\dpp=\frac{\upp}{\alpha}.
\end{align*}
\end{definition}
The goal of early stopping value/distance is understanding the behavior of the algorithm at a particular stopping time that depends on $\Gamma$ and the spectrum cutoff $\alpha$. In particular, as we will see later on the early stopping distance characterizes the distance from initialization at an appropriate early stopping time. We continue by stating and proving a few simple lemmas. The first Lemma provides upper/lower bounds on the early stopping value.
\begin{lemma}[Bounds on Early-Stopping Value]\label{bndearlystopval} The early stopping value $\upp$ from Definition \ref{earlyval} obeys
\begin{align}
&\upp\leq  \left(\tn{\Pi_{\Rc}(\rb_0)}^2+{\Gamma}^2\tn{\Pi_{\Rcb}(\rb_0)}^2\right)^{1/2}\le \Gamma\twonorm{\rb_0}\label{upp up bound}\\
&\upp\geq \frac{\alpha}{\la_1}\twonorm{\Pi_{\Rc}(\rb_0)}.\label{low bound}
\end{align}
\end{lemma}
\begin{proof}
To prove the upper bound we use the fact that $\alpha\le \la_s$ for $s\le r$ and $\alpha\geq \la_s$ for $s\geq r$ to conclude that
\begin{align*}
\upp &\le \left(\sum_{s=1}^r \left(\langle\vct{u}_s,\vct{r}_0\rangle\right)^2+\Gamma^2\sum_{s=r+1}^{nK} \left(\langle\vct{u}_s,\vct{r}_0\rangle\right)^2\right)^{1/2}\\
&= \left(\twonorm{\Pi_{\Rc}(\rb_0)}^2+{\Gamma}^2\twonorm{\Pi_{\Rcb}(\rb_0)}^2\right)^{1/2}\\
&\leq \Gamma \tn{\rb_0}.
\end{align*}
To prove the lower bound, we use the facts that $\alpha^2/\la_s^2\geq \alpha^2/\la_1^2$ to conclude that
\begin{align*}
\upp=&\left(\sum_{s=1}^r  \frac{\alpha^2}{\la_s^2}\left(\langle\vct{u}_s,\vct{r}_0\rangle\right)^2+\Gamma^2\sum_{s=r+1}^{nK}\frac{\la_s^2}{\alpha^2} \left(\langle\vct{u}_s,\vct{r}_0\rangle\right)^2\right)^{1/2},\\
\ge&\left(\sum_{s=1}^r  \frac{\alpha^2}{\la_s^2}\left(\langle\vct{u}_s,\vct{r}_0\rangle\right)^2\right)^{1/2},\\
\ge&\frac{\alpha}{\lambda_1}\twonorm{\Pi_{\Rc}(\rb_0)}.
\end{align*}
\end{proof}
It is of course well known that the mapping $\left(\mtx{I}-\eta\A\A^T\right)$ is a contraction for sufficiently small values of $\eta$. The next lemma shows that if we replace one of the $\A$ matrices with a matrix $\B$ which is close to $\A$ the resulting matrix $\left(\mtx{I}-\eta\A\B^T\right)$, while may not be contractive, is not too expansive.
\begin{lemma} [Asymmetric PSD increase]\label{asym pert} Let $\A,\B\in\R^{n\times p}$ be matrices obeying 
\begin{align*}
\opnorm{\A}\le \beta,\quad\opnorm{\B}\le \beta,\quad\text{and}\quad\opnorm{\B-\A}\leq \eps.
\end{align*}
Then, for all $\rb\in\R^n$ and $\eta\leq 1/\beta^2$ we have
\[
\twonorm{\left(\Iden-\eta \A\B^T\right)\rb}\le \left(1+\eta\eps^2\right)\twonorm{\rb}.
\]
\end{lemma}
\begin{proof} Note that using $\eta\leq 1/\beta^2$ and $\opnorm{\B-\A}\leq \eps$ we conclude that
\begin{align*}
\twonorm{\left(\Iden-\eta \A\B^T\right)\rb}^2=&\twonorm{\left(\Iden-\eta \B\B^T-\eta(\A-\B)\B^T\right)\rb}^2\\
=&\twonorm{\rb-\eta(\A-\B+\B)\B^T\rb}^2\\
=&\twonorm{\rb}^2-2\eta\rb^T(\A-\B+\B)\B^T\rb+\eta^2\twonorm{\A\B^T\rb}^2\\
\le& \twonorm{\rb}^2-2\eta\tn{\B^T\rb}^2+2\eta \twonorm{(\A-\B)^T\rb}\tn{\B^T\rb}+\eta^2\|\A\|^2\tn{\B^T\rb}^2\\
=& \twonorm{\rb}^2-\eta\tn{\B^T\rb}^2+2\eta \twonorm{(\A-\B)^T\rb}\tn{\B^T\rb}+\left(\eta^2\|\A\|^2\tn{\B^T\rb}^2-\eta\tn{\B^T\rb}^2\right)\\
\overset{\eta\le 1/\beta^2}{\leq}&\twonorm{\rb}^2-\eta\tn{\B^T\rb}^2+2\eta \twonorm{(\A-\B)^T\rb}\tn{\B^T\rb}\\
\overset{\opnorm{\A-\B}\le \epsilon}{\le} & \twonorm{\rb}^2-\eta\tn{\B^T\rb}^2+2\eta \eps\tn{\B^T\rb}\tn{\rb}\\
=&(1+\eta\eps^2)\twonorm{\rb}^2-\eta\left(\eps\twonorm{\rb}-\twonorm{\B^T\rb}\right)^2\\
\le&(1+\eta\eps^2)\twonorm{\rb}^2,
\end{align*}
completing the proof.
\end{proof}
The next lemma shows that if two PSD matrices are close to each other then an appropriate square root of these matrices will also be close.
\begin{lemma}\label{psd pert} Let $\A$ and $\B$ be $n\times n$ positive semi-definite matrices satisfying $\|\A-\B\|\leq \alpha^2$ for a scalar $\alpha\ge 0$. Then for any $\X\in\R^{n\times p}$ with $p\geq n$ obeying $\A=\X\X^T$, there exists a matrix $\Y\in\R^{n\times p}$ obeying $\B=\Y\Y^T$ such that
\[
\|\Y-\X\|\leq 2\alpha
\]
\end{lemma}
\begin{proof}
First we note that for any two PSD matrices $\A_+, \B_+\in\R^{n\times n}$ obeying $\A_+,\B_+\succeq \frac{\alpha^2}{4}\Iden_n$, Lemma 2.2 of \cite{schmitt1992perturbation} guarantees that
\[
\opnorm{\A_+^{1/2}-\B_+^{1/2}}\leq \frac{\|\A_+-\B_+\|}{\alpha}.
\]
In the above for a PSD matrix $\A\in\R^{n\times n}$ with an eigenvalue decomposition $\A=\mtx{U}\mtx{\Lambda}\mtx{U}^T$ we use $\A^{1/2}$ to denote the square root of the matrix given by $\A=\mtx{U}\mtx{\Lambda}^{1/2}\mtx{U}^T$. We shall use this result with $\A_+=\A+\frac{\alpha^2}{4}\Iden_n$ and $\B_+=\B+\frac{\alpha^2}{4}\Iden_n$ to conclude that
\[
\opnorm{\A_+^{1/2}-\B_+^{1/2}}\leq \frac{\opnorm{\A_+-\B_+}}{\alpha}=\frac{\opnorm{\A-\B}}{\alpha}.
\]
Furthermore, using the fact that the eigenvalues of $\A_+$ and $\B_+$ are just shifted versions of the eigenvalues of $\A$ and $\B$ by $\alpha^2/4$ we can conclude that
\[
\opnorm{\A_+^{1/2}-\A^{1/2}}\le \frac{\alpha}{2}\quad\text{and}\quad\opnorm{\B_+^{1/2}-\B^{1/2}}\leq\frac{\alpha}{2}.
\]
Combining the latter two inequalities with the assumption that $\opnorm{\A-\B}\le \alpha^2$ we conclude that
\begin{align}
\label{tempkey1}
\opnorm{\A^{1/2}-\B^{1/2}}\le& \opnorm{\A_+^{1/2}-\B_+^{1/2}}+\opnorm{\A_+^{1/2}-\A^{1/2}}+\opnorm{\B_+^{1/2}-\B^{1/2}}\nn\\
\le&\frac{\opnorm{\A-\B}}{\alpha}+\frac{\alpha}{2}+\frac{\alpha}{2}\nn\\
\le&2\alpha.
\end{align}
Suppose $p\geq n$ and assume the matrices $\A$ and $\B$ have eigenvalue decompositions given by $\A=\Ub_{\A}\La_{\A}\Ub_{\A}^T$ and $\B=\Ub_{\B}\La_{\B}\Ub_{\B}^T$. Then, any $\X\in\R^{n\times p}$ with $p\ge n$ has the form $\X=\mtx{U}_{\A}\La_{\A}^{1/2}\mtx{V}_{\A}^T$ with $\mtx{V}_{\A}\in\R^{p\times n}$ an orthonormal matrix. Now pick
\begin{align*}
\Y= \Ub_{\B}\La_{\B}^{1/2}\Ub_{\B}^T\mtx{U}_{\A}\mtx{V}_{\A}^T.
\end{align*}
Then clearly $\Y\Y^T=\B$. Furthermore, we have
\begin{align*}
\opnorm{\X-\Y}=&\opnorm{\mtx{U}_{\A}\La_{\A}^{1/2}\mtx{V}_{\A}^T-\Ub_{\B}\La_{\B}^{1/2}\Ub_{\B}^T\mtx{U}_{\A}\mtx{V}_{\A}^T}\\
=&\opnorm{\mtx{U}_{\A}\La_{\A}^{1/2}\mtx{U}_{\A}^T\mtx{U}_{\A}\mtx{V}_{\A}^T-\Ub_{\B}\La_{\B}^{1/2}\Ub_{\B}^T\mtx{U}_{\A}\mtx{V}_{\A}^T}\\
=&\opnorm{\left(\mtx{U}_{\A}\La_{\A}^{1/2}\mtx{U}_{\A}^T-\Ub_{\B}\La_{\B}^{1/2}\Ub_{\B}^T\right)\mtx{U}_{\A}\mtx{V}_{\A}^T}\\
=&\opnorm{\left(\A^{1/2}-\B^{1/2}\right)\mtx{U}_{\A}\mtx{V}_{\A}^T}\\
=&\opnorm{\A^{1/2}-\B^{1/2}}.
\end{align*}
Combining the latter with \eqref{tempkey1} completes the proof. 
\end{proof}

\subsubsection{Key lemmas for general nonlinearities}
\label{key5}
Throughout this section we assume $\Jb$ is the reference Jacobian per Definition \ref{tjac} with eigenvalue decomposition $\Jb=\Ub\mtx{\Lambda}\Vb^T=\sum_{s=1}^{\K n} \la_s\ub_s\vb_s^T$ with $\mtx{\Lambda}=\text{diag}(\vct{\la})$. We also define $\vct{a}=\mtx{U}^T\rb_0=\mtx{U}^T\widetilde{\rb}_0\in\R^{nK}$ be the coefficients of the initial residual in the span of the column space of this reference Jacobian. 

We shall first characterize the evolution of the linearized parameter $\bbteta_\tau$ and residual $\rbb_{\tau}$ vectors from \eqref{lin res eq} in the following lemma.
\begin{lemma} \label{lem linear} The linearized residual vector $\rbb_{\tau}$ can be written in the form
\begin{align}
\rbb_\tau=\mtx{U}\left(\Iden-\eta\mtx{\Lambda}^2\right)^\tau\vct{a}=\sum_{s=1}^{nK} (1-\eta \la_s^2)^{\tau} a_s\ub_s.\label{res result}
\end{align}
Furthermore, assuming $\eta\leq 1/\la_1^2$ the linear updates $\bbteta_\tau$ obey
\begin{align}
\label{tetaresult}
\tn{\bbteta_\tau-\bbteta_0}^2\leq\sum_{s=1}^r  \frac{a_s^2}{\la_s^2}+\tau^2\eta^2\sum_{s=r+1}^{nK} \la_s^2 a_s^2.
\end{align}
\end{lemma}
\begin{proof}
Using the fact that $\Jb\Jb^T=\mtx{U}\mtx{\Lambda}^2\mtx{U}^T$ we have
\begin{align*}
(\Iden-\eta\Jb\Jb^T)^\tau=\mtx{U}\left(\Iden-\eta\mtx{\Lambda}^2\right)^\tau\mtx{U}^T
\end{align*}
Using the latter combined with \eqref{lin res eq} we thus have
\begin{align*}
\rbb_{\tau}=&(\Iden-\eta\Jb\Jb^T)^\tau\rb_0,\\
=&\mtx{U}\left(\Iden-\eta\mtx{\Lambda}^2\right)^\tau\mtx{U}^T\rb_0,\\
=&\mtx{U}\left(\Iden-\eta\mtx{\Lambda}^2\right)^\tau\vct{a},\\
=&\sum_{s=1}^{nK} (1-\eta \la_s^2)^{\tau} a_s\ub_s,
\end{align*}
completing the proof of \eqref{res result}.

We now turn our attention to proving \eqref{tetaresult} by tracking the representation of $\bbteta_\tau$ in terms of the right singular vectors of $\Jb$. To do this note that using \eqref{res result} we have
\begin{align*}
\Jb^T\rbb_t=\mtx{V}\mtx{\Lambda}\mtx{U}^T\rbb_t=\mtx{V}\mtx{\Lambda}\left(\Iden-\eta\mtx{\Lambda}^2\right)^t\vct{a}.
\end{align*}
Using the latter together with the gradient update on the linearized problem we have
\begin{align*}
\bbteta_\tau-\bbteta_0=&-\eta\left(\sum_{t=0}^{\tau-1}\nabla \mathcal{L}_{lin}(\bbteta_t)\right)=-\eta\left(\sum_{t=0}^{\tau-1}\Jb^T\rbb_t\right)=-\eta\mtx{V}\left(\sum_{t=0}^{\tau-1}\mtx{\Lambda}\left(\Iden-\eta\mtx{\Lambda}^2\right)^t\right)\vct{a}.
\end{align*}
Thus for any $s\in\{1,2,\ldots,nK\}$
\begin{align*}
\vct{v}_s^T\left(\bbteta_\tau-\bbteta_0\right)=-\eta\lambda_s\vct{a}_s\left(\sum_{t=0}^{\tau-1}\left(1-\eta\lambda_s^2\right)^t\right)=-\eta\lambda_s\vct{a}_s\frac{1-\left(1-\eta\lambda_s^2\right)^\tau}{\eta\lambda_s^2}=-\vct{a}_s\frac{1-\left(1-\eta\lambda_s^2\right)^\tau}{\lambda_s}.
\end{align*}
Noting that for $\eta\le 1/\lambda_1^2\le 1/\lambda_s^2$ we have $1-\eta\lambda_s^2\ge 0$, the latter identity implies that
\begin{align}
\label{vproj1}
\abs{\vct{v}_s^T\left(\bbteta_\tau-\bbteta_0\right)}\le \frac{\abs{\vct{a}_s}}{\lambda_s}.
\end{align}
Furthermore, using the fact that $1-\eta\lambda_s^2\le 1$ we have
\begin{align}
\label{vproj2}
\abs{\vct{v}_s^T\left(\bbteta_\tau-\bbteta_0\right)}=\eta\lambda_s\abs{\vct{a}_s}\left(\sum_{t=0}^{\tau-1}\left(1-\eta\lambda_s^2\right)^t\right)\le \eta\lambda_s\abs{\vct{a}_s}\tau
\end{align}
Combining \eqref{vproj1} for $1\le s\le r$ and \eqref{vproj2} for $s>r$ we have
\[
\twonorm{\bbteta_\tau-\bbteta_0}^2=\sum_{s=1}^{nK} \abs{\vct{v}_s^T\left(\bbteta_\tau-\bbteta_0\right)}^2\leq \sum_{s=1}^r  \frac{a_s^2}{\la_s^2}+\tau^2\eta^2\sum_{s=r+1}^{nK} \la_s^2 a_s^2,
\]
completing the proof of \eqref{tetaresult}.
\end{proof}
For future use we also state a simple corollary of the above Lemma below.
\begin{corollary}\label{cor simp} Consider the setting and assumptions of Lemma \ref{lem linear}. Then, after $\tau$ iterations we have
\begin{align}
\label{cor76conc1}
\twonorm{\rbb_\tau}\le \left(1-\eta\alpha^2\right)^\tau\tn{\Pi_{\Rc}(\rb_0)}+\tn{\Pi_{\Rcb}(\rb_0)}.
\end{align}
Furthermore, after $T=\frac{\Gamma}{\eta\alpha^2}$ iterations we have
\begin{align}
\label{cor76conc2}
\tn{\rbb_T}\le e^{-\Gamma}\tn{\Pi_{\Rc}(\rb_0)}+\tn{\Pi_{\Rcb}(\rb_0)}.
\end{align}
and
\begin{align*}
&\twonorm{\bbteta_T-\bbteta_0}^2\leq \sum_{s=1}^r  \frac{\vct{a}_s^2}{\la_s^2}+\Gamma^2\sum_{s=r+1}^{nK}\frac{\la_s^2 \vct{a}_s^2}{\alpha^4}=\frac{\upp^2}{\alpha^2}.
\end{align*}
with $\upp$ given by \eqref{pinv def2} per Definition \ref{bndearlystopval}.
\end{corollary}
\begin{proof} 
To prove the first bound on the residual (\eqref{cor76conc1})note that using \eqref{res result} we have
\begin{align*}
\mtx{U}_{\mathcal{I}}^T\rbb_\tau=\left(\mtx{I}-\eta\Lambda_{\calF}^2\right)^\tau\mtx{U}_{\mathcal{I}}^T\rbb_0\quad \text{and}\quad\mtx{U}_{\mathcal{N}}^T\rbb_\tau=\left(\mtx{I}-\eta\Lambda_{\calS}^2\right)^\tau\mtx{U}_{\calS}^T\rbb_0
\end{align*}
Thus, using the fact that for $s\leq r$ we have $\lambda_s\ge \alpha$ we have $(1-\eta\la_s^2)^{\tau}\leq(1-\eta\alpha^2)^{\tau}$ and for $s>r$ we have $(1-\eta\la_s^2)^{\tau}\le 1$, we can conclude that
\begin{align*}
\twonorm{\mtx{U}_{\mathcal{I}}^T\rbb_\tau}\le \left(1-\eta\alpha^2\right)^\tau\twonorm{\mtx{U}_{\mathcal{I}}^T\rb_0}\quad\text{and}\quad\twonorm{\mtx{U}_{\mathcal{N}}^T\rbb_\tau}\le \twonorm{\mtx{U}_{\mathcal{N}}^T\rb_0}.
\end{align*}
Combining these with the triangular inequality we have
\begin{align*}
\twonorm{\rbb_\tau}=\twonorm{\begin{bmatrix}\mtx{U}_{\mathcal{I}}^T\rbb_\tau\\\mtx{U}_{\mathcal{N}}^T\rbb_\tau\end{bmatrix}}\le\twonorm{\mtx{U}_{\mathcal{I}}^T\rbb_\tau}+\twonorm{\mtx{U}_{\mathcal{N}}^T\rbb_\tau}\le\left(1-\eta\alpha^2\right)^\tau\twonorm{\mtx{U}_{\mathcal{I}}^T\rb_0}+\twonorm{\mtx{U}_{\mathcal{N}}^T\rb_0},
\end{align*}
concluding the proof of \eqref{cor76conc1}. The second bound on the residual simply follows from the fact that $(1-\eta\alpha^2)^{T}\leq e^{-\Gamma}$. The bound on $\twonorm{\bbteta_T-\bbteta_0}^2$ is trivially obtained by using $T^2=\frac{\Gamma^2}{\eta^2\alpha^4}$ in \eqref{tetaresult}.
\end{proof}
The lemma above shows that with enough iterations, gradient descent on the linearized problem fits the residual over the information space and the residual is (in the worst case) unchanged over the nuisance subspace $\Rcb$. Our hypothesis is that, when the model is generalizable the residual mostly lies on the information space $\Rc$ which contains the directions aligned with the top singular vectors. Hence, the smaller term $\tn{\Pi_{\mathcal{N}}(\rb_0)}$ over the nuisance space will not affect generalization significantly. To make this intuition precise however we need to connect the residual of the original problem to that of the linearized problem. The following lemma sheds light on the evolution of the original problem \eqref{nonlin gd} by characterizing the evolution of the difference between the residuals of the original and linearized problems from one iteration to the next.
\begin{lemma}[Keeping track of perturbation - one step]\label{grwth} Assume Assumptions \ref{ass2} and \ref{ass3} hold and $\bteta_{\tau}$ and $\bteta_{\tau+1}$ are within an $R$ neighborhood of $\bteta_0$, that is,
\begin{align*}
\twonorm{\bteta_{\tau}-\bteta_0}\le R\quad\text{and}\quad\twonorm{\bteta_{\tau+1}-\bteta_0}\le R.
\end{align*}
Then with a learning rate obeying $\eta\leq 1/\beta^2$, the deviation in the residuals of the original and linearized problems $\eb_{\tau+1}=\rb_{\tau+1}-\rbb_{\tau+1}$ obey
\begin{align}
\tn{\eb_{\tau+1}}&\leq \eta(\eps_0^2+\eps\beta)\tn{\rbb_{\tau}}+(1+\eta\eps^2)\tn{\eb_{\tau}}.\label{etau repeat}
\end{align}
\end{lemma}
\begin{proof} For simplicity, denote $\B_1=\Jc(\bteta_{\tau+1},\bteta_\tau)$, $\B_2=\Jc(\bteta_\tau)$, $\A=\Jc(\bteta_0)$ where
\[
\Jc(\bb,\ab)=\int_{0}^1 \Jc(t \bb+(1-t)\ab)dt.
\]
We can write the predictions due to $\bteta_{\tau+1}$ as
\begin{align*}
f(\bteta_{\tau+1})&=f(\bteta_{\tau}-\eta\grad{\bteta_\tau})=f(\bteta_{\tau})+\eta \Jc(\bteta_{\tau+1},\bteta_\tau) \grad{\bteta_\tau}\\
&=f(\bteta_{\tau})+\eta \Jc(\bteta_{\tau+1},\bteta_\tau)\Jc^T(\bteta_\tau)(f(\bteta_{\tau})-\y).
\end{align*}
This implies that
\[
\rb_{\tau+1}=f(\bteta_{\tau+1})-\y=(\Iden-\eta\B_1\B_2^T)\rb_\tau.
\]
Similarly, for linearized problem we have $\rbb_{\tau+1}=(\Iden-\eta\Jb\Jb^T)\rbb_\tau$. Thus,%
\begin{align}
\label{tkey2}
\tn{\eb_{\tau+1}}&=\twonorm{(\Iden-\eta\B_1\B_2^T)\rb_\tau-(\Iden-\eta\Jb\Jb^T)\rbb_\tau}\nn\\
&=\twonorm{(\Iden-\eta\B_1\B_2^T)\eb_\tau-\eta(\B_1\B_2^T-\Jb\Jb^T)\rbb_\tau}\nn\\
&\leq \tn{(\Iden-\eta\B_1\B_2^T)\eb_\tau}+\eta\tn{(\B_1\B_2^T-\Jb\Jb^T)\rbb_\tau}\nn\\
&\leq \tn{(\Iden-\eta\B_1\B_2^T)\eb_\tau}+\eta\opnorm{(\B_1\B_2^T-\Jb\Jb^T)}\tn{\rbb_\tau}.
\end{align}
We proceed by bounding each of these two terms. For the first term, we apply Lemma \ref{asym pert} with $\A=\B_1$ and $\B=\B_2$ and use $\|\B_1-\B_2\|\leq \eps$ to conclude that
\begin{align}
\label{fstkey2}
\tn{(\Iden-\eta\B_1\B_2^T)\eb_\tau}\leq (1+\eta\eps^2)\tn{\eb_\tau}.
\end{align}
Next we turn our attention to bounding the second term. To this aim note that
\begin{align}
\label{sstkey2}
\|\B_1\B_2^T-\Jb\Jb^T\|&=\|\B_1\B_2^T-\A\A^T+\A\A^T-\Jb\Jb^T\|\nn\\
&\leq \|\B_1\B_2^T-\A\A^T\|+\|\A\A^T-\Jb\Jb^T\|\nn\\
&\leq \|(\B_1-\A)\B_2^T\|+\|\A(\B_2-\A)^T\|+\|\A\A^T-\Jb\Jb^T\|\nn\\
&\le \|\B_1-\A\|\|\B_2\|+\|\B_2-\A\|\|\A\|+\|\A\A^T-\Jb\Jb^T\|\nn\\
&\leq \beta\frac{\eps}{2}+\beta\frac{\eps}{2}+ \eps_0^2\nn\\
&=\eps_0^2+\eps\beta.
\end{align}
In the last inequality we use the fact that per Assumption \ref{ass3} we have $\|\B_1-\A\|\le \eps/2$ and $\|\B_2-\A\|\le \eps/2$ as well as the fact that per Definition \ref{tjac} $\|\A\A^T-\Jb\Jb^T\|\le \eps_0^2$. Plugging \eqref{fstkey2} and \eqref{sstkey2} in \eqref{tkey2} completes the proof.
\end{proof}
Next we prove a result about the growth of sequences obeying certain assumptions. As we will see later on in the proofs this lemma allows us to control the growth of the perturbation between the original and linearized residuals ($e_\tau=\twonorm{\eb_\tau}$).  
\begin{lemma}[Bounding residual perturbation growth for general nonlinearities] \label{lem growth} Consider positive scalars $\Gamma,\alpha,\eps,\eta>0$. Also assume $\eta\leq 1/\alpha^2$ and $\alpha\ge \sqrt{2\Gamma}\eps$ and set $T=\frac{\Gamma}{\eta\alpha^2}$. Assume the scalar sequences $e_\tau$ (with $e_0=0$) and $\widetilde{r}_\tau$ obey the following identities 
\begin{align}
{\widetilde{r}}_{\tau}\le& (1-\eta\alpha^2)^\tau \rho_++\rho_-,\nn\\
e_{\tau}\le& (1+\eta\eps^2)e_{\tau-1}+ \eta \Theta{\widetilde{r}}_{\tau-1},\label{lem78ass}
\end{align}
for all $0\le \tau\le T$ and non-negative values $\rho_-,\rho_+\ge 0$. Then, for all $0\leq \tau\leq T$, 
\begin{align}
e_{\tau}\leq \Theta \Lambda\quad\text{holds with}\quad\Lambda=\frac{2(\Gamma\rho_-+\rho_+)}{\alpha^2}.\label{e control}
\end{align}
\end{lemma}
\begin{proof} We shall prove the result inductively. Suppose \eqref{e control} holds for all $t\leq \tau-1$. Consequently, we have
\begin{align*}
e_{t+1}\le&(1+\eta\eps^2)e_t+ \eta\Theta {\widetilde{r}}_t\\
\le&e_t+\eta\eps^2e_t+\eta\Theta\left((1-\eta\alpha^2)^t \rho_++\rho_-\right)\\
\le&e_t+\eta \Theta\left(\eps^2 \Lambda+(1-\eta\alpha^2)^{t} \rho_++\rho_-\right).
\end{align*}
Thus
\begin{align}
\frac{e_{t+1}-e_{t}}{\Theta}\le \eta \left(\eps^2 \Lambda+(1-\eta\alpha^2)^{t} \rho_++\rho_-\right).\label{diff eq}
\end{align}
Summing up both sides of \eqref{diff eq} for $0\leq t\leq \tau-1$ we conclude that
\begin{align*}
\frac{e_{\tau}}{\Theta}&=\sum_{t=0}^{\tau-1} \frac{e_{t+1}-e_t}{\Theta}\\
&\leq \eta\tau\left(\eps^2\Lambda+\rho_-\right)+\eta\rho_+\sum_{t=0}^{\tau-1}(1-\eta\alpha^2)^{t} \\
&= \eta\tau\left(\eps^2\Lambda+\rho_-\right)+\eta\rho_+\frac{1-\left(1-\eta\alpha^2\right)^\tau}{\eta\alpha^2} \\
&\le \eta\left( \tau\eps^2 \Lambda+\frac{\rho_+}{\eta\alpha^2} +\tau\rho_-\right)\\
&= \eta\tau(\eps^2 \Lambda+\rho_-)+\frac{\rho_+}{\alpha^2}\\
&\leq \eta T(\eps^2 \Lambda+\rho_-)+\frac{\rho_+}{\alpha^2}\\
&=\frac{\Gamma\eps^2\Lambda +\Gamma\rho_-+\rho_+}{\alpha^2}\\
&=\frac{\Gamma\eps^2\Lambda}{\alpha^2}+\frac{\Lambda}{2}\\
&\leq \Lambda,
\end{align*}
where in the last inequality we used the fact that $\alpha^2\geq 2\Gamma\eps^2$. This completes the proof of the induction step and the proof of the lemma.
\end{proof}

\subsubsection{Completing the proof of Theorem \ref{many step thm}}
\label{comp5}
With the key lemmas in place in this section we wish to complete the proof of Theorem \ref{many step thm}.
We will use induction to prove the result. Suppose the statement is true for some $\tau-1\leq T-1$. In particular, we assume the identities \eqref{main res eq1} and \eqref{main res eq2} hold for all $0\leq t\leq\tau-1$. We aim to prove these identities continue to hold for iteration $\tau$. We will prove this result in multiple steps.


\noindent{\bf{Step I: Next iterate obeys $\twonorm{\bteta_\tau-\bteta_0}\le R$.}}\\ 
We first argue that $\bteta_{\tau}$ lies in the domain of interest as dictated by \eqref{assumpgen}, i.e.~$\tn{\bteta_\tau-\bteta_0}\leq R$. To do this note that per the induction assumption \eqref{main res eq2} holds for iteration $\tau-1$ and thus $\tn{\bteta_{\tau-1}-\bteta_0}\leq R/2$. As a result using the triangular inequality to show $\twonorm{\bteta_\tau-\bteta_0}\le R$ holds it suffices to show that $\tn{\bteta_{\tau}-\bteta_{\tau-1}}\leq R/2$ holds. To do this note that
\begin{align}
\tn{\bteta_\tau-\bteta_{\tau-1}}=&\eta\tn{ \grad{\bteta_{\tau-1}}}\nn\\
=&\eta\twonorm{\Jc^T(\bteta_{\tau-1})\rb_{\tau-1}}\nn\\
=&\eta\twonorm{\Jcb^T(\bteta_{\tau-1})\rb_{\tau-1}}\nn\\
\overset{(a)}{\le}& \eta \tn{\Jcb^T(\bteta_{\tau-1})\rbb_{\tau-1}}+\eta \tn{\Jcb^T(\bteta_{\tau-1})(\rb_{\tau-1}-\rbb_{\tau-1})}\nn\\
\overset{(b)}{\le}& \eta \twonorm{\Jb^T\rbb_{\tau-1}}+\eta \opnorm{\Jcb(\bteta_{\tau-1})-\Jb}\twonorm{\rbb_{\tau-1}}+\eta \opnorm{\Jcb(\bteta_{\tau-1})}\twonorm{\rb_{\tau-1}-\rbb_{\tau-1}}\nn\\
\overset{(c)}{\le}& \eta \twonorm{\Jb^T\rbb_{\tau-1}}+\frac{\eps_0+\eps}{\beta^2}\twonorm{\rbb_{\tau-1}}+\frac{1}{\beta}\tn{\rb_{\tau-1}-\rbb_{\tau-1}}\nn\\
\overset{(d)}{\le}& \eta \tn{\Jb^T\rbb_{\tau-1}}+\frac{2\delta\alpha}{5\beta^2}\tn{{\rb}_0}+\frac{1}{\beta}\twonorm{\rb_{\tau-1}-\rbb_{\tau-1}}\nn\\
\overset{(e)}{\le}& \eta \tn{\Jb^T\rbb_{\tau-1}}+\frac{2\delta\alpha}{5\beta^2}\tn{{\rb}_0}+\frac{3\delta\alpha}{5\beta^2}\tn{\rb_0}\nn\\
=& \eta \tn{\Jb^T\rbb_{\tau-1}}+\frac{\delta\alpha}{\beta^2}\tn{\rb_0}\nn\\
\overset{(f)}{\le}& \eta\beta^2\frac{\upp}{\alpha}+\frac{\delta\alpha}{\beta^2}\twonorm{{\rb}_0}\nn\\
\overset{(g)}{\le}& \frac{\upp}{\alpha}+\frac{\delta\alpha}{\beta^2}\twonorm{{\rb}_0}\nn\\
\overset{(h)}{\le}& \frac{\upp}{\alpha}+\frac{\delta\Gamma}{\alpha}\twonorm{{\rb}_0}\nn\\
=&\frac{R}{2}.\nn
\end{align}
Here, (a) and (b) follow from a simple application of the triangular inequality, (c) from the fact that $\opnorm{\Jc(\bteta_{\tau-1})-\mtx{J}}\le\opnorm{\Jc(\bteta_{\tau-1})-\Jc(\bteta_0)}+\opnorm{\Jc(\bteta_0)-\mtx{J}}\le \eps+\eps_0$, (d) from combining the bounds in \eqref{cndd}, (e) from the induction hypothesis that postulates \eqref{main res eq1} holds for iteration $\tau-1$, (f) from considering the SVD $\mtx{J}=\mtx{U}\mtx{\Lambda}\mtx{V}^T$ which implies that
\begin{align*}
\twonorm{\mtx{J}^T\rbb_{\tau-1}}^2=&\twonorm{\mtx{J}^T\left(\mtx{I}-\eta\mtx{J}\mtx{J}^T\right)^{\tau-1}\rb_0}^2=\twonorm{\mtx{V}\mtx{\Lambda}\left(\mtx{I}-\eta\mtx{\Lambda}^2\right)^{\tau-1}\mtx{U}^T\rb_0}^2\\
=&\twonorm{\mtx{\Lambda}\left(\mtx{I}-\eta\mtx{\Lambda}^2\right)^{\tau-1}\mtx{U}^T\rb_0}^2\\
=&\sum_{s=1}^{nK} \lambda_s^2 (1-\eta\lambda_s^2)^{2(\tau-1)} \left(\langle\vct{u}_s,\rb_0\rangle\right)^2\\
\le&\sum_{s=1}^{nK} \lambda_s^2 \left(\langle\vct{u}_s,\rb_0\rangle\right)^2\\
=&\sum_{s=1}^{r} \lambda_s^2 \left(\langle\vct{u}_s,\rb_0\rangle\right)^2+\sum_{s=r+1}^{nK} \lambda_s^2 \left(\langle\vct{u}_s,\rb_0\rangle\right)^2\\
\le&\beta^4\sum_{s=1}^{r} \frac{1}{\lambda_s^2} \left(\langle\vct{u}_s,\rb_0\rangle\right)^2+\sum_{s=r+1}^{nK} \lambda_s^2 \left(\langle\vct{u}_s,\rb_0\rangle\right)^2\\
\le&\beta^4\left(\sum_{s=1}^{r} \frac{1}{\lambda_s^2} \left(\langle\vct{u}_s,\rb_0\rangle\right)^2+\Gamma^2\sum_{s=r+1}^{nK} \frac{\lambda_s^2}{\alpha^4} \left(\langle\vct{u}_s,\rb_0\rangle\right)^2\right)\\
=&\beta^4\left(\frac{\upp}{\alpha}\right)^2
\end{align*}
(g) from the fact that $\eta\le \frac{1}{\beta^2}$, and (h) from the fact that $\alpha\le \beta$ and $\Gamma \ge 1$.

\noindent{\bf{Step II: Original and linearized residuals are close (first part of \eqref{main res eq1}).}}\\
In this step we wish to show that the first part of \eqref{main res eq1} holds for iteration $\tau$. Since we established in the previous step that $\tn{\bteta_{\tau}-\bteta_0}\leq R$ the assumption of Lemma \ref{grwth} holds for iterations $\tau-1$ and $\tau$. Hence, using Lemma \ref{grwth} equation \eqref{etau repeat} we conclude that
\begin{align*}
\tn{\eb_{\tau}}&\leq \eta(\eps_0^2+\eps\beta)\tn{\rbb_{\tau-1}}+(1+\eta\eps^2)\tn{\eb_{\tau-1}}.
\end{align*}
This combined with the induction assumption implies that
\begin{align}
\label{pf53t1}
\tn{\eb_{t}}&\leq \eta(\eps_0^2+\eps\beta)\tn{\rbb_{t-1}}+(1+\eta\eps^2)\tn{\eb_{t-1}},
\end{align}
holds for all $t\le \tau\le T$. Furthermore, using Lemma \ref{lem linear} equation \eqref{cor76conc1} for all $t\le \tau\le T$ we have
\begin{align}
\label{pf53t2}
\twonorm{\rbb_t}\le \left(1-\eta\alpha^2\right)^t\tn{\Pi_{\Rc}(\rb_0)}+\tn{\Pi_{\Rcb}(\rb_0)},
\end{align}
To proceed, we shall apply Lemma \ref{lem growth} with the following variable substitutions
\begin{align}
\Theta:=\eps_0^2+\eps\beta,\quad\rho_+=\tn{\Pi_{\Rc}(\rb_0)},\quad\rho_-=\tn{\Pi_{\Rcb}(\rb_0)},\quad e_\tau:=\tn{\eb_\tau},\quad {\widetilde{r}}_\tau:= \tn{\rbb_\tau}.
\end{align}
We note that Lemma \ref{lem growth} is applicable since (i) $\eta\le 1/\beta^2\le 1/\alpha^2$, (ii) based on \eqref{cndd} we have $\frac{\alpha}{\eps}\ge \frac{5\Gamma}{\delta}\frac{\beta^2}{\alpha^2}\ge \sqrt{2\Gamma}$, (iii) $\tau$ obeys $\tau\leq T= \frac{\Gamma}{\eta\bn^2}$, and (iv) \eqref{lem78ass} holds based on \eqref{pf53t1} and \eqref{pf53t2}. Thus using Lemma \ref{lem growth} we can conclude that
\begin{align}
\tn{\eb_{\tau}}\le& 2(\eps_0^2+\eps\beta)\frac{\left(\tn{\Pi_{\Rc}(\rb_0)}+\Gamma\tn{\Pi_{\Rcb}(\rb_0)}\right)}{\bn^2}\nn\\
\le&\frac{2\Gamma(\eps_0^2+\eps\beta)\tn{\rb_0}}{\alpha^2}\label{etau termm}\\
\le&( \frac{2}{25}+\frac{2}{5})\frac{\delta\alpha}{\beta} \tn{\rb_0}\leq\frac{3}{5} \frac{\delta\alpha}{\beta} \tn{\rb_0},\label{etau term}
\end{align}
where in the last inequality we used \eqref{cndd}. This completes the first part of \eqref{main res eq1} via induction.

\noindent{\bf{Step III: Original and linearized parameters are close (second part of \eqref{main res eq1}).}}\\
 In this step we wish to show that the second part of \eqref{main res eq1} holds for iteration $\tau$. To do this we begin by noting that by the fact that $\Jb$ is a reference Jacobian we have $\|\Jcb(\bteta_0)-\Jb\|\leq \eps_0$ where $\Jcb$ augments $\Jc(\bteta_0)$ by padding zero columns to match size of $\Jb$. Also by Assumption \ref{ass3} we have $\|\Jc(\bteta)-\Jc(\bteta_0)\|\leq \frac{\eps}{2}$. Combining the latter two via the  triangular inequality we conclude that
\begin{align}
\label{53temp3}
\|\Jcb(\bteta_\tau)-\Jb\|\leq \eps_0+\eps.
\end{align}
Let $\brteta$ and $\nabla \bar{\Lc}(\bteta)$ be vectors augmented by zero padding $\bteta,\grad{\bteta}$ so that they have dimension $\max(\K n, p)$. Now, we track the difference between $\brteta$ and linearized $\tilde{\bteta}$ as follows
\begin{align}
\frac{\tn{\brteta_{\tau}-\widetilde{\bteta}_{\tau}}}{\eta}&=\twonorm{\sum_{t=0}^{\tau-1}\nabla\bar{\Lc}(\bteta_t)-\nabla\mathcal{L}_{lin}(\bbteta_t)}\nn\\
&=\twonorm{\sum_{t=0}^{\tau-1}\Jcb(\bteta_t)^T\rb_t- \Jb^T\rbb_t}\nn\\
&\leq \sum_{t=0}^{\tau-1} \tn{\Jcb(\bteta_t)^T\rb_t- \Jb^T\rbb_t}\nn\\
&\leq \sum_{t=0}^{\tau-1} \tn{(\Jcb(\bteta_t)-\Jb)^T\rbb_t}+\tn{\Jcb(\bteta_t)^T(\rb_t-\rbb_t)}\nn\\
&= \sum_{t=0}^{\tau-1} \tn{(\Jcb(\bteta_t)-\Jb)^T\rbb_t}+\tn{\Jcb(\bteta_t)^T\eb_t}\nn\\
&\leq  \sum_{t=0}^{\tau-1}(\eps+\eps_0)\tn{\rbb_t}+\beta\tn{\eb_t}.\label{combined eq}
\end{align}
In the last inequality we used the fact that $\|\Jcb(\bteta_t)-\Jb\|\leq \eps+\eps_0$ and $\|\Jb\|\leq \beta$. We proceed by bounding each of the two terms in \eqref{combined eq} above. For the first term we use the fact that $\tn{\rbb_\tau}\leq \tn{\rb_0}$ to conclude
\begin{align}
\sum_{t=0}^{\tau-1} \tn{\rbb_t}\leq\tau \tn{\rb_0}\leq T \tn{\rb_0}= \frac{\Gamma\tn{\rb_0}}{\eta\alpha^2}.\label{rrr bound}
\end{align}
To bound the second term in \eqref{combined eq} we use \eqref{etau termm} together with $\tau\le T\le \frac{\Gamma}{\eta\alpha^2}$ to conclude that
\begin{align}
\sum_{t=0}^{\tau-1} \tn{\eb_t}\leq \tau  \frac{2(\eps\beta+\eps_0^2)}{\alpha^2} \Gamma\tn{\rb_0}\leq   \frac{2\Gamma^2(\eps\beta+\eps_0^2)}{\eta \alpha^4} \tn{\rb_0}.\label{eee bound}
\end{align}
Combining \eqref{rrr bound} and \eqref{eee bound} in \eqref{combined eq}, we conclude that
\begin{align*}
\tn{\brteta_{\tau}-\widetilde{\bteta}_{\tau}}\le& \left(  \frac{2\Gamma(\eps\beta^2+\eps_0^2\beta)}{\alpha^3}+\frac{\eps+\eps_0}{\alpha}\right) \frac{\Gamma}{\alpha}\tn{\rb_0}\\
=&\left(\eps\frac{2\Gamma\beta^2}{\alpha^3}+\eps_0^2\frac{2\Gamma\beta}{\alpha^3}+\frac{\eps+\eps_0}{\alpha}\right)\frac{\Gamma}{\alpha}\tn{\rb_0}\\
\overset{(a)}{\le}&\left(\frac{2}{5}\delta+\eps_0^2\frac{2\Gamma\beta}{\alpha^3}+\frac{\eps+\eps_0}{\alpha}\right)\frac{\Gamma}{\alpha}\tn{\rb_0}\\
\overset{(b)}{\le}&\left(\frac{2}{5}\delta+\frac{2}{25}\delta+\frac{\eps+\eps_0}{\alpha}\right)\frac{\Gamma}{\alpha}\tn{\rb_0}\\
\overset{(c)}{\le}&\left(\frac{2}{5}\delta+\frac{2}{25}\delta+\frac{1}{5}\delta+\frac{\eps_0}{\alpha}\right)\frac{\Gamma}{\alpha}\tn{\rb_0}\\
\overset{(d)}{\le}&\left(\frac{2}{5}\delta+\frac{2}{25}\delta+\frac{1}{5}\delta+\frac{1}{5}\delta\right)\frac{\Gamma}{\alpha}\tn{\rb_0}\\
=&\frac{22}{25}\frac{\delta}{\alpha}\Gamma\tn{\rb_0}.
\end{align*}
Here, (a) follows from $\eps\leq \frac{\delta\alpha^3}{5\Gamma\beta^2}$ per Assumption \eqref{cndd}, (b) from $\eps_0\le \frac{1}{5}\sqrt{\frac{\delta\alpha^3}{\Gamma\beta}}$ per Assumption \eqref{cndd}, (c) from $\eps\leq \frac{\delta\alpha^3}{5\Gamma\beta^2}\le \frac{\delta\alpha}{5\Gamma}\le \frac{\delta\alpha}{5}$ per Assumption \eqref{cndd}, and (d) from $\eps_0\le \frac{\delta\alpha}{5}$ per Assumption \eqref{cndd}. Thus,
\begin{align*}
\tn{\brteta_{\tau}-\widetilde{\bteta}_{\tau}}\le\frac{\delta}{\alpha}\Gamma\tn{\rb_0}.
\end{align*}
Combining the latter with the fact that $\tn{\bbteta_\tau-\brteta_0}\leq \frac{\upp}{\alpha}$ (which follows from Lemma \ref{lem linear} equation \eqref{tetaresult}) we conclude that
\begin{align*}
\twonorm{\bteta_\tau-\bteta_0}=\twonorm{\brteta_\tau-\brteta_0}\le \twonorm{\bbteta_\tau-\brteta_0}+\twonorm{\brteta_\tau-\bbteta_\tau}\le \frac{\upp}{\alpha}+\frac{\delta}{\alpha}\Gamma\tn{\rb_0}\le \twonorm{\mtx{J}_{\calF}^{\dagger}\vct{r}_0}+\frac{\Gamma}{\alpha}\twonorm{\Pi_{\calS}\left(\vct{r}_0\right)}+\frac{\delta}{\alpha}\Gamma\tn{\rb_0}
\end{align*}
The completes the proof of the bound \eqref{main res eq2}. 
  
\noindent\textbf{Step V: Bound on residual with early stopping.}\\
In this step we wish to prove \eqref{eqqq2}. To this aim note that
\begin{align*}
\twonorm{\rb_T}\overset{(a)}{\le}& \twonorm{\rbb_T}+\twonorm{\rbb_T-\rb_T}\\
\overset{(b)}{\le}&\twonorm{\rbb_T}+\frac{\delta\alpha}{\beta}\tn{\rb_0}\\
\overset{(c)}{\le}&e^{-\Gamma}\tn{\Pi_{\Rc}(\rb_0)}+\tn{\Pi_{\Rcb}(\rb_0)}+\frac{\delta\alpha}{\beta}\tn{\rb_0}
\end{align*}
where (a) follows from the triangular inequality, (b) from the conclusion of Step II (first part of \eqref{main res eq1}), and (c) from Corollary \ref{cor simp} equation \eqref{cor76conc2}. This completes the proof of \eqref{eqqq2}.

\subsection{Key lemmas and identities for neural networks}
In this section we prove some key lemmas and identities regarding the Jacobian of one-hidden layer networks as well as the size of the initial residual that when combined with Theorem \ref{many step thm} allows us to prove theorems involving neural networks. We begin with some preliminary identities and calculations in Section \ref{prelimiden}. Next, in Section \ref{fundet} we prove a few key properties of the Jacobian mapping of a one-hidden layer neural network. Section \ref{Jacrandinit} focuses on a few further properties of the Jacobian at a random initialization. Finally, in Section \ref{initmisfit} we provide bounds on the initial misfit.

For two matrices
\begin{align*}
\mtx{A}=\begin{bmatrix}\mtx{A}_1\\\mtx{A}_2\\\vdots\\\mtx{A}_p\end{bmatrix}\in\R^{p\times m}\quad\text{and}\quad \mtx{B}=\begin{bmatrix}\mtx{B}_1\\\mtx{B}_2\\\vdots\\\mtx{B}_p\end{bmatrix}\in\R^{p\times n},
\end{align*}
we define their Khatri-Rao product as $\mtx{A} * \mtx{B}  = [\mtx{A}_1\otimes \mtx{B}_1,\dotsc, \mtx{A}_p\otimes \mtx{B}_p]\in\R^{p\times mn}$, where $\otimes$ denotes the Kronecker product.
\subsubsection{Preliminary identities and calculations}
\label{prelimiden}
We begin by discussing a few notations. Throughout we use $\w_\ell$ and $\vb_\ell$ to denote the $\ell$th row of input and output weight matrices $\W$ and $\Vb$. Given a matrix $\M$ we use $\trow{\M}$ to denote the largest Euclidean norm of the rows of $\M$. We begin by noting that for a one-hidden layer neural network of the form $\vct{x}\mapsto \Vb\phi\left(\W\x\right)$, the Jacobian matrix with respect to vect$(\mtx{W})\in\R^{kd}$ takes the form
\begin{align}
\mathcal{J}(\mtx{W})=
\begin{bmatrix}
\mathcal{J}_1(\W) \\ \vdots \\ \mathcal{J}_\K(\W)
\end{bmatrix}\in\R^{\K n\times kd}\label{jacob multi}
\end{align}
where $\Jc_\ell(\W)$ is the Jacobian matrix associated with the $\ell$th class. In particular, $\Jc_\ell(\W)$ is given by
\begin{align*}
\mathcal{J}_\ell(\mtx{W})=
\begin{bmatrix}
\mathcal{J}_\ell(\vct{w}_1) & \ldots & \mathcal{J}_\ell(\vct{w}_k)
\end{bmatrix}\in\R^{n\times kd}\quad\text{with}\quad\mathcal{J}_\ell(\vct{w}_s):=\mtx{V}_{\ell,s}\text{diag}\left(\phi'(\mtx{X}\vct{w}_s)\right)\mtx{X}.
\end{align*}
Alternatively using Khatri-Rao products this can be rewritten in the more compact form
\begin{align}
\label{KR}
\mathcal{J}_\ell(\mtx{W})=\left(\phi'\left(\mtx{X}\mtx{W}^T\right)\text{diag}(\vct{v}_\ell)\right)*\mtx{X}.
\end{align}
An alternative characterization of the Jacobian is via its matrix representation. Given a vector $\ub\in\R^{\K n}$ let us partition it into $\K$ size $n$ subvectors so that $\ub=[\ub_1^T~\dots~\ub_\K^T]^T$. We have
\begin{align}
\text{mat}\left(\mathcal{J}^T(\mtx{W})\vct{u}\right)= \sum_{\ell=1}^\K\text{diag}(\vct{v}_\ell)\phi'\left(\mtx{W}\mtx{X}^T\right)\text{diag}(\vct{u}_\ell)\mtx{X}.\label{mat form}
\end{align}

\subsubsection{Fundamental properties of the Jacobian of the neural network}
\label{fundet}
In this section we prove a few key properties of the Jacobian mapping of a one-hidden layer neural network.

\begin{lemma}[Properties of Single Output Neural Net Jacobian]\label{props thm} Let $\K=1$ so that $\Vb^T=\vb\in\R^n$. Suppose $\phi$ is an activation obeying $|\phi'(z)|\leq B$ for all $z$. Then, for any $\W\in\R^{k\times d}$ and any unit length vector $\ub$, we have
\begin{align*}
\|\Jc(\W)\|\le& B\sqrt{k}\infnorm{\vct{v}}\opnorm{\X}
\end{align*}
and
\begin{align}
\label{rowwisebnd}
\trow{\text{mat}\left(\mathcal{J}^T(\mtx{W})\vct{u}\right)}\le& B\tin{\vb}\|\X\|
\end{align}
Furthermore, suppose $\phi$ is twice differentiable and $|\phi''(z)|\leq B$ for all $z$. Also assume all data points have unit Euclidean norm ($\twonorm{\vct{x}_i}=1$). Then the Jacobian mapping is Lipschitz with respect to spectral norm i.e.~{for all} $\widetilde{\W},\W\in\R^{k\times d}$ we have
\begin{align*}
\opnorm{\mathcal{J}(\widetilde{\mtx{W}})-\mathcal{J}(\mtx{W})}\le B\infnorm{\vct{v}}\opnorm{\mtx{X}}\fronorm{\widetilde{\mtx{W}}-\mtx{W}}.
\end{align*}
\end{lemma}
\begin{proof} The result on spectral norm and Lipschitzness of $\Jc(\W)$ have been proven in \cite{onehidden}. To show the row-wise bound \eqref{rowwisebnd}, we use \eqref{mat form} to conclude that
\begin{align*}
\trow{\text{mat}\left(\mathcal{J}^T(\mtx{W})\vct{u}\right)}&= \trow{\text{diag}(\vct{v})\phi'\left(\mtx{W}\mtx{X}^T\right)\text{diag}(\vct{u})\mtx{X}}\\
&\le \tin{\vb}\max_{1\leq \ell\leq k}\tn{\phi'\left(\w_\ell^T\mtx{X}^T\right)\text{diag}(\vct{u})\mtx{X}}\\
&\leq \tin{\vb}\|\X\|\max_{1\leq \ell\leq k}\tn{\phi'\left(\w_\ell^T\mtx{X}^T\right)\text{diag}(\vct{u})}\\
&\leq B\tin{\vb}\|\X\|\tn{\ub}\\
&=B\tin{\vb}\|\X\|.
\end{align*}
\end{proof}
Next we extend the lemma above to the multi-class setting.
\begin{lemma}[Properties of Multiclass Neural Net Jacobian]\label{props lem} Suppose $\phi$ is an activation obeying $|\phi'(z)|\le B$ for all $z$. Then, for any $\W\in\R^{k\times d}$ and any unit length vector $\ub$, we have
\[
\|\Jc(\W)\|\leq B\sqrt{\K k}\infnorm{\Vb}\opnorm{\X}
\]
and
\begin{align}
\label{rowbndm}
\trow{\text{mat}\left(\mathcal{J}^T(\mtx{W})\vct{u}\right)}\leq B\sqrt{\K}\tin{\Vb}\|\X\|.
\end{align}
Furthermore, suppose $\phi$ is twice differentiable and $|\phi''(z)|\leq B$ for all $z$. Also assume all data points have unit Euclidean norm ($\twonorm{\vct{x}_i}=1$). Then the Jacobian mapping is Lipschitz with respect to spectral norm i.e.~{for all} $\widetilde{\W},\W\in\R^{k\times d}$ we have
\begin{align*}
\opnorm{\mathcal{J}(\widetilde{\mtx{W}})-\mathcal{J}(\mtx{W})}\le B\sqrt{\K}\infnorm{\Vb}\opnorm{\mtx{X}}\fronorm{\widetilde{\mtx{W}}-\mtx{W}}.
\end{align*}
\end{lemma}
\begin{proof} The proof will follow from Lemma \ref{props thm}. First, given $\A=[\A_1^T~\dots~\A_\K^T]^T$ and $\B=[\B_1^T~\dots~\B_\K^T]^T$, observe that
\[
\|\A\|\leq \sqrt{\K} \sup_{1\leq \ell\leq \K}\|\A_\ell\|\quad\text{and}\quad\|\A-\B\|\leq \sqrt{\K} \sup_{1\leq \ell\leq \K}\|\A_\ell-\B_\ell\|.
\]
These two identities applied to the components $\Jc_\ell(\W)$ and $\Jc_\ell(\widetilde{\W})-\Jc_\ell(\W)$ completes the proof of the bound on the spectral norm and the perturbation. To prove the bound in \eqref{rowbndm} we use the identity \eqref{mat form} to conclude that
\begin{align*}
\trow{\text{mat}\left(\mathcal{J}^T(\mtx{W})\vct{u}\right)}&=\trow{\sum_{\ell=1}^\K\text{diag}(\vct{v}_\ell)\phi'\left(\mtx{W}\mtx{X}^T\right)\text{diag}(\vct{u}_\ell)\mtx{X}}\\
&\le\sum_{\ell=1}^\K \trow{\text{diag}(\vct{v}_\ell)\phi'\left(\mtx{W}\mtx{X}^T\right)\text{diag}(\vct{u}_\ell)\mtx{X}}\\
&\le\sum_{\ell=1}^\K B\tin{\Vb}\|\X\|\tn{\ub_\ell}\\
&= B\tin{\Vb}\|\X\|\left(\sum_{\ell=1}^\K\tn{\ub_\ell}\right)\\
&\le B\tin{\Vb}\|\X\|\sqrt{K}\left(\sum_{\ell=1}^\K\tn{\ub_\ell}^2\right)^{1/2}\\
&= B\tin{\Vb}\|\X\|\sqrt{\K},
\end{align*}
where the penultimate inequality follows from Cauchy Schwarz, completing the proof.
\end{proof}
\subsubsection{Properties of the Jacobian at random initialization}
\label{Jacrandinit}
In this section we prove a few lemmas characterizing the properties of the Jacobian at the random initialization.
\begin{lemma}[Multiclass covariance] \label{multi cov}Given input and output layer weights $\Vb$ and $\W$, consider the Jacobian described in \eqref{jacob multi}. Given an $\K n\times \K n$ matrix $\M$, for $1\leq \ell,\widetilde{\ell}\leq \K$, let $\M[\ell,\widetilde{\ell}]$ denote the $(\ell,\widetilde{\ell})$th submatrix. For $\Cb(\W)=\Jc(\W)\Jc(\W)^T$ we have
\[
\Cb(\W)[\ell,\widetilde{\ell}]=\sum_{s=1}^k(\X\X^T)\odot (\mtx{V}_{\ell,s}\mtx{V}_{\widetilde{\ell},s}{\phi'(\X\w_{s})}{\phi'(\X\w_{s})^T}).
\]
Suppose $\W\distas\Nn(0,1)$ and $\Vb$ has i.i.d.~zero-mean entries with $\vrn^2$ variance. Then $\E[\Cb(\W)]$ is a block diagonal matrix given by the Kronecker product
\[
\E[\Cb(\W)]=k\vrn^2 \bSi(\X).
\]
where $\bSi(\X)$ is equal to $\Iden_{\K}\otimes[(\X\X^T)\odot \E[{\phi'(\X\w_s)}{\phi'(\X\w_s)^T}]]$.
\end{lemma}
\begin{proof} The $(\ell,\widetilde{\ell} )$th submatrix of $\Cb(\W)$ is given by
\begin{align}
\label{Celll}
\Cb(\W)[\ell,\widetilde{\ell}]&=((\text{diag}(\vb_\ell)\phi'(\W\X^T))*\X^T)((\text{diag}(\vb_{\widetilde{\ell}})\phi'(\W\X^T))*\X^T)^T\nn\\
&=\sum_{s=1}^k \Jc_\ell(\w_s)\Jc_{\widetilde{\ell}}(\w_s)^T\nn\\
&=\sum_{s=1}^k\mtx{V}_{\ell,s}\mtx{V}_{\widetilde{\ell},s}(\diag{\phi'(\X\w_s)}\X)(\diag{\phi'(\X\w_s)}\X)^T\nn\\
&=\sum_{s=1}^k\mtx{V}_{\ell,s}\mtx{V}_{\widetilde{\ell},s}(\X\X^T)\odot ({\phi'(\X\w_s)}{\phi'(\X\w_s)^T})\nn\\
&=\sum_{s=1}^k\left(\X\X^T\right)\odot \left(\mtx{V}_{\ell,s}\mtx{V}_{\widetilde{\ell},s}{\phi'(\X\w_s)}{\phi'(\X\w_s)^T}\right).
\end{align}
Setting $\W\distas\Nn(0,1)$ and $\mtx{V}$ with i.i.d. zero-mean and $\vrn^2$-variance entries, we conclude that
\begin{align*}
\E[\Cb(\W)[\ell,\widetilde{\ell}]]&=\sum_{s=1}^k(\X\X^T)\odot (\E[\mtx{V}_{\ell,s}\mtx{V}_{\widetilde{\ell},s}]\E[{\phi'(\X\w_s)}{\phi'(\X\w_s)^T}])\\
&=\sum_{s=1}^k\vrn^2\delta(\ell-\widetilde{\ell}) [(\X\X^T)\odot \E[{\phi'(\X\w_s)}{\phi'(\X\w_s)^T}]] \\
&=k\delta(\ell-\widetilde{\ell})\vrn^2 \tilde{\bSi}(\X),
\end{align*}
where $\delta(x)$ is the discrete $\delta$ function which is 0 for $x\neq 0$ and $1$ for $x=0$ and $\tilde{\bSi}(\X)$ is single output kernel matrix which concludes the proof.
\end{proof}
Next we state a useful lemma from \cite{Schur1911} which allows us to bound the eigenvalues of the Hadamard product of the two PSD matrices.
\begin{lemma}[\cite{Schur1911}]\label{minHad} Let $\mtx{A},\mtx{B}\in\R^{n\times n}$ be two Positive Semi-Definite (PSD) matrices. Then,
\begin{align*}
\lambda_{\min}\left(\mtx{A}\odot\mtx{B}\right)\ge& \left(\min_{i} \mtx{B}_{ii}\right)\lambda_{\min}\left(\mtx{A}\right),\\
\lambda_{\max}\left(\mtx{A}\odot\mtx{B}\right)\le& \left(\max_{i} \mtx{B}_{ii}\right)\lambda_{\max}\left(\mtx{A}\right).
\end{align*}
\end{lemma}
Next we state a lemma regarding concentration of the Jacobian matrix at initialization.
\begin{lemma}[Concentration of the Jacobian at initialization]\label{minspectJ} Consider a one-hidden layer neural network model of the form $\vct{x}\mapsto \Vb\phi\left(\W\x\right)$ where the activation $\phi$ obeys $|\phi(0)|\le B$ and $\abs{\phi'(z)}\le B$ for all $z$. Also assume we have $n\geq K$ data points $\vct{x}_1, \vct{x}_2,\ldots,\vct{x}_n\in\R^d$ with unit euclidean norm ($\twonorm{\vct{x}_i}=1$) aggregated as the rows of a matrix $\X\in\R^{n\times d}$. Furthermore, suppose $\Vb$ has i.i.d.~$\vrn$-scaled Rademacher entries (i.e.~$\pm \vrn$ equally-likely). Then, 
the Jacobian matrix at a random point $\mtx{W}_0\in\R^{k\times d}$ with i.i.d.~$\mathcal{N}(0,1)$ entries obeys 
\begin{align*}
&\opnorm{\Jc(\W_0)\Jc(\W_0)^T-\E[\Jc(\W_0)\Jc(\W_0)^T]}\leq 30\K\sqrt{k}\vrn^2B^2\|\X\|^2\log(n).
\end{align*}
with probability at least $1-1/n^{100}$. In particular, as long as
\[
k\geq \frac{1000\K^2B^4\|\X\|^4\log(n)}{\delta^2},
\]
with the same probability, we have that
\[
\opnorm{\frac{1}{k\vrn^2}\Jc(\W_0)\Jc(\W_0)^T- \bSi(\X)}\leq \delta.
\]
\end{lemma}
\begin{proof} 
Define $\Cb=\Jc(\W_0)\Jc(\W_0)^T$. We begin by showing that the diagonal blocks of $\Cb$ are concentrated. To do this first
for $1\le s\le k$ define the random matrices
\[
\A_s=\left(\phi'\left(\X\w_s\right)\phi'\left(\X\w_s\right)^T\right)\odot\left(\X\X^T\right).
\]
Now consider $n\times n$ diagonal blocks of $\Cb$ (denoted by $\Cb[\ell,\ell]$) and note that we have
\begin{align*}
 \Cb[\ell,\ell]=&\left(\phi'\left(\X\W^T\right)\text{diag}(\vct{v}_\ell)\text{diag}(\vct{v}_\ell)\phi'\left(\W\X^T\right)\right)\odot\left(\X\X^T\right)\\
 =&\sum_{s=1}^k \mtx{V}_{\ell,s}^2\A_s\\
 =&\vrn^2\sum_{s=1}^k \A_s.
\end{align*}
Furthermore, using Lemma \ref{minHad}
\begin{align*}
\opnorm{\A_s}\le \left(\underset{i}{\max}\text{ }\left(\phi'(\vct{x}_i^T\w_s)\right)^2\right)\opnorm{\X}^2\le B^2\opnorm{\X}^2.
\end{align*}
Also, using Jensen's inequality 
\begin{align*}
\opnorm{\E[\A_s]}\le \E\opnorm{\A_s}\le B^2\opnorm{\X}^2.
\end{align*}
Combining the latter two identities via the triangular inequality we conclude that
\begin{align}
\label{randjtemp}
\opnorm{(\A_s-\E[\A_s])^2}=\opnorm{\A_s-\E[\A_s]}^2\le \left(\opnorm{\A_s}+\opnorm{\E[\A_s]}\right)^2\leq \left(2B^2\|\X\|^2\right)^2.
\end{align}
To proceed, we will bound the weighted sum
\[
\Sb=\sum_{s=1}^k\vrn^2(\A_s-\E[\A_s])
\]
in spectral norm. To this aim we utilize the Matrix Hoeffding inequality which states that
\[
\Pro(\|\Sb\|\geq t)\leq 2ne^{-\frac{t^2}{2\Delta^2}},
\]
where $\Delta^2$ is an upper bound on $\opnorm{\sum_{s=1}^k\vrn^4(\A_s-\E[\A_s])^2}$. Using \eqref{randjtemp} we can pick $\Delta^2=\sum_{s=1}^k (2\vrn^2B^2\|\X\|^2)^2=4k\vrn^4B^4\|\X\|^4$. Setting $t=30\sqrt{k}\vrn^2B^2\|\X\|^2\sqrt{\log(n)}$, we conclude that
\[
\Pro\Big\{\opnorm{\Cb[\ell,\ell]-\E[\Cb[\ell,\ell]]}\geq t\Big\}=\Pro(\|\Sb\|\geq t)\leq n^{-102}
\]
concluding the proof of concentration of the diagonal blocks of $\Cb$. 

For the off-diagonal blocks $\Cb[\ell,\widetilde{\ell}]$ using \eqref{Celll} from the proof of Lemma \ref{multi cov} we have that
\[
\Cb[\ell,\widetilde{\ell}]=\sum_{s=1}^k \mtx{V}_{\ell,s}\mtx{V}_{\widetilde{\ell},s}\A_{s}.
\]
Note that by construction $\{\mtx{V}_{\ell,s}\mtx{V}_{\widetilde{\ell},s}\}_{s=1}^k$ are i.i.d.~$\pm \vrn^2$ Rademacher variables and thus $\Cb[\ell,\widetilde{\ell}]$ is sum of zero-mean i.i.d.~matrices and we are again in the position to apply Hoeffding's inequality. To this aim note that 
\begin{align*}
\opnorm{\sum_{s=1}^k\mtx{V}_{\ell,s}^2\mtx{V}_{\widetilde{\ell},s}^2\A_{s}^2}=\vrn^4\opnorm{\sum_{s=1}^k\A_{s}^2}\le \vrn^4\sum_{s=1}^k\opnorm{\A_s}^2\le \vrn^4kB^4\opnorm{\X}^4,
\end{align*}
so that we can take $\Delta^2=\vrn^4kB^4\opnorm{\X}^4$ and again conclude that for $t=30\sqrt{k}\vrn^2B^2\|\X\|^2\log(n)$ we have
\[
\Pro\Big\{\opnorm{\Cb[\ell,\widetilde{\ell}]}\geq t\Big\}\leq n^{-102}
\]
Using the fact that $\E[\Cb[\ell,\widetilde{\ell}]]=0$ and $K\le n$, combined with a union bound over all sub-matrices $1\leq \ell,\widetilde{\ell}\leq K$ we conclude that
\[
\Pro\Big\{\opnorm{\Cb[\ell,\widetilde{\ell}]-\E\big[\Cb[\ell,\widetilde{\ell}]\big]}\geq t\Big\}\leq K^2n^{-102}\leq n^{-100}.
\]
All that remains is to combine the concentration results for the sub-matrices to arrive at the complete bound. In mathematical terms we need to bound ${\bf{D}}:=\|\Cb-\E[\Cb]\|$.To this aim define $\Db[\ell,:]$ to denote the $\ell$th block row of $\Db$. Standard bounds on spectral norm in terms of sub-matrices allow us to conclude that 
\begin{align*}
\|\Db[\ell,:]\|\leq& \sqrt{\K}\sup_{1\leq \widetilde{\ell}\leq \K} \opnorm{\Db[\ell,\widetilde{\ell}]}\leq \sqrt{\K}t\quad\Rightarrow\\
\|\Db\|\leq& \sqrt{\K}\sup_{1\leq \ell\leq \K} \|\Db[\ell,:]\|\leq \sqrt{\K}\sqrt{\K}t=\K t=30\K\sqrt{k}\vrn^2B^2\|\X\|^2\log(n),
\end{align*}
concluding the proof. The result in terms of $\delta$ is obtained by using the population covariance Lemma \ref{multi cov}.
\end{proof}

\subsubsection{Upper bound on initial residual}
\label{initmisfit}
In this section we prove a lemma concerning the size of the initial misfit. The proof of this lemma (stated below) follows from a similar argument in the proof of \cite[Lemma 6.12]{onehidden}.
\begin{lemma}[Upper bound on initial residual]\label{upresz} Consider a one-hidden layer neural network model of the form $\vct{x}\mapsto \Vb\phi\left(\W\x\right)$ where the activation $\phi$ has bounded derivatives obeying $|\phi(0)|,\abs{\phi'(z)}\le B$. Also assume we have $n$ data points $\vct{x}_1, \vct{x}_2,\ldots,\vct{x}_n\in\R^d$ with unit euclidean norm ($\twonorm{\vct{x}_i}=1$) aggregated as rows of a matrix $\X\in\R^{n\times d}$ and the corresponding labels given by $\vct{y}\in\R^{\K n}$. Furthermore, assume the entries of $\Vb$ are i.i.d.~Rademacher variables scaled by $\frac{\vrn\twonorm{\y}}{50B\sqrt{\K \log(2\K)kn}}$ and the entries of $\mtx{W}\in\R^{k\times d}$ are i.i.d.~$\mathcal{N}(0,1)$. Then,
\begin{align*}
\fronorm{\Vb\phi\left(\W\X^T\right)}\le\vrn\twonorm{\y},
\end{align*}
holds with probability at least $1-(2\K)^{-100}$.
\end{lemma}
\begin{proof}
\label{upreszpf}
We begin the proof by noting that
\[
\fronorm{\Vb\phi\left(\W\X^T\right)}^2=\sum_{\ell=1}^\K \twonorm{\vb_\ell^T\phi\left(\W\X^T\right)}^2.
\]
We will show that for any row $\vb$ of $\Vb$, with probability at least $1-(2\K)^{-101}$,
\begin{align}
\label{singr}
 \twonorm{\vb_\ell^T\phi\left(\W\X^T\right)}\leq \frac{\vrn}{\sqrt{\K}}\twonorm{\y}.
\end{align}
so that a simple union bound can conclude the proof. Therefore, all that remains is to show \eqref{singr} holds. To prove the latter, note that for any two matrices $\widetilde{\W}, \W\in\R^{k\times d}$ we have
\begin{align*}
\abs{\twonorm{\phi\left(\X\widetilde{\W}^T\right)\vct{v}}-\twonorm{\phi\left(\X\W^T\right)\vct{v}}}\le& \twonorm{\phi\left(\X\widetilde{\W}^T\right)\vct{v}-\phi\left(\X\W^T\right)\vct{v}}\\
\le&\opnorm{\phi\left(\X\widetilde{\W}^T\right)-\phi\left(\X\W^T\right)}\twonorm{\vct{v}}\\
\le& \fronorm{\phi\left(\X\widetilde{\W}^T\right)-\phi\left(\X\W^T\right)}\twonorm{\vct{v}}\\
& \hspace{-100pt}\overset{(a)}{=}\fronorm{\left(\phi'\left(\Sb\odot\X\widetilde{\W}^T+(1_{k\times n}-\Sb)\odot\X\W^T\right)\right)\odot \left(\X(\widetilde{\W}-\W)^T\right)}\twonorm{\vct{v}}\\
\le&B\fronorm{\X(\widetilde{\W}-\W)^T}\twonorm{\vct{v}}\\
\le&B\opnorm{\X}\twonorm{\vct{v}}\fronorm{\widetilde{\W}-\W},
\end{align*}
where in (a) we used the mean value theorem with $\Sb$ a matrix with entries obeying $0\le \Sb_{i,j}\le 1$ and $1_{k\times n}$ the matrix of all ones. Thus, $\twonorm{\phi\left(\X\W^T\right)\vct{v}}$ is a $B\opnorm{\X}\twonorm{\vct{v}}$-Lipschitz function of $\mtx{W}$. Thus, fixing $\vb$, for a matrix $\W$ with i.i.d.~Gaussian entries 
\begin{align}
\label{tempmyflip}
\twonorm{\phi\left(\X\W^T\right)\vct{v}}\le \E\big[\twonorm{\phi\left(\X\W^T\right)\vct{v}}\big]+t,
\end{align}
holds with probability at least $1-e^{-\frac{t^2}{2B^2\twonorm{\vct{v}}^2\opnorm{\X}^2}}$. Next given $g\sim\Nn(0,1)$, we have
\begin{align}
|\E[\phi(g)]|\leq |\E[\phi(0)]|+|\E[\phi(g)-\phi(0)]|\leq B+B\E[|g|]\leq 2B\quad\text{and}\quad \text{Var}(\phi(g))\le B^2.\label{poinc stuff}
\end{align}
where the latter follows from Poincare inequality (e.g.~see \cite[p. 49]{ledoux}).
Furthermore, since $\vb$ has i.i.d.~Rademacher entries, applying Bernstein bound, event 
\begin{align}
E_{\vb}:=\{|\onebb^T\vb|^2\leq 250\log \K \tn{\vb}^2\} \label{berns v}
\end{align}
holds with probability $1-(2\K)^{-102}$. Conditioned on $E_{\vb}$, we now upper bound the expectation via
\begin{align*}
\E\big[\twonorm{\phi\left(\X\W^T\right)\vct{v}}\big]\overset{(a)}{\le}& \sqrt{\E\big[\twonorm{\phi\left(\X\W^T\right)\vct{v}}^2\big]}\\
=&\sqrt{\sum_{i=1}^n \E\big[\left(\vct{v}^T\phi(\W\vct{x}_i)\right)^2\big]}\\
\overset{(b)}{=}&\sqrt{n}\sqrt{\E_{\vct{g}\sim\mathcal{N}(\vct{0},\mtx{I}_k)}\big[\left(\vct{v}^T\phi(\vct{g})\right)^2\big]}\\
\overset{(c)}{=}&\sqrt{n}\sqrt{\twonorm{\vct{v}}^2\E_{g\sim\mathcal{N}(0,1)}\big[\left(\phi(g)-\E[\phi(g)]\right)^2\big]+(\vct{1}^T\vct{v})^2(\E_{g\sim\mathcal{N}(0,1)}[\phi(g)])^2}\\
\overset{(d)}{\le}&\sqrt{n}\twonorm{\vct{v}}\sqrt{250\times 4B^2\log (2\K) +B^2}\\
{\le}&32\sqrt{n\log (2\K) }B\twonorm{\vct{v}}.
\end{align*}
Here, (a) follows from Jensen's inequality, (b) from linearity of expectation and the fact that for $\vct{x}_i$ with unit Euclidean norm $\mtx{W}\vct{x}_i\sim\mathcal{N}(\vct{0},\mtx{I}_k)$, (c) from simple algebraic manipulations, (d) from the inequalities \eqref{berns v} and \eqref{poinc stuff}. Thus using $t=18\sqrt{n\log (2\K) }B\twonorm{\vct{v}}$ in \eqref{tempmyflip}, conditioned on $E_{\vb}$ we conclude that
\begin{align}
\twonorm{\phi\left(\X\W^T\right)\vct{v}}\le&50\sqrt{n\log (2\K) }B\twonorm{\vct{v}}=50\sqrt{n\log (2\K) }B\sqrt{k}\frac{\vrn\twonorm{\y}}{50B\sqrt{\K\log (2\K) kn}}=\frac{\vrn \tn{\y}}{\sqrt{\K}},\label{fin fin eq}
\end{align}
holds with probability at least $1-\exp(-102\log (2\K)\frac{n}{\opnorm{\X}^2})\geq 1-(2\K)^{-102}$ where we used $n\geq \|\X\|^2$. Using a union bound over $E_{\vb}$ and the conditional concentration over $\W$, the overall probability of success in \eqref{fin fin eq} is at least $1-(2\K)^{-101}$ concluding the proof of \eqref{singr} and the Lemma.
\end{proof}

\subsection{Rademacher complexity and generalization bounds}
\label{radam}
In this section we state and prove some Rademacher complexity results that will be used in our generalization bounds. We begin with some basic notation regarding Rademacher complexity.
Let $\Fc$ be a function class. Suppose $f\in\Fc$ maps $\R^d$ to $\R^K$. Let $\{\vct{\eps}_i\}_{i=1}^n$ be i.i.d.~vectors in $\R^K$ with  i.i.d.~Rademacher variables. Given i.i.d.~samples $\Sc=\{(\x_i,\y_i)\}_{i=1}^n\sim \Dc$, we define the empirical Rademacher complexity to be
\[
\mathcal{R}_{\Sc}(\Fc)=\frac{1}{n}\E\bigg[\sup_{f\in\Fc}\text{ }\sum_{i=1}^n \vct{\eps}_{i}^T f(\x_i)\bigg].
\]
We begin by stating a vector contraction inequality by Maurer \cite{maurer2016vector}. This is obtained by setting $h_i(f(\x_i))=h(\y_i,f(\x_i))$ in Corollary 4 of \cite{maurer2016vector}.
\begin{lemma} \label{lem multi}Let $f(\cdot):\R^d\rightarrow \R^\K$ and let $\ell:\R^\K\times \R^K\rightarrow\R$ be a $1$ Lipschitz loss function with respect to second variable. Let $\{\eps_i\}_{i=1}^n$ be i.i.d.~Rademacher variables. Given i.i.d.~samples $\{(\x_i,\y_i)\}_{i=1}^n$, define
\[
\mathcal{R}_{\Sc}(\ell,\Fc)=\E\bigg[\sup_{f\in \Fc}\sum_{i=1}^n\eps_i \ell(\y_i,f(\x_i))\bigg].
\]
We have that
\[
\mathcal{R}_{\Sc}(\ell,\Fc)\leq \sqrt{2}\mathcal{R}_{\Sc}(\Fc).
\]
\end{lemma}

Combining the above result with standard generalization bounds based on Rademacher complexity \cite{bartlett2002rademacher} allows us to prove the following result.
\begin{lemma} \label{standard rad}Let $\ell(\cdot,\cdot):\R^\K\times \R^\K\rightarrow [0,1]$ be a $1$ Lipschitz loss function. Given i.i.d.~samples $\{(\x_i,\y_i)\}_{i=1}^n$, consider the empirical loss
\[
\Lc(f,\ell)=\frac{1}{n}\sum_{i=1}^n \ell(\y_i,f(\x_i)).
\]
With probability $1-\delta$ over the samples, for all $f\in \Fc$, we have that
\[
\E[\Lc(f,\ell)]\leq \Lc(f,\ell)+2\sqrt{2}\mathcal{R}_{\Sc}(\Fc)+\sqrt{\frac{5\log(2/\delta)}{n}}
\]
\end{lemma}
\begin{proof} Based on \cite{bartlett2002rademacher}, 
\[
\E[\Lc(f,\ell)]\leq \Lc(f,\ell)+2\mathcal{R}_{\Sc}(\ell,\Fc)+\sqrt{\frac{5\log(2/\delta)}{n}}
\]
holds with $1-\delta$ probability. Combining the latter with Lemma \ref{lem multi} completes the proof.
\end{proof}

\begin{lemma} \label{rad rad bound} Consider a neural network model of the form $\x\mapsto f(\x;\Vb,\W)=\Vb\phi\left(\W\x\right)$ with $\W\in\R^{k\times d}$ and $\Vb\in\R^{\K\times k}$ denoting the input and output weight matrices. Suppose $\mtx{V}_0\in\R^{K\times k}$ is a matrix obeying $\tin{\Vb_0}\leq \vrn/\sqrt{k\K}$. Also let $\W_0\in\R^{k\times d}$ be a reference input weight matrix. Furthermore, we define the neural network function  space parameterized by the weights as follows
\begin{align}
&\mathcal{F}_{\Vc,\Wc}=\Big\{f(\x;\Vb,\W)\quad\text{such that}\quad\Vb\in\Vc\quad\text{and}\quad\W\in\Wc\Big\}\quad\text{with}\quad\Vc=\bigg\{\Vb: \fronorm{\Vb-\Vb_0}\le \frac{\vrn M_{\mathcal{V}}}{\sqrt{\K k}}\bigg\}\nn\\
&\quad\quad\text{and}\quad\Wc=\bigg\{\W: \fronorm{\W-\W_0}\le M_{\mathcal{W}}\quad\text{and}\quad \|\W-\W_0\|_{2,\infty}\le \frac{R}{\sqrt{k}}\Big\}.\label{space func}
\end{align}
Additionally, assume the training data $\{(\vct{x}_i,\y_i)\}_{i=1}^n$are generated i.i.d. with the input data points of unit Euclidean norm (i.e.~$\tn{\x_i}=1$). Also, define the average energy at $\W_0$ as
\[
E=\left(\frac{1}{kn}\sum_{i=1}^n\tn{\phi(\W_0\x_i)}^2\right)^{1/2}.
\]
Also let $\{\vct{\xi}_i\}_{i=1}^n\in\R^\K$ be i.i.d.~vectors with i.i.d.~Rademacher entries and define the empirical Rademacher complexity
\begin{align*}
\mathcal{R}_{\mathcal{S}}\left(\mathcal{F}_{\Vc,\Wc}\right):=&\frac{1}{n}\E\Bigg[\underset{f\in\mathcal{F}_{\Vc,\Wc}}{\sup}\text{ }\sum_{i=1}^n \vct{\xi}_i^Tf(\x_i)\Bigg].
\end{align*}
Then,
\begin{align}
\mathcal{R}_{\mathcal{S}}\left(\mathcal{F}_{\Vc,\Wc}\right)\le  \vrn B\left(\frac{M_{\mathcal{W}}+EM_{\mathcal{V}}}{\sqrt{n}}+\frac{R^2+M_{\mathcal{W}}M_{\mathcal{V}}}{\sqrt{k}}\right).\label{emp rad comp}
\end{align} 
\end{lemma}
\begin{proof}
We shall use $\w_\ell$ to denote the rows of $\W$ (same for $\W_0,\Vb,\Vb_0$). We will approximate $\phi\left(\langle \w_\ell,\x_i\rangle\right)$ by its linear approximation $\phi\left(\langle \w_\ell^0, \x_i\rangle\right)+\phi'(\langle\w_\ell^0,\x_i\rangle)\left(\langle\w_\ell-\w_\ell^0,\x_i\rangle\right)$ via the second order Taylor's mean value theorem. We thus have
\begin{align*}
\mathcal{R}_{\mathcal{S}}\left(\mathcal{F}_{\Vc,\Wc}\right)\le&\frac{1}{n}\E\Bigg[\sum_{i=1}^n\vct{\xi}^T_i\Vb_0\phi\left(\W_0\x_i\right)\Bigg]\\
&+\underbrace{\frac{1}{n}\E\Bigg[\underset{\W\in\Wc}{\sup}\text{ }\sum_{i=1}^n \vct{\xi}^T_i\Vb_0\text{diag}\left(\phi'\left(\W_0\x_i\right)\right)\left(\W-\W_0\right)\x_i\Bigg]}_{\mathcal{R}_1}\\
&+\underbrace{\frac{1}{2n}\E_{\xi_{i,j}\distas\pm 1}\Bigg[\underset{\W\in\Wc}{\sup}\text{ }\sum_{i=1}^n\sum_{\ell=1}^k  \sum_{j=1}^\K\xi^T_{i,j}\mtx{V}_{j,\ell}^{0}\phi''\left((1-t_{i\ell})\langle\w_\ell^0,\x_i\rangle+t_{i\ell}\langle\w_\ell,\x_i\rangle\right)\left(\langle\w_\ell-\w_\ell^0,\x_i\rangle\right)^2\Bigg]}_{\mathcal{R}_2}\\
&+\underbrace{\frac{1}{n}\E\bigg[\sup_{\Vb\in\Vc,\W\in\Wc}\sum_{i=1}^n\vct{\xi}_i^T(\Vb-\Vb_0)(\phi(\W\x_i)-\phi(\W_0\x_i))\bigg]}_{\mathcal{R}_3}\\
&+\underbrace{\frac{1}{n}\E\bigg[\sup_{\Vb\in\Vc}\sum_{i=1}^n\vct{\xi}_i^T(\Vb-\Vb_0)\phi(\W_0\x_i)\bigg]}_{\mathcal{Rc}_4}
\end{align*}
We proceed by bounding each of these four terms. For the first term note that
\begin{align*}
\mathcal{R}_1\le& \frac{1}{n}\E\Bigg[\underset{\fronorm{\W-\W_0}\le M_{\mathcal{W}}}{\sup}\text{ }\sum_{i=1}^n \vct{\xi}^T_i\Vb_0\text{diag}\left(\phi'\left(\W_0\x_i\right)\right)\left(\W-\W_0\right)\x_i\Bigg]\\
\le& \frac{1}{n}\E\Bigg[\underset{\fronorm{\W-\W_0}\le M_{\mathcal{W}}}{\sup}\text{ }\Big\langle\sum_{i=1}^n \text{diag}\left(\phi'\left(\W_0\x_i\right)\right)\Vb_0^T\vct{\xi}_i\x_i^T,\W-\W_0\Big\rangle\Bigg]\\
\le&\frac{M_{\mathcal{W}}}{n}\E\Bigg[\fronorm{\left(\sum_{i=1}^n \text{diag}\left(\Vb_0^T\vct{\xi}_i\right)\phi'\left(\W_0\x_i\right)\x_i^T\right)}\Bigg]\\
\le&\frac{M_{\mathcal{W}}}{n}\E\Bigg[\fronorm{\left(\sum_{i=1}^n \text{diag}\left(\Vb_0^T\vct{\xi}_i\right)\phi'\left(\W_0\x_i\right)\x_i^T\right)}^2\Bigg]^{1/2}\\
=&\frac{M_{\mathcal{W}}}{n}\Bigg[\sum_{i=1}^n\E\fronorm{ \text{diag}\left(\Vb_0^T\vct{\xi}_i\right)\phi'\left(\W_0\x_i\right)\x_i^T}^2\Bigg]^{1/2}\\
=&\frac{M_{\mathcal{W}}}{n}\Bigg[\sum_{i=1}^n\E\tn{ \text{diag}\left(\Vb_0^T\vct{\xi}_i\right)\phi'\left(\W_0\x_i\right)}^2\Bigg]^{1/2}\\
\le&\frac{BM_{\mathcal{W}}}{n}\Bigg[\sum_{i=1}^n\E\tn{ \Vb_0^T\vct{\xi}_i}^2\Bigg]^{1/2}\\
\le&\frac{BM_{\mathcal{W}}}{n}\tf{\Vb_0}\\
\le& \frac{BM_{\mathcal{W}}\vrn}{\sqrt{n}},
\end{align*}
where in the last inequality we used the fact that $\tf{\Vb_0}\leq \vrn$. For the second term note that
\begin{align*}
\mathcal{R}_2\le& \frac{1}{2n}\E\Bigg[\underset{\|\W-\W_0\|_{2,\infty}\le R}{\sup}\text{ }\sum_{i=1}^n\sum_{\ell=1}^k  \sum_{j=1}^\K\xi_{i,j}\vb_{0,j,\ell}\phi''\left((1-t_{i\ell})\langle\w_\ell^0,\x_i\rangle+t_{i\ell}\langle\w_\ell,\x_i\rangle\right)\left(\langle\w_\ell-\w_\ell^0,\x_i\rangle\right)^2\Bigg]\\
\le&\frac{1}{2n}\sum_{\ell=1}^k \E\Bigg[\underset{\tn{\w_\ell-\w_\ell^0}\le R}{\sup}\text{ }\sum_{i=1}^n\abs{ \sum_{j=1}^\K\xi_{i,j}\vb_{0,j,\ell}}\abs{\phi''\left((1-t_{i\ell})\langle\w_\ell^0,\x_i\rangle+t_{i\ell}\langle\w_\ell,\x_i\rangle\right)}\left(\langle\w_\ell-\w_\ell^0,\x_i\rangle\right)^2\Bigg]\\
\le&\frac{1}{2kn}\sum_{\ell=1}^k\sum_{i=1}^n  \E\Bigg[\abs{\sum_{j=1}^\K\xi_{i,j}\vb_{0,j,\ell}}\Bigg]R^2B\\
\le &\frac{BR^2}{2k} \|\Vb_0^T\|_{2,1}\\
\le& \frac{BR^2\vrn}{2\sqrt{k}}.
\end{align*} 
In the above we used $\|\mtx{M}\|_{2,1}$ for a matrix $\mtx{M}$ to denote the sum of the Euclidean norm of the rows of $\mtx{M}$. We also used the fact that $ \|\Vb_0^T\|_{2,1}\leq \vrn\sqrt{k}$. To bound the third term note that
\begin{align*}
\mathcal{R}_3&=\frac{1}{n}\E\bigg[\sup_{\Vb\in\Vc,\W\in\Wc}\sum_{i=1}^n\vct{\xi}_i^T(\Vb-\Vb_0)(\phi(\W\x_i)-\phi(\W_0\x_i))\bigg]\\
&\le\frac{1}{n}\E\bigg[\sup_{\Vb\in\Vc,\W\in\Wc}\sum_{i=1}^n\fronorm{\Vb-\Vb_0}\twonorm{\vct{\xi}_i}\twonorm{(\phi(\W\x_i)-\phi(\W_0\x_i))}\bigg]\\
&\le\frac{\vrn M_{\mathcal{V}}}{n\sqrt{kK}}\E\Big[\sup_{\W\in\Wc}\sum_{i=1}^n\tn{\vct{\xi}_i}\tn{\phi(\W\x_i)-\phi(\W_0\x_i)}\Big]\\
&\le\frac{\vrn M_{\mathcal{V}}}{n\sqrt{k}}\cdot\sup_{\W\in\Wc}\sum_{i=1}^n\tn{\phi(\W\x_i)-\phi(\W_0\x_i)}\\ 
&\le\frac{\vrn B M_{\mathcal{V}}}{n\sqrt{k}}\cdot\sup_{\W\in\Wc}\sum_{i=1}^n\tn{(\W-\W_0)\x_i}\\ 
&\le\frac{\vrn B M_{\mathcal{V}}}{n\sqrt{k}}\cdot\sup_{\W\in\Wc}\sum_{i=1}^n\fronorm{(\W-\W_0)}\\ 
&=\frac{\vrn B M_{\mathcal{V}}}{\sqrt{k}}\cdot\sup_{\W\in\Wc}\fronorm{(\W-\W_0)}\\ 
&=\frac{\vrn BM_{\mathcal{V}}M_{\mathcal{W}}}{\sqrt{k}}.
\end{align*}
Finally, to bound the fourth term note that we have
\begin{align*}
\mathcal{R}_4&=\frac{1}{n}\E\bigg[\sup_{\Vb\in\Vc}\sum_{i=1}^n\vct{\xi}_i^T(\Vb-\Vb_0)\phi(\W_0\x_i)\bigg]\\
&=\frac{1}{n}\E\bigg[\sup_{\fronorm{\Vb-\Vb_0}\le \frac{\vrn M_{\mathcal{V}}}{\sqrt{kK}}}\Big\langle\sum_{i=1}^n\vct{\xi}_i\phi(\W_0\x_i)^T,(\Vb-\Vb_0)\Big\rangle\bigg]\\
&=\frac{\vrn M_{\mathcal{V}}}{n\sqrt{kK}}\E\bigg[\fronorm{\sum_{i=1}^n\phi(\W_0\x_i)\vct{\xi}_i^T}\bigg]\\
&\le\frac{\vrn M_{\mathcal{V}}}{n\sqrt{kK}}\left(\E\bigg[\fronorm{\sum_{i=1}^n\phi(\W_0\x_i)\vct{\xi}_i^T}^2\bigg]\right)^{1/2}\\
&=\frac{\vrn M_{\mathcal{V}}}{n\sqrt{kK}}\left(\sum_{i=1}^n\E\bigg[\fronorm{\phi(\W_0\x_i)\vct{\xi}_i^T}^2\bigg]\right)^{1/2}\\
&=\frac{\vrn M_{\mathcal{V}}}{\sqrt{n}}\left(\frac{1}{kn}\sum_{i=1}^n\E\bigg[\fronorm{\phi(\W_0\x_i)}^2\bigg]\right)^{1/2}\\
&=\frac{\vrn EM_{\mathcal{V}}}{\sqrt{n}}\\
&\le\frac{\vrn BEM_{\mathcal{V}}}{\sqrt{n}}
\end{align*}
Combining these four bounds we conclude that
\[
\mathcal{R}_{\mathcal{S}}\left(\mathcal{F}_{\Vc,\Wc}\right)\leq \vrn B\left(\frac{M_{\mathcal{W}}}{\sqrt{n}}+\frac{R^2}{\sqrt{k}}+\frac{M_{\mathcal{V}}M_{\mathcal{W}}}{\sqrt{k}}+\frac{EM_{\mathcal{V}}}{\sqrt{n}}\right),
\]
concluding the proof of \eqref{emp rad comp}.
\end{proof}
Next we state a crucial lemma that connects the test error measured by any Lipschitz loss to that of the quadratic loss on the training data. 
\begin{lemma} \label{generalize corollary} Consider a one-hidden layer neural network with input to output mapping of the form $\vct{x}\in\R^d\mapsto f(\vct{x};\Vb,\mtx{W})=\mtx{V}\phi\left(\W\vct{x}\right)\in\R^K$ with $\W\in\R^{k\times d}$ denoting the input-to-hidden weights and $\V\in\R^{k\times K}$ the hidden-to-output weights. Suppose $\mtx{V}_0\in\R^{K\times k}$ is a matrix obeying $\tin{\Vb_0}\leq \vrn/\sqrt{k\K}$. Also let $\W_0\in\R^{k\times d}$ be a reference input weight matrix. Also define the empirical losses
\begin{align*}
\mathcal{L}(\V,\W)=\frac{1}{n}\sum_{i=1}^n \tn{\y_i-f(\x_i;\Vb,\W)}^2,
\end{align*}
and
\begin{align*}
\mathcal{L}(f,\ell)=\frac{1}{n}\sum_{i=1}^n \ell(f(\x_i;\Vb,\W),\y_i),
\end{align*}
with $\ell:\R^\K\times \R^\K\rightarrow [0,1]$ a one Lipschitz loss function obeying $\ell(\y,\y)=0$. Additionally, assume the training data $\{(\vct{x}_i,\y_i)\}_{i=1}^n$are generated i.i.d. according to a distribution $\mathcal{D}$ with the input data points of unit Euclidean norm (i.e.~$\tn{\x_i}=1$). Also, define the average energy at $\W_0$ as
\[
E=\left(\frac{1}{kn}\sum_{i=1}^n\tn{\phi(\W_0\x_i)}^2\right)^{1/2}.
\]
Then for all $f$ in the function class $\mathcal{F}_{\Wc}$ given by \eqref{space func}
\begin{align}
\E[\Lc(f,\ell)]\leq \sqrt{\Lc(\Vb,\W)}+ 2\sqrt{2}\vrn B\left(\frac{M_{\mathcal{W}}}{\sqrt{n}}+\frac{R^2}{\sqrt{k}}+\frac{\vrn BM_{\mathcal{V}}M_{\mathcal{W}}}{\sqrt{k}}+\frac{EM_{\mathcal{V}}}{\sqrt{n}}\right)+\sqrt{\frac{5\log(2/\delta)}{n}},\label{ell pop bound}
\end{align}
holds with probability at least $1-\delta$. Furthermore, Suppose labels are one-hot encoded and thus unit Euclidian norm. Given a sample $(\x,\y)\in\R^d\times \R^\K$ generated according to the distribution $\mathcal{D}$, define the population classification error
\[
\err{\W}=\Pro(\arg\max_{1\leq \ell\leq K}\y_i\neq \arg\max_{1\leq \ell\leq K}f_i(\x,\W)).
\]
Then, we also have
\begin{align}
\err{\W}\leq 2\left[\sqrt{\Lc(\Vb,\W)}+ 2\sqrt{2}\vrn B\left(\frac{M_{\mathcal{W}}}{\sqrt{n}}+\frac{R^2}{\sqrt{k}}+\frac{\vrn BM_{\mathcal{V}}M_{\mathcal{W}}}{\sqrt{k}}+\frac{EM_{\mathcal{V}}}{\sqrt{n}}\right)+\sqrt{\frac{5\log(2/\delta)}{n}}\right].\label{err pop bound}
\end{align}
\end{lemma}
\begin{proof} 
To begin first note that any 1-Lipschitz $\ell$ with $\ell(\y,\y)=0$ obeys $\ell(\y,\hat{\y})\leq \tn{\y-\hat{\y}}$. Thus, we have
\[
\Lc(f,\ell)\leq \frac{1}{n}\sum_{i=1}^n\tn{\y_i-f(\x_i;\Vb,\W)}\leq \sqrt{\Lc(\Vb,\W)},
\] 
where the last inequality follows from Cauchy-Schwarz. Consequently, applying Lemmas \ref{standard rad} and \ref{rad rad bound} we conclude that
\begin{align*}
\E[\Lc(f,\ell)]&\leq \Lc(f,\ell)+2\sqrt{2}\cdot\mathcal{R}_{\Sc}(\Fc)+\sqrt{\frac{5\log(2/\delta)}{n}},\\
&\leq\sqrt{\Lc(\W)}+ 2\sqrt{2}\vrn B\left(\frac{M_{\mathcal{W}}}{\sqrt{n}}+\frac{R^2}{\sqrt{k}}+\frac{\vrn BM_{\mathcal{V}}M_{\mathcal{W}}}{\sqrt{k}}+\frac{EM_{\mathcal{V}}}{\sqrt{n}}\right)+\sqrt{\frac{5\log(2/\delta)}{n}},
\end{align*}
which yields the first statement.

To prove the second statement on classification accuracy, we pick the $\ell$ function as follows
\[
\ell(\y,\hat{\y})=\min(1,\tn{\y-\hat{\y}}).
\]
Note that, given a sample $(\x,\y)\in\R^d\times \R^\K$ with one-hot encoded labels, if 
\begin{align*}
\arg\max_{1\leq \ell\leq K}\y_\ell\neq \arg\max_{1\leq \ell\leq K}f_\ell(\x;\Vb,\W),
\end{align*} 
this implies 
\[
\ell\left(\y,f(\x;\Vb,\W)\right)\geq 0.5.
\] 
Combining the latter with Markov inequality we arrive at
\[
\err{\W}\leq 2\E_{(\x,\y)\sim\Dc}[\ell(\y,f(\x;\Vb,\W))]=2\E[\Lc(\ell,\W)].
\]
Now since $\ell$ is $1$ Lipschitz and bounded, it obeys \eqref{ell pop bound}, which combined with the above identity yields \eqref{err pop bound}, completing the proof.
\end{proof}
\subsection{Proofs for neural nets with arbitrary initialization (Proof of Theorem \ref{nn deter gen})}
In this section we prove Theorem \ref{nn deter gen}. We first discuss a preliminary optimization result in Section \ref{prelimopt}. Next, in Section \ref{mainopt} we build upon this result to prove our main optimization result. Finally, in Section \ref{maingen} we use these optimization results to prove our main generalization result, completing the proof of Theorem \ref{nn deter gen}.
\subsubsection{Preliminary Optimization Result}
\label{prelimopt}
\begin{lemma} [Deterministic convergence guarantee] \label{nn deter}Consider a one-hidden layer neural net of the form $\vct{x}\mapsto f(\vct{x};\mtx{W}):=\Vb\phi(\mtx{W}\vct{x})$ with input weights $\W\in\R^{k\times d}$ and output weights $\Vb\in\R^{K\times k}$ and an activation $\phi$ obeying $\abs{\phi(0)}\le B$, $\abs{\phi'(z)}\le B$, and $\abs{\phi''(z)}\le B$ for all $z$. Also assume $\Vb$ is fixed with all entries bounded by $\infnorm{\Vb}\le\frac{\vrn}{\sqrt{kK}}$ and we train over $\W$ based on the loss 
\begin{align*}
\mathcal{L}(\W)=\frac{1}{2}\sum_{i=1}^n\twonorm{f(\vct{x}_i;\W)-\y_i}^2.
\end{align*}
Also, consider a point $\W_0\in\R^{k\times d}$ with $\mtx{J}$ an $(\epsilon_0,\vrn B\|\X\|)$ reference Jacobian associated with $\mathcal{J}(\W_0)$ per Definition \ref{tjac}. Furthermore, define the information $\calF$ and nuisance $\calS$ subspaces and the truncated Jacobian $\mtx{J}_{\calF}$ associated with the reference Jacobian $\mtx{J}$ based on a cut-off spectrum value of $\alpha$ per Definition \ref{tjac2}. Let the initial residual vector be $\rb_0=\y-f(\W_0)\in\R^{n\K}$. Furthermore, assume
\begin{align}
\label{epsass}
 \eps_0\leq \frac{\alpha}{5}\min\left(\delta,\sqrt{\frac{\delta\alpha}{\Gamma\vrn B\|\X\|}}\right)
 \end{align}
and
\begin{align}
k\geq 400\frac{\vrn^6B^6\opnorm{\X}^6\Gamma^2}{\delta^2\alpha^8}\left(\upp+\delta\Gamma\tn{\rb_0}\right)^2,\label{fund k bound}
\end{align}
with $0\le \delta\le1$ and $\Gamma\ge 1$. We run gradient descent iterations of the form $\W_{\tau+1}=\W_{\tau}-\eta\nabla \mathcal{L}(\W_\tau)$ starting from $\W_0$ with step size $\eta$ obeying $\eta\le \frac{1}{\vrn^2B^2\opnorm{\X}^2}$. Then for all iterates $\tau$ obeying $0\le \tau\le T:=\frac{\Gamma}{\eta\alpha^2}$
\begin{align}
\tf{\W_\tau-\W_0}\le& \frac{\upp}{\alpha}+\delta\frac{\Gamma}{\alpha}\tn{\rb_0}
.\label{main res eq22}\\
\trow{\W_{\tau}-\W_0}\le& \frac{ 2\vrn B\Gamma\opnorm{\X}}{\sqrt{k}\alpha^2}\tn{\rb_0}.\label{row boundsss}
\end{align}
Furthermore, after $\tau=T$ iteration we have
\begin{align}
\tn{\rb_{T}}\le  e^{-\Gamma}\tn{\Pi_{{\Rc}}(\rb_0)}+\tn{ \Pi_{\mathcal{N}}(\rb_0)}+ \frac{\delta\alpha}{\vrn B\opnorm{\X}}\tn{\rb_0}.\label{eqqq22}
\end{align}
\end{lemma}
\begin{proof} To prove this lemma we wish to apply Theorem \ref{many step thm}. We thus need to ensure that the assumptions of this theorem are satisfied. To do this note that by Lemma \ref{props lem} Assumption \ref{ass2} holds with $\beta=\vrn B\opnorm{\X}$. Furthermore, we pick $\eps=\frac{\delta\alpha^3}{5\Gamma \vrn^2 B^2\opnorm{\X}^2}=\frac{\delta\alpha^3}{5\Gamma \beta^2}$ which together with \eqref{epsass} guarantees \eqref{cndd} holds.   We now turn our attention to verifying Assumption \ref{ass3}. To this aim note that for all $\W\in\R^{k\times d}$ obeying
\begin{align*}
\fronorm{\W-\W_0}\le R:=2\left(\frac{\upp}{\alpha}+\delta\frac{\Gamma}{\alpha}\tn{\rb_0}\right)
\end{align*}
as long as \eqref{fund k bound} holds  by Lemma \ref{props lem} we have
\begin{align*}
\|\Jc(\W)-\Jc(\W_0)\|\le&B\sqrt{\K}\infnorm{\Vb}\opnorm{\mtx{X}}R\\
\le&\frac{\vrn}{\sqrt{k}} B\opnorm{\mtx{X}}R\\
=&\frac{\delta\alpha^3}{10\Gamma\vrn^2 B^2\opnorm{\X}^2}\frac{\frac{20\Gamma\vrn^3 B^3\opnorm{\X}^3}{\delta\alpha^4}\left(\upp+\delta\Gamma\tn{\rb_0}\right)}{\sqrt{k}}\\
\le&\frac{\delta\alpha^3}{10\Gamma\vrn^2 B^2\opnorm{\X}^2}\\
=&\frac{\eps}{2}.
\end{align*}
Thus, Assumption \ref{ass3} holds with $\fronorm{\W-\W_0}\le R:=2\left(\frac{\upp}{\alpha}+\delta\frac{\Gamma}{\alpha}\tn{\rb_0}\right)$. Now that we have verified that the assumptions of Theorem \ref{many step thm} hold so do its conclusions and thus \eqref{main res eq22} and \eqref{eqqq22} hold.

We now turn our attention to proving the row-wise bound \eqref{row boundsss}. To this aim let $\w_{\ell}^{(\tau)}$ denote the $\ell$th row of $\W_\tau$. Also note that 
\[
\grad{\w_\ell}=\text{$\ell$th row of }\text{mat}(\mathcal{J}(\mtx{W})^T\rb_\tau).
\]
Hence, using Lemma \ref{props lem} equation \eqref{rowbndm} we conclude that
\[
\twonorm{\grad{\w_{\ell}^{(\tau)}}}\le B\sqrt{K}\infnorm{\Vb}\opnorm{\X}\twonorm{\rb_\tau} \le \frac{\vrn B\opnorm{\X}}{\sqrt{k}}\tn{\rb_\tau}.
\]
Consequently, for any row $1\leq \ell\leq k$, we have
\begin{align}
\twonorm{\w_{\ell}^{(\tau)}-\w_{\ell}^{(0)}}\leq \eta \frac{\vrn B\opnorm{\X}}{\sqrt{k}}\sum_{t=0}^{\tau-1}\tn{\rb_t}.\label{roww bound}
\end{align}
To bound the right-hand side we use the triangular inequality combined with \eqref{rrr bound} and \eqref{eee bound} to conclude that
\begin{align}
\label{A7}
\eta \sum_{t=0}^{\tau-1}\tn{\rb_\tau}\le&\eta\sum_{t=0}^{\tau-1}\tn{\rbb_\tau}+\eta\sum_{t=0}^{\tau-1}\tn{\rb_\tau-\rbb_\tau}\nn\\
 \le&\frac{\Gamma}{\alpha^2}\tn{\rb_0}+\frac{2\Gamma^2(\eps\beta+\eps_0^2)}{\alpha^4} \tn{\rb_0}\nn\\
 =&\frac{2\Gamma (\eps_0^2+\eps\beta)+\alpha^2}{\alpha^4}\Gamma\tn{\rb_0}\nn\\
 \le& 2\frac{\Gamma}{\alpha^2}\tn{\rb_0},
\end{align}
where in the last inequality we used the fact that $\eps_0^2\le \frac{\alpha^2}{25\Gamma}$ per \eqref{epsass} and $\epsilon\beta=\frac{\delta\alpha^3}{5\Gamma \beta}\le \frac{\alpha^2}{5\Gamma}$ per our choice of $\epsilon$. Combining \eqref{roww bound} and \eqref{A7}, we obtain
\[
\twonorm{\w_{\ell}^{(\tau)}-\w_{\ell}^{(0)}}\leq \frac{ 2\vrn B\|\X\|\Gamma}{\sqrt{k}\alpha^2}\tn{\rb_0},
\]
completing the proof of \eqref{row boundsss} and the theorem.
\end{proof}
\subsubsection{Main Optimization Result}
\label{mainopt}
\begin{lemma} [Deterministic optimization guarantee] \label{nn deter2} Consider the setting and assumptions of Lemma \ref{nn deter}. Also assume $\tn{\Pi_{\Rc}(\rb_0)}\geq c\tn{\rb_0}$ for a constant $c>0$ if $\eps_0>0$. Furthermore, assume
\begin{align}
\eps_0^2\le \frac{\alpha^2}{25}\min\left(c\frac{\upp\alpha}{\vrn B\Gamma^2\tn{\rb_0}\opnorm{\X}},\frac{\zeta^2\vrn^2 B^2\opnorm{\X}^2}{\alpha^2},\frac{\zeta}{\Gamma}\right),\label{ref bound eps}
\end{align}
and
\begin{align}
k\geq 1600\left(\frac{\alpha}{\zeta\vrn B\opnorm{\X}}+\frac{\Gamma \tn{\rb_0}}{\upp}\right)^2\frac{\vrn^6B^6\opnorm{\X}^6\Gamma^2\upp^2}{\alpha^8},\label{k bound det}
\end{align}
and $\Gamma\ge 1$. We run gradient descent iterations of the form $\W_{\tau+1}=\W_{\tau}-\eta\nabla \mathcal{L}(\W_\tau)$ starting from $\W_0$ with step size $\eta$ obeying $\eta\le \frac{1}{\vrn^2B^2\opnorm{\X}^2}$. Then for all iterates $\tau$ obeying $0\le \tau\le T:=\frac{\Gamma}{\eta\alpha^2}$
\begin{align}
&\tf{\W_\tau-\W_0}\leq \frac{2\upp}{\alpha}.\label{main res eq23}\\
&\trow{\W_{\tau}-\W_0}\leq \frac{ 2\vrn B\Gamma\opnorm{\X}}{\sqrt{k} \alpha^2} \tn{\rb_0}.\label{row boundsss2}
\end{align}
Furthermore, after $\tau=T$ iteration we have
\begin{align}
&\twonorm{f(\W_T)-\y}\leq  e^{-\Gamma}\tn{\Pi_{{\Rc}}(\rb_0)}+\tn{ \Pi_{\mathcal{N}}(\rb_0)}+\zeta\tn{\rb_0}.\label{eqqq23}
\end{align}
\end{lemma}
\begin{proof} To prove this lemma we aim to substitute
\begin{align}
\label{delchoice}
\delta=\min\left(\frac{\zeta\vrn B\opnorm{\X}}{\alpha},\frac{\upp}{\Gamma \tn{\rb_0}}\right)\leq 1,
\end{align}
in Theorem \ref{nn deter}. To do this we need to verify the assumptions of Theorem \ref{nn deter}. To this aim note that the choice of $\delta$ from \eqref{delchoice} combined with \eqref{k bound det} ensures that
\begin{align*}
k\ge&1600\left(\frac{\alpha}{\zeta\vrn B\opnorm{\X}}+\frac{\Gamma \tn{\rb_0}}{\upp}\right)^2\frac{\vrn^6B^6\opnorm{\X}^6\Gamma^2\upp^2}{\alpha^8}\\
\ge&\max\left(\frac{\alpha}{\zeta\vrn B\opnorm{\X}},\frac{\Gamma \tn{\rb_0}}{\upp}\right)^2\frac{1600\vrn^6B^6\opnorm{\X}^6\Gamma^2\upp^2}{\alpha^8}\\
\ge&\frac{1}{\min\left(\frac{\zeta\vrn B\opnorm{\X}}{\alpha},\frac{\upp}{\Gamma \tn{\rb_0}}\right)^2}\frac{1600\vrn^6B^6\opnorm{\X}^6\Gamma^2\upp^2}{\alpha^8}\\
=&  \frac{1600\Gamma^2\vrn^6B^6\opnorm{\X}^6\upp^2}{\delta^2\alpha^8}\\
=&  400\frac{\Gamma^2\vrn^6B^6\opnorm{\X}^6(\upp+\upp)^2}{\delta^2\alpha^8}\\
\ge&400\frac{\Gamma^2\vrn^6B^6\opnorm{\X}^6(\upp+\delta \Gamma \tn{\rb_0})^2}{\delta^2\alpha^8},
\end{align*}
so that \eqref{fund k bound} holds. We thus turn our attention to proving \eqref{epsass}. If $\eps_0=0$, the statement already holds. Otherwise, note that based on Lemma \ref{bndearlystopval} equation \eqref{low bound} we have
\begin{align}
\upp\geq \frac{\alpha\tn{\Pi_{\Rc}(\rb_0)}}{\lambda_1}\ge \frac{\alpha\tn{\Pi_{\Rc}(\rb_0)}}{\vrn B \|\X\|}\geq \frac{c\alpha\tn{\rb_0}}{\vrn B \|\X\|}\quad\Rightarrow\quad \frac{\upp}{c \tn{\rb_0}}\geq \frac{\alpha}{\vrn B\opnorm{\X}}.\label{rhs bound}
\end{align}
Recall that $\alpha=\frac{\nu\alpha_0}{\sqrt{K}}$, which implies that
\begin{itemize}
\item If $\delta=\frac{\zeta\vrn B\opnorm{\X}}{\alpha}$: For \eqref{epsass} to hold it suffices to have $\eps_0\le\frac{\alpha}{5} \min\left(\frac{\zeta\vrn B\opnorm{\X}}{\alpha},\sqrt{\frac{\zeta}{\Gamma}}\right)$.
\item If $\delta=\frac{\upp}{\Gamma \tn{\rb_0}}$: For \eqref{epsass} to hold it suffices to have $\eps_0\le \frac{\alpha}{5}\sqrt{c\frac{\upp\alpha}{\vrn B\Gamma^2\tn{\rb_0}\opnorm{\X}}}$ as based on \eqref{rhs bound} we have $\sqrt{\frac{c\alpha}{\beta\Gamma}}=\sqrt{\frac{c\alpha}{\vrn B\Gamma\opnorm{\X}}}\le \sqrt{\delta}$
\begin{align*}
 \eps_0\le&\frac{\alpha}{5}\sqrt{\frac{c\upp\alpha}{\vrn B\Gamma^2\tn{\rb_0}\opnorm{\X}}}\\
 =&\frac{\alpha}{5}\sqrt{\delta}\sqrt{\frac{c\alpha}{\Gamma\beta}}\\
=&\frac{\alpha}{5}\sqrt{\delta}\cdot\min\left(\sqrt{\delta},\sqrt{\frac{c\alpha}{\Gamma\beta}}\right)\\
=& \frac{\alpha}{5}\min\left(\delta,\sqrt{\frac{\delta\alpha}{\Gamma\beta}}\right)
\end{align*}
\end{itemize}
Combining the latter two cases as long as
\begin{align*}
\eps_0^2\le \frac{\alpha^2}{25}\min\left(c\frac{\upp\alpha}{\vrn B\Gamma^2\tn{\rb_0}\opnorm{\X}},\frac{\zeta^2\vrn^2 B^2\opnorm{\X}^2}{\alpha^2},\frac{\zeta}{\Gamma}\right)\quad\Leftrightarrow \quad\eqref{ref bound eps},
\end{align*}
then \eqref{epsass} holds. As a result when \eqref{ref bound eps} and \eqref{k bound det} hold with $\delta=\min\left(\frac{\zeta\vrn B\opnorm{\X}}{\alpha},\frac{\upp}{\Gamma \tn{\rb_0}}\right)$ both assumptions of Theorem \ref{nn deter} also hold and so do its conclusions. In particular, \eqref{main res eq23} follows from \eqref{main res eq22} by noting that based on our choice of $\delta$ we have $\delta\frac{\Gamma}{\alpha}\twonorm{\rb_0}\le \frac{\upp}{\alpha}$, \eqref{row boundsss2} follows immediately from \eqref{row boundsss}, and \eqref{eqqq23} follows from \eqref{eqqq22} by noting that based on our choice of $\delta$ we have $\frac{\delta\alpha}{\vrn B\opnorm{\X}}\le \zeta$.
\end{proof}

\subsubsection{Main generalization result (completing the proof of Theorem \ref{nn deter gen})}\label{maingen}
Theorem \ref{nn deter gen} immediately follows from Theorem \ref{nn deter gennn} below by upper bounding $\dpp$ (see Definition \ref{earlyval}) using Lemma \ref{bndearlystopval} equation \eqref{upp up bound}.
\begin{theorem}\label{nn deter gennn}
Consider a training data set $\{(\x_i,\y_i)\}_{i=1}^n\in\R^d\times \R^K$ generated i.i.d.~according to a distribution $\Dc$ where the input samples have unit Euclidean norm. Also consider a neural net with $k$ hidden nodes as described in \eqref{neural net func} parameterized by $\W$ where the activation function $\phi$ obeys $\abs{\phi'(z)}, \abs{\phi''(z)}\le B$. Let $\W_0$ be the initial weight matrix with i.i.d.~$\Nn(0,1)$ entries. Also assume the output matrix has bounded entries obeying $\infnorm{\mtx{V}}\le \frac{\vrn}{\sqrt{kK}}$. Furthermore, set $\mtx{J}:=\mathcal{J}(\W_0)$ and define the information $\calF$ and nuisance $\calS$ subspaces and the truncated Jacobian $\mtx{J}_{\calF}$ associated with the reference/initial Jacobian $\mtx{J}$ based on a cut-off spectrum value $\alpha=\vrn B \bar{\alpha} \sqrt[4]{n}\sqrt{\opnorm{\X}}$. Also define the initial residual $\rb_0=f(\W_0)-\y\in\R^{nK}$ and pick $C_r>0$ so that $\frac{\tn{\rb_0}}{\sqrt{n}}\le C_r$. Also assume, the number of hidden nodes $k$ obeys
\begin{align}
 k\ge25600 \frac{C_r^2\Gamma^4}{\bar{\alpha}^8\vrn^2B^2\zeta^2},\label{k bound det gen}
\end{align}
with $\Gamma\ge 1$ and tolerance level $\zeta\le 2$.
Run gradient descent updates \eqref{grad dec me} with learning rate $\eta\le \frac{1}{\vrn^2B^2\opnorm{\X}^2}$. Then, after $T=\frac{\Gamma }{\eta\alpha^2}$ iterations, with probability at least $1-\delta$, the generalization error obeys
\begin{align}
\err{\W_{T}}&\leq \underbrace{\frac{2\tn{ \Pi_{\mathcal{N}}(\rb_0)}}{\sqrt{n}}}_{bias~term}+\underbrace{\frac{12\vrn B \dpp}{\sqrt{n}}}_{variance~term}+5\sqrt{\frac{\log(2/\delta)}{n}}+2C_r(e^{-\Gamma}+\zeta),
\end{align}
where $\dpp$ is the early stopping distance as in Def. \eqref{earlyval}.
\end{theorem}
\begin{proof}
First, note that using $\alpha\le \beta=\vrn B\opnorm{\X}$, $\frac{\upp}{\Gamma}\le \twonorm{\rb_0}$ per \eqref{upp up bound}, and $\tn{\rb_0}\leq  C_r\sqrt{n}$ we have
\begin{align*}
\frac{\alpha}{ \vrn B\|\X\|}\leq1\leq  \frac{\Gamma \tn{\rb_0}}{\upp}\leq C_r \frac{\Gamma\sqrt{n}}{\upp}. 
\end{align*}
This together with $\zeta\le 2$ implies that 
\begin{align}\label{temp521}
\frac{\alpha}{\frac{\zeta}{2}\vrn B\opnorm{\X}}+\frac{\Gamma \tn{\rb_0}}{\upp}\le&\frac{\alpha}{\frac{\zeta}{2}\vrn B\opnorm{\X}}+\frac{\Gamma \tn{\rb_0}}{\frac{\zeta}{2}\upp}\nn\\
\le&2\frac{\Gamma \tn{\rb_0}}{\frac{\zeta}{2}\upp}\nn\\
\le&2\frac{C_r}{\frac{\zeta}{2}}\frac{\Gamma}{\upp}\sqrt{n}.
\end{align}
Thus when
\begin{align*}
k\ge&25600 \frac{C_r^2\Gamma^4}{\bar{\alpha}^8\vrn^2B^2\zeta^2}\\\
=&6400 \left(\frac{C_r\Gamma}{\frac{\zeta}{2}}\right)^2\frac{n^2\Gamma^2\vrn^6B^6\opnorm{\X}^4}{\alpha^8}\\
\overset{\sqrt{n}\ge\|\X\|}{\ge}&6400 \left(\frac{C_r\Gamma}{\frac{\zeta}{2}}\right)^2\frac{n\vrn^6B^6\opnorm{\X}^6}{\alpha^8}\\
=& 1600 \left(2\frac{C_r}{\frac{\zeta}{2}}\frac{\Gamma}{\upp}\sqrt{n}\right)^2\frac{\upp^2\vrn^6B^6\opnorm{\X}^6}{\alpha^8}\\
\overset{\eqref{temp521}}{\ge}&1600\left(\frac{\alpha}{\frac{\zeta}{2}\vrn B\opnorm{\X}}+\frac{\Gamma \tn{\rb_0}}{\upp}\right)^2\frac{\upp^2\vrn^6B^6\opnorm{\X}^6}{\alpha^8}
\end{align*}
Thus, \eqref{k bound det} holds. Also \eqref{ref bound eps} trivially holds for $\eps_0$. Thus applying Theorem \ref{nn deter2} with $\eps_0=0$ the following three conclusions hold
\begin{align}
&\tf{\W_\tau-\W_0}\leq \frac{2\upp}{\alpha}=2\dpp.\label{frott51}
\end{align}
and
\begin{align}
\trow{\W_{\tau}-\W_0}\le& \frac{ 2\vrn B\Gamma\opnorm{\X}}{\sqrt{k} \alpha^2} \tn{\rb_0}\nn\\
\overset{\tn{\rb_0}\le C_r\sqrt{n}}{\le}&\frac{ 2\sqrt{n}C_r\vrn B\Gamma\opnorm{\X}}{\sqrt{k} \alpha^2}.
\label{rowtt521}
\end{align}
and 
\begin{align}
\label{tt521}
\twonorm{f(\W_T)-\y}\le&  e^{-\Gamma}\tn{\Pi_{{\Rc}}(\rb_0)}+\tn{ \Pi_{\mathcal{N}}(\rb_0)}+\frac{\zeta}{2}\tn{\rb_0}\nn\\
\le&  \tn{ \Pi_{\mathcal{N}}(\rb_0)}+\left(e^{-\Gamma}+\frac{\zeta}{2}\right)\tn{\rb_0}\nn\\
\overset{\tn{\rb_0}\le C_r\sqrt{n}}{\le}&  \tn{ \Pi_{\mathcal{N}}(\rb_0)}+C_r\left(e^{-\Gamma}+\frac{\zeta}{2}\right)\sqrt{n}.
\end{align}
Furthermore, using the assumption that $\infnorm{\Vb}\le \frac{\vrn}{\sqrt{kK}}$ Lemma \ref{generalize corollary} applies and hence equation \eqref{err pop bound} with $\W=\W_T$, $\sqrt{\mathcal{L}(\W_T)}=\frac{\twonorm{f(\W_T)-\y}}{\sqrt{n}}$, $M_{\mathcal{W}}=2\dpp$, $M_{\mathcal{V}}=0$, and $R=\frac{2\sqrt{n}C_r\vrn B\Gamma\opnorm{\X}}{\alpha^2}$ implies that
\begin{align}
\label{temperr}
\err{\W_T}\leq 2\left[\frac{\twonorm{f(\W_T)-\y}}{\sqrt{n}}+3\vrn B\left(\frac{2\dpp}{\sqrt{n}}+\frac{R^2}{\sqrt{k}}\right)+\sqrt{\frac{5\log(2/\delta)}{n}}\right]
\end{align}
Also note that using \eqref{k bound det gen} we have
\begin{align}
\label{R2k}
\frac{3\vrn B R^2}{\sqrt{k}}\leq \frac{ 12C_r^2\Gamma^2 \vrn^3 B^3n\|\X\|^2}{\sqrt{k}\alpha^4}\leq C_r\frac{\zeta}{2}
\end{align}
Plugging \eqref{tt521} and \eqref{R2k} into \eqref{temperr} completes the proof.
\end{proof}
\subsection{Proofs for neural network with random initialization (proof of Theorem \ref{gen main})}
In this section we prove Theorem \ref{gen main}. We first discuss and prove an optimization result in Section \ref{mainopt22}. Next, in Section \ref{comp22} we build upon this result to complete the proof of Theorem \ref{gen main}.

\subsubsection{Optimization result}
\label{mainopt22}
\begin{theorem} [Optimization guarantee for random initialization]\label{apply 1} Consider a training data set $\{(\x_i,\y_i)\}_{i=1}^n\in\R^d\times \R^K$ generated i.i.d.~according to a distribution $\Dc$ where the input samples have unit Euclidean norm and the concatenated label vector obeys $\tn{\y}=\sqrt{n}$ (e.g.~one-hot encoding). Consider a neural net with $k$ hidden layers as described in \eqref{neural net func} parameterized by $\W$ where the activation function $\phi$ obeys $\abs{\phi'(z)}, \abs{\phi''(z)}\le B$. Let $\W_0$ be the initial weight matrix with i.i.d.~$\Nn(0,1)$ entries. Fix a precision level $\zeta$ and set 
\begin{align}
\vrn=\frac{\zeta}{50B\sqrt{\log(2\K)}}.\label{simplified cond}
\end{align}
Also assume the output layer $\Vb$ has i.i.d.~Rademacher entries scaled by $\frac{\vrn}{\sqrt{k\K}}$. Furthermore, set $\mtx{J}:= \bSi(\X)^{1/2}$ and define the information $\calF$ and nuisance $\calS$ spaces and the truncated Jacobian $\mtx{J}_{\calF}$ associated with the reference Jacobian $\mtx{J}$ based on a cut-off spectrum value of $\alpha_0=\bar{\alpha}\sqrt[4]{n}\sqrt{K\opnorm{\X}}B\le B\sqrt{K}\opnorm{\X}$ per Definition \ref{tjac neural} so as to ensure $\twonorm{\Pi_{\calF}(\y)}\ge c\twonorm{\y}$ for some constant $c$. Assume\vspace{-6pt}
\begin{align}
\label{k condition}
k\ge 12\times 10^7\frac{\Gamma^4\K^4B^8\|\X\|^6n\log(n)}{c^4\zeta^4\alpha_0^8}
\end{align}
with $\Gamma\ge 1$ and $\zeta\le \frac{c}{2}$. We run gradient descent iterations of the form \eqref{grad dec me} with a learning rate $\eta\le \frac{1}{\vrn^2B^2\opnorm{\X}^2}$. Then, after $T=\frac{\Gamma K}{\eta\vrn^2\alpha_0^2}$ iterations, the following identities 
\begin{align}
&\tn{f(\W_{\tau_0})-\y}\leq  \tn{ \Pi_{\Rcb}(\y)}+ \e^{-\Gamma}\tn{\Pi_{{\Rc}}(\y)}+4\zeta\sqrt{n},\label{eqqq24}\\
&\tn{\W_\tau-\W_0}\leq \frac{2\sqrt{\K}(\upz(\y)+\Gamma\zeta \sqrt{n})}{\vrn\alpha_0},\label{main res eq24}\\
&\trow{\W_{\tau}-\W_0}\leq \frac{4\Gamma B\K \|\X\|}{\vrn\alpha_0^2}\sqrt{\frac{n}{k}},\label{row boundsss3}
\end{align}
hold with probability at least $1-(2\K)^{-100}$.
\end{theorem}
\begin{proof} To prove this result we wish to apply Theorem \ref{nn deter2}. To do this we need to verify the assumptions of this theorem. To start with, using Lemma \ref{upresz} with probability at least $1-(2\K)^{-100}$, the initial prediction vector $f(\W_0)$ obeys
\begin{align}
\tn{f(\W_0)}\leq \zeta\tn{\y}=\zeta\sqrt{n}\leq \frac{\sqrt{n}}{2}.\label{y diff diff}
\end{align}
Hence the initial residual obeys $\tn{\rb_0}\leq 2\sqrt{n}$. Furthermore, using $\zeta\leq c/2$
\begin{align}
\tn{\rb_0+\y}\le& \zeta \tn{\y}\implies \tn{\Pi_{\Rc}(\rb_0+\y)}\le \zeta \tn{\y}.
\end{align}
Thus,
\begin{align}
\tn{\Pi_{\Rc}(\rb_0)}\ge&\tn{\Pi_{\Rc}(\y)}-\tn{\Pi_{\Rc}(\rb_0+\y)}\nn\\
\ge&\tn{\Pi_{\Rc}(\y)}-\zeta\tn{\y}\nn\\
\ge&\left(c-\zeta\right)\tn{\y}\nn\\
\ge&\frac{c}{2}\tn{\y}\label{penult}\\
\ge&\frac{c}{4}\tn{\rb_0}.\label{lastbri}
\end{align}
Thus the assumption on the ratio of information to total energy of residual holds and we can replace $c$ with $\frac{c}{4}$ in Theorem \ref{nn deter2}. Furthermore, since $\upz(\cdot)$ is $\Gamma$-Lipschitz function of its input vector in $\ell_2$ norm hence we also have
\begin{align}
\upz(\rb_0)\leq \upz(\y) +\Gamma\tn{\rb_0+\y}\leq \upz(\y) +\Gamma \zeta\tn{\y}.\label{upp lip}
\end{align}
Next we wish to show that \eqref{ref bound eps} holds. In particular we will show that there exists an $\eps_0$-reference Jacobian $\Jb$ for $\Jc(\W_0)$ satisfying $\Jb\Jb^T=\E[\Jc(\W_0)\Jc(\W_0)^T]$. Note that, such a $\Jb$ will have exactly same information/nuisance spaces as the square-root of the multiclass kernel matrix i.e.~$\left(\E[\Jc(\W_0)\Jc(\W_0)^T]\right)^{\frac{1}{2}}$ since these subspaces are governed by the left eigenvectors. Applying Lemmas \ref{minspectJ} (with a scaling of the Jacobian by $1/\sqrt{kK}$ due to the different scaling of $\V$), we find that if 
\begin{align}
k\geq \frac{1000\K^2B^4\|\X\|^4\log(n)}{\delta^2}\label{kk bound}
\end{align}
then,
\begin{align}
\opnorm{\Jc(\W_0)\Jc(\W_0)^T-\E[\Jc(\W_0)\Jc(\W_0)^T]}\leq \frac{\delta\vrn^2}{K}.\label{ref jac}
\end{align}
Let $\Jcb(\W)$ be obtained by adding $\max(\K n-p,0)$ zero columns to $\Jc(\W)$. Then, using \eqref{ref jac} and Lemma \ref{psd pert}, there exists $\Jb$ satisfying $\Jb\Jb^T=\E[\Jc(\W_0)\Jc(\W_0)^T]$ and\vspace{-5pt}
\begin{align*}
\opnorm{\Jcb(\W_0)-\Jb}\le 2\sqrt{\frac{\delta\vrn^2}{K}}.
\end{align*}
Therefore, $\Jb$ is an $\eps_0^2=4\frac{\delta\vrn^2}{K}$ reference Jacobian. Now set
\[
\Theta=\min\left(c\frac{\upz\alpha_0}{B\Gamma^2\opnorm{\X}\sqrt{nK}},\left(\frac{\zeta B\sqrt{K}\opnorm{\X}}{\alpha_0}\right)^2,\frac{\zeta}{\Gamma}\right)
\]
and note that using $\alpha=\frac{\vrn}{\sqrt{K}}\alpha_0$ and $\tn{\rb_0}\le 2\sqrt{n}$
\begin{align}
\label{tt2}
\Theta=&\min\left(c\frac{\upz\alpha_0}{B\Gamma^2\opnorm{\X}\sqrt{nK}},\left(\frac{\zeta B\sqrt{K}\opnorm{\X}}{\alpha_0}\right)^2,\frac{\zeta}{\Gamma}\right)\nn\\
=&\min\left(c\frac{\upz\alpha}{\vrn B\Gamma^2\opnorm{\X}\sqrt{n}},\left(\frac{\vrn\zeta B\opnorm{\X}}{\alpha}\right)^2,\frac{\zeta}{\Gamma}\right)\nn\\
\le&2\cdot\min\left(c\frac{\upp\alpha}{\vrn B\Gamma^2\tn{\rb_0}\opnorm{\X}},\frac{\zeta^2\vrn^2 B^2\opnorm{\X}^2}{\alpha^2},\frac{\zeta}{\Gamma}\right)
\end{align}
To continue further, note that $\upz$ calculated with respect to $ \bSi(\X)^{1/2}$ with cutoff $\alpha_0$ is exactly same as $\upp$ calculated with respect to $\Jb$ with cutoff $\alpha=\frac{\vrn \alpha_0}{\sqrt{\K}}$ which is a square-root of $\E[\Jc(\W_0)\Jc(\W_0)^T]$. Thus, using \eqref{tt2} to ensure \eqref{ref bound eps} holds it suffices to show
\[
\eps_0^2=4\frac{\delta\vrn^2}{\K}\leq  \frac{\alpha^2}{25}\frac{\Theta}{2}=\frac{\vrn^2\alpha_0^2}{50\K}\Theta.
\]
Hence, to ensure \eqref{ref bound eps} holds we need to ensure that $\delta$ obeys
\[
\delta\leq \frac{\alpha_0^2}{200}\Theta.
\]
Thus using $\delta=\frac{\alpha_0^2}{200}\Theta$ to ensure \eqref{ref bound eps} we need to make sure $k$ is sufficiently large so that \eqref{kk bound} holds with this value of $\delta$. Thus it suffices to have
\begin{align}
k\ge&12\times 10^7\frac{\Gamma^4\K^4B^8\|\X\|^6n\log(n)}{c^4\zeta^4\alpha_0^8}\label{k1}\\
\ge&12\times 10^7\frac{\Gamma^4\K^4B^8\|\X\|^8\log(n)}{c^4\zeta^4\alpha_0^8}\nn\\
\ge&4\times 10^7\cdot\left(\frac{4K^2B^4\Gamma^4\opnorm{\X}^4}{c^4\alpha_0^4}+\frac{1}{\zeta^4 }+\frac{\Gamma^2}{\zeta^2}\right)\frac{\K^2B^4\|\X\|^4\log(n)}{\alpha_0^4}\nn\\
\ge&4\times 10^7\cdot\left(\frac{4KB^4\Gamma^4\opnorm{\X}^4}{c^4\alpha_0^4}+\frac{1}{\zeta^4 }+\frac{\Gamma^2}{\zeta^2}\right)\frac{\K^2B^4\|\X\|^4\log(n)}{\alpha_0^4}\nn\\
\ge&4\times 10^7\cdot\left(\frac{4KB^4\Gamma^4\opnorm{\X}^4}{c^4\alpha_0^4}+\frac{\alpha_0^4}{\zeta^4 B^4K^2\opnorm{\X}^4}+\frac{\Gamma^2}{\zeta^2}\right)\frac{\K^2B^4\|\X\|^4\log(n)}{\alpha_0^4}\nn\\
\overset{(a)}{\ge}&4\times 10^7\cdot\left(\frac{nKB^2\Gamma^4\opnorm{\X}^2}{c^2\upz^2\alpha_0^2}+\frac{\alpha_0^4}{\zeta^4 B^4K^2\opnorm{\X}^4}+\frac{\Gamma^2}{\zeta^2}\right)\frac{\K^2B^4\|\X\|^4\log(n)}{\alpha_0^4}\nn\\
\ge&4\times 10^7\cdot\max\left(\frac{nKB^2\Gamma^4\opnorm{\X}^2}{c^2\upz^2\alpha_0^2},\frac{\alpha_0^4}{\zeta^4 B^4K^2\opnorm{\X}^4},\frac{\Gamma^2}{\zeta^2}\right)\frac{\K^2B^4\|\X\|^4\log(n)}{\alpha_0^4}\nn\\
=&4\times 10^7\frac{\K^2B^4\|\X\|^4\log(n)}{\alpha_0^4\cdot\min\left(\left(\frac{c\upz\alpha_0}{B\Gamma^2\opnorm{\X}\sqrt{nK}}\right)^2,\left(\frac{\zeta B\sqrt{K}\opnorm{\X}}{\alpha_0}\right)^4,\frac{\zeta^2}{\Gamma^2}\right)}\nn\\
=&\frac{1000\K^2B^4\|\X\|^4\log(n)}{\delta^2}
\label{kreq11}
\end{align}
Here, (a) follows from the fact that $\| \bSi(\X)^{1/2}\|:=\la_1\leq B\|\X\|$,  equation \eqref{low bound}, and $\twonorm{\Pi_{\mathcal{I}}(\vct{r}_0)}\ge \frac{c}{2}\twonorm{\y}=\frac{c}{2}\sqrt{n}$ which combined imply
\begin{align}
\label{bri}
\upz\ge \frac{\alpha}{\lambda_1}\twonorm{\Pi_{\mathcal{I}}(\rb_0)}\ge \frac{\alpha}{B\opnorm{\X}}\twonorm{\Pi_{\mathcal{I}}(\rb_0)}\geq \frac{\alpha_0c}{2}\frac{\twonorm{\y}}{B\|\X\|}=\frac{\alpha_0c}{2}\frac{\sqrt{n}}{B\|\X\|}.
\end{align}
To be able to apply Theorem \ref{nn deter2} we must also ensure \eqref{k bound det} holds. Therefore, it suffices to have
\begin{align}
k\ge&64\times10^6\frac{K^4B^8\opnorm{\X}^6\Gamma^4n\log(n)}{\zeta^4\alpha_0^8}\\\overset{(a)}{\ge}&25600\frac{K^4B^6\opnorm{\X}^6\Gamma^4n}{\zeta^2\vrn^2\alpha_0^8}\label{k2}\\
k\ge&12800\left(\frac{\alpha_0^2}{\zeta^2B^2K\opnorm{\X}^2}+1\right)\frac{K^4B^6\opnorm{\X}^6\Gamma^4n}{\vrn^2\alpha_0^8}\nn\\
\overset{(b)}{\ge}&3200\left(\frac{\alpha_0^2}{\zeta^2B^2K\opnorm{\X}^2}+\frac{4\Gamma^2 n}{\upp^2}\right)\frac{K^4B^6\opnorm{\X}^6\Gamma^2\upp^2}{\vrn^2\alpha_0^8}\nn\\
\ge&3200\left(\frac{\alpha_0^2}{\zeta^2B^2K\opnorm{\X}^2}+\frac{\Gamma^2 \tn{\rb_0}^2}{\upp^2}\right)\frac{K^4B^6\opnorm{\X}^6\Gamma^2\upp^2}{\vrn^2\alpha_0^8}\nn\\
\ge&1600\left(\frac{\alpha_0}{\zeta\sqrt{K} B\opnorm{\X}}+\frac{\Gamma \tn{\rb_0}}{\upp}\right)^2\frac{K^4B^6\opnorm{\X}^6\Gamma^2\upp^2}{\vrn^2\alpha_0^8}\nn\\
=&1600\left(\frac{\alpha}{\zeta\vrn B\opnorm{\X}}+\frac{\Gamma \tn{\rb_0}}{\upp}\right)^2\frac{\vrn^6B^6\opnorm{\X}^6\Gamma^2\upp^2}{\alpha^8}
\end{align}
Here, (a) follows from the fact that $n\ge K$ and the relationship between $\zeta$ and $\vrn$ per \eqref{simplified cond} and (b) follows from the fact that per equation \eqref{upp up bound} we have
\begin{align*}
\upp\le \Gamma\twonorm{\rb_0}\le 2\Gamma\sqrt{n}
\end{align*}
Note that \eqref{k1} and \eqref{k2} are implied by
\begin{align}
k\ge12\times 10^7\frac{\Gamma^4\K^4B^8\|\X\|^6n\log(n)}{c^4\zeta^4\alpha_0^8},\label{kreq2}
\end{align}
which is the same as \eqref{k condition}. What remains is stating the optimization bounds in terms of the labels $\y$. This follows by substituting \eqref{y diff diff}, \eqref{upp lip}, and the fact that $\twonorm{\rb_0}\le 2\sqrt{n}$ into \eqref{eqqq23}, \eqref{main res eq23}, and \eqref{row boundsss2}, respectively.
\end{proof}

\subsubsection{Generalization result (completing the proof of Theorem \ref{gen main})}
\label{comp22}
Theorem below is a restatement of Theorem \ref{gen main} after substituting the upper bound on the early stopping distance $\dpz$ of Def. \eqref{earlyval}.
\begin{theorem} [Neural Net -- Generalization] \label{gen mainnn}
Consider a training data set $\{(\x_i,\y_i)\}_{i=1}^n\in\R^d\times \R^K$ generated i.i.d.~according to a distribution $\Dc$ where the input samples have unit Euclidean norm and the concatenated label vector obeys $\tn{\y}=\sqrt{n}$ (e.g.~one-hot encoding). Consider a neural net with $k$ hidden nodes as described in \eqref{neural net func} parameterized by $\W$ where the activation function $\phi$ obeys $\abs{\phi'(z)}, \abs{\phi''(z)}\le B$. Let $\W_0$ be the initial weight matrix with i.i.d.~$\Nn(0,1)$ entries. Fix a precision level $\zeta\le \frac{c}{2}$ and set 
\begin{align}
\vrn=\frac{\zeta}{50B\sqrt{\log(2\K)}}.\label{simplified cond}
\end{align}
Also assume the output layer $\Vb$ has i.i.d.~Rademacher entries scaled by $\frac{\vrn}{\sqrt{k\K}}$. Furthermore, set $\mtx{J}:= \bSi(\X)^{1/2}$ and define the information $\calF$ and nuisance $\calS$ spaces and the truncated Jacobian $\mtx{J}_{\calF}$ associated with the Jacobian $\mtx{J}$ based on a cut-off spectrum value of $\alpha_0=\bar{\alpha}\sqrt[4]{n}\sqrt{K\opnorm{\X}}B \leq B\|\X\|$ per Definition \ref{tjac neural} chosen to ensure $\tn{\Pi_{\Ic}(\y)}\geq c\tn{\y}$ for some constant $c>0$. Assume\vspace{-6pt}
\begin{align}
\label{hidnum}
k\ge 12\times 10^7\frac{\Gamma^4\K^4B^8\|\X\|^4n^2\log(n)}{c^4\zeta^4\alpha_0^8}
\end{align}
with $\Gamma\ge 1$. We run gradient descent iterations of the form \eqref{grad dec me} with a learning rate $\eta\le \frac{1}{\vrn^2B^2\opnorm{\X}^2}$. Then, after $T=\frac{\Gamma K}{\eta\vrn^2\alpha_0^2}$ iterations, classification error $\err{\W_{T}}$ is upper bounded by \vspace{-6pt}
\[
\err{\W_T}\le 2\frac{\tn{ \Pi_{\Rcb}(\y)}+e^{-\Gamma}\twonorm{\Pi_{\mathcal{I}}(\y)}}{\sqrt{n}}+\frac{12B\sqrt{\K}}{\sqrt{n}}\dpz+12\Big(1+\frac{\Gamma}{\bar{\alpha}\sqrt[4]{n\|\X\|^2}}\Big)\zeta+10\sqrt{\frac{\log(2/\delta)}{n}},
\]
holds with probability at least $1-(2\K)^{-100}-\delta$.
\end{theorem}
\begin{proof}
Under the stated assumptions, Theorem \ref{apply 1} holds with probability $1-(2\K)^{-100}$. The proof will condition on outcomes of the Theorem \ref{apply 1}. Specifically, we shall apply \eqref{err pop bound} of Lemma \ref{generalize corollary} with $M_{\mathcal{W}}$ and $R$ dictated by Theorem \ref{apply 1} where the output layer $\Vb$ is fixed. Observe that $\tf{\Vb}=\sqrt{\K k}\tin{\Vb}=\vrn$, we have
\begin{align}
\label{tempkillme}
\err{\W_T}\leq 2\left[\frac{\twonorm{f(\W_T)-\y}}{\sqrt{n}}+3\vrn B\left(\frac{M_{\mathcal{W}}}{\sqrt{n}}+\frac{R^2}{\sqrt{k}}\right)+\sqrt{\frac{5\log(2/\delta)}{n}}\right].
\end{align}
Theorem \ref{apply 1} yields
\begin{align}
\label{killme1}
\frac{\tn{f(\W_T)-\y}}{\sqrt{n}}\leq \frac{\tn{ \Pi_{\Rcb}(\y)}+ \e^{-\Gamma}\tn{\Pi_{{\Rc}}(\y)}}{\sqrt{n}}+4\zeta.
\end{align}
Using \eqref{main res eq24} for $M_{\mathcal{W}}$
\begin{align}
\label{killme2}
\frac{{\vrn} BM_{\mathcal{W}}}{\sqrt{n}}\leq \frac{2B\sqrt{\K} \dpz(\y)}{\sqrt{n}}+\frac{2B\sqrt{\K}\Gamma\zeta}{\alpha_0}.
\end{align}
Using \eqref{row boundsss3} on row bound $R$ and lower bound on $k$
\begin{align}
\label{killme3}
3{\vrn} B\frac{R^2}{\sqrt{k}}=&\frac{48n\Gamma^2B^3K^2\opnorm{\X}^2}{\vrn\alpha_0^4\sqrt{k}}\nn\\
\le&\frac{c^2\zeta^2}{230\vrn B\log(n)}\nn\\
\le& \zeta.
\end{align}
Plugging in \eqref{killme1}, \eqref{killme2}, and \eqref{killme3} into \eqref{tempkillme} concludes the proof.
\end{proof}

\section*{Acknowledgements}
M. Soltanolkotabi is supported by the Packard Fellowship in Science and Engineering, a Sloan Research Fellowship in Mathematics, an NSF-CAREER under award \#1846369, the Air Force Office of Scientific Research Young Investigator Program (AFOSR-YIP)
under award \#FA9550-18-1-0078, an NSF-CIF award \#1813877, and a Google faculty research award. 
{
\bibliographystyle{acm}
\bibliography{Bibfiles}
}
\appendix
\section{The Jacobian of the Mixture Model is low-rank\\(Proofs for Section \ref{mix sec})}
\label{Jacclust}
The following theorem considers a simple noiseless mixture model and proves that its Jacobian is low-rank and the concatenated multiclass label vectors lie on a rank $K^2C$ information space associated with this Jacobian.

\begin{theorem} \label{mix mod} Consider a data set of size $n$ consisting of input/label pairs $\{(\vct{x}_i,\y_i)\}_{i=1}^n\in\R^d\times \R^K$ generated according to the Gaussian mixture model of Definition \ref{GMM} with $K$ classes each consisting of $C$ clusters with the cluster centers given by $\{\vct{\mu}_{\ell,\widetilde{\ell}}\}_{(\ell,\widetilde{\ell})=(1,1)}^{(K,C)}$ and $\sigma=0$. Let $\bSi(\X)$ be the multiclass neural tangent kernel matrix associated with input matrix $\X=[\x_1~\dots~\x_n]^T$ with the standard deviation of the output layer set to $\nu=\frac{1}{\sqrt{k}}$. Also define the information space $\Rc$ to be the range space of $\bSi(\X)$. Also let $\M=[\bmu_{1,1}~\dots~\bmu_{\K,C}]^T$ be the matrix obtained by aggregating all the cluster centers as rows and let $\vct{g}$ be a Gaussian random vector with distribution $\mathcal{N}(\vct{0},\mtx{I}_d)$. Define the neural tangent kernel matrix associated with the cluster centers as
\[
\widetilde{\bSi}(\M)=(\M\M^T)\odot \E_{\g\sim\mathcal{N}\left(\vct{0},\mtx{I}_d\right)}[{\phi'(\M\g)}{\phi'(\M\g)^T}]\in\R^{\K C\times \K C},
\] 
and assume that $\widetilde{\bSi}(\M)$ is full rank. Then, the following properties hold with probability $1-KC\exp(-\frac{n}{8KC})$
\begin{itemize}
\item $\Rc$ is a $K^2C$ dimensional subspace.
\item The concatenated label vector $\y=\begin{bmatrix}\y_1^T & \y_2^T & \ldots & \y_n^T \end{bmatrix}^T$  lies on $\Rc$.
\item  The nonzero eigenvalues (top $K^2C$ eigenvalues) of $\bSi(\X)$ are between $\frac{n}{2KC}s_{\min}(\bSi(\X))$ and $\frac{2n}{KC}\|\bSi(\X)\|$. Hence the eigenvalues of the information space grow with $\frac{n}{\K C}$.
\end{itemize}
\end{theorem}
\begin{proof} First, we establish that each cluster has around the same size. Applying Chernoff bound and a union bound, we find that with probability $1-KC\exp(-\frac{n}{8KC})$
\[
0.5\tilde{n}\leq \widetilde{n}_{\ell,\tilde{\ell}}\leq 2\tilde{n}.
\]
Note that based on Lemma \ref{multi cov}, the multiclass covariance is given by 
\[
\bSi(\X)=k\nu^2 \Iden_{\K}\otimes\widetilde{\bSi}(\X).
\]
where $\widetilde{\bSi}(\X)=(\X\X^T)\odot \E_{\g\distas \Nn(0,1)}[{\phi'(\X\g)}{\phi'(\X\g)^T}]$. Due to this Kronecker product representation, the range space of $\bSi(\X)$ is separable. In particular, note that with
\[
\widetilde{\Ic}=\text{Range}\left(\widetilde{\bSi}(\X)\right)
\]
we have $\Ic=\Iden_{\K}\otimes\widetilde{\Ic}$ which also implies $\text{rank}(\Ic)=\K\cdot\text{rank}\left(\widetilde{\Ic}\right)$. Hence, this identity allows us to reduce the problem to a single output network. To complete the proof we will prove the following three identities:
 \begin{itemize}
\item $\widetilde{\Ic}$ has rank $\K C$.
\item The nonzero eigenvalues of $\widetilde{\bSi}\left(\X\right)$ are between $0.5\widetilde{n}s_{\min}(\widetilde{\bSi}(\M))$ to $2\widetilde{n}\|\widetilde{\bSi}(\M)\|$.
\item The portion of the label vector associated with class $\ell$ i.e.~$\y^{(\ell)}\in\R^{n}$ (see \eqref{model}) lies on $\mathcal{\Ic}$. Hence, the concatenated vector $\y$ lies on $\Ic=\Iden_{\K}\otimes \widetilde{\Ic}$.
\end{itemize}
To prove these statements let $\Jc_\ell(\X;\W_0)$ and $\Jc_\ell(\M;\W_0)$ be the Jacobian associated with the $\ell$th output of the neural net (see \eqref{model}) for data matrices $\X$ and $\M$. Observe that the columns of $\Jc_\ell(\X;\W_0)$ are chosen from $\Jc_\ell(\M;\W_0)$ and in particular each column of $\Jc_\ell(\M;\W_0)$ is repeated between $0.5\widetilde{n}$ to $2\widetilde{n}$ times. To mathematically relate this, define the $KC$ dimensional subspace $\mathcal{S}$ of $\R^{n}$ where for any $\vb\in \mathcal{S}$, entries $v_i$ and $v_j$ of $\vb$ are equal iff data point $\x_i$ and $\x_j$ are equal (i.e.~belong to the same class/cluster pair). Now, we define the orthonormal matrix $\Ub_{\mathcal{S}}\in\R^{n\times \K C}$ as the 0-1 matrix with orthogonal rows that map $\R^{\K C}$ to $\mathcal{S}$ as follows.  Assume the $i$th data point $\x_i$ belongs to the class/cluster pair $(\ell_i,\widetilde{\ell}_i)$. We then set the $i$th row of $\Ub_{\mathcal{S}}$ as vect$\left(\vct{e}_{\ell_i}\vct{e}_{\widetilde{\ell}_i}^T\right)$. Using $\Ub_{\mathcal{S}}$ we have
\[
\Ub_{\mathcal{S}}\Jc_\ell(\M;\W_0)=\Jc_\ell(\X;\W_0).
\]
Now note that using the above identity we have
\[
{\Ub}_{\mathcal{S}}\widetilde{\bSi}(\M){\Ub}_{\mathcal{S}}^T=\widetilde{\bSi}(\X).
\]
Since ${\Ub}_{\mathcal{S}}$ is tall and orthogonal, the range of $\widetilde{\bSi}(\X)$ is exactly the range of $\Ub_{\mathcal{S}}$ hence $\widetilde{\Ic}=\mathcal{S}$ which is $\K C$ dimensional. Furthermore, nonzero eigenvectors of $\widetilde{\bSi}(\X)$ lie on $\Sc$ and any eigenvector $\vb$ satisfies
\[
\vb^T\widetilde{\bSi}(\X)\vb\geq s_{\min}({\Ub}_{\mathcal{S}})^2s_{\min}(\widetilde{\bSi}(\M))\geq 0.5\bar{n} s_{\min}(\widetilde{\bSi}(\M))
\]
and similarly
\[
\vb^T\widetilde{\bSi}(\X)\vb\leq  2\bar{n} \|\widetilde{\bSi}(\M)\|
\]
which follows from the fact that $\ell_2$-norm-squared of columns of $\Ub$ are between $0.5\bar{n}$ to $2\bar{n}$.
 Finally, we will argue that label vector $\y^{(\ell)}$ lies on $\mathcal{S}$. Note that for all samples $i$ that belong to the same cluster $\y^{(\ell)}_i$ will be the same (either zero or one), thus $\y^{(\ell)}\in\mathcal{S}$.
\end{proof}
Next lemma provides a perturbation analysis when there is noise.
\begin{lemma}\label{perturb}
Consider the single-output NTK kernel given by
\begin{align*}
\widetilde{\mtx{\Sigma}}(\X)=\E\Big[\phi'(\X\w)\phi'(\X\w)^T\Big]\odot\left(\X\X^T\right),
\end{align*}
and assume that this matrix has rank $r$ so that $\lambda_{r+1}\left(\widetilde{\mtx{\Sigma}}(\X)\right)=\lambda_{r+2}\left(\widetilde{\mtx{\Sigma}}(\X)\right)=\ldots=\lambda_{n}\left(\widetilde{\mtx{\Sigma}}(\X)\right)=0$. Also assume a noise corrupted version of $\X$ given by
\begin{align*}
\widetilde{\X}=\X+\frac{\sigma}{\sqrt{d}}\mtx{Z}
\end{align*}
with $\mtx{Z}$ a matrix consisting of i.i.d.~$\mathcal{N}(0,1)$ entries. Then, $\opnorm{\widetilde{\mtx{\Sigma}}(\widetilde{\X})-\widetilde{\mtx{\Sigma}}(\X)}\lesssim \Delta$ where
\begin{align}
\Delta&:=\sigma^2B^2\log n\opnorm{\X}^2 +\sigma^2B^2 (n/d+1)+\sqrt{\log n}\cdot\sigma B^2\opnorm{\X}^2+\sigma B^2\sqrt{n/d+1}\opnorm{\X}
\end{align}
holds with probability at least $1-2n\e^{-\frac{d}{2}}$. Whenever $\sigma\leq \frac{1}{\sqrt{\log n}}$, $\Delta$ is upper bounded as
\begin{align}
\frac{\Delta}{n}&\lesssim  B^2\sigma\sqrt{\log n} \label{sigma bound}.
\end{align}
 Furthermore, let $\widetilde{\mtx{V}}, \mtx{V}\in\R^{n\times r}$ be orthonormal matrices corresponding to the top $r$ eigenvalues of $\widetilde{\mtx{\Sigma}}(\widetilde{\X})$ and $\widetilde{\mtx{\Sigma}}(\X)$. Then,
\begin{align*}
\opnorm{\widetilde{\V}\widetilde{\V}^T-\V\V^T}\le \frac{\Delta}{\lambda_{r}(\widetilde{\mtx{\Sigma}}(\X))-\Delta}
\end{align*}
\end{lemma}
\begin{proof}
Note that
\begin{align*}
\text{diag}\left(\phi'\left(\widetilde{\X}\w\right)\right)\widetilde{\X}-\text{diag}\left(\phi'\left(\X\w\right)\right)\X=&\text{diag}\left(\phi'\left(\widetilde{\X}\w\right)\right)\widetilde{\X}-\text{diag}\left(\phi'\left(\X\w\right)\right)\X\\
=&\text{diag}\left(\phi'\left(\widetilde{\X}\w\right)-\phi'\left(\X\w\right)\right)\X\\
&+\text{diag}\left(\phi'\left(\widetilde{\X}\w\right)\right)(\widetilde{\X}-\X)
\end{align*}
Now define $\widetilde{\mtx{M}}=\text{diag}\left(\phi'(\widetilde{\X}\w)\right)\widetilde{\X}$ and $\mtx{M}=\text{diag}\left(\phi'(\X\w)\right)\X$ and note that using the above we can conclude that
\begin{align*}
\opnorm{\widetilde{\mtx{M}}-\mtx{M}}\le&\opnorm{\text{diag}\left(\phi'\left(\widetilde{\X}\w\right)-\phi'\left(\X\w\right)\right)\X}\\
&+\opnorm{\text{diag}\left(\phi'\left(\widetilde{\X}\w\right)\right)(\widetilde{\X}-\X)}\\
\le&B\infnorm{(\widetilde{\X}-\X)\w}\opnorm{\X}+B\opnorm{\widetilde{\X}-\X}
\end{align*}
Now using the fact that 
\begin{align*}
\opnorm{\widetilde{\mtx{M}}\widetilde{\mtx{M}}^T-\mtx{M}\mtx{M}^T}\le \opnorm{\widetilde{\mtx{M}}-\mtx{M}}^2+2\opnorm{\widetilde{\mtx{M}}-\mtx{M}}\opnorm{\mtx{M}},
\end{align*}
we conclude that
\begin{align*}
\opnorm{\widetilde{\mtx{\Sigma}}(\widetilde{\X})-\widetilde{\mtx{\Sigma}}(\X)}=&\opnorm{\E\Big[\widetilde{\mtx{M}}\widetilde{\mtx{M}}^T-\mtx{M}\mtx{M}^T\Big]}\\
\le&\E\Big[\left(B\infnorm{(\widetilde{\X}-\X)\w}\opnorm{\X}+B\opnorm{\widetilde{\X}-\X}\right)^2\Big]\\
&+2B\opnorm{\X}\left(B\opnorm{\X}\E\big[\infnorm{(\widetilde{\X}-\X)\w}\big]+B\opnorm{\widetilde{\X}-\X}\right)\\
\le&2B^2\opnorm{\X}^2\E\big[\infnorm{(\widetilde{\X}-\X)\w}^2\big]+2B^2\opnorm{\widetilde{\X}-\X}^2\\
&+2B^2\opnorm{\X}^2\E\big[\infnorm{(\widetilde{\X}-\X)\w}\big]+2B^2\opnorm{\widetilde{\X}-\X}\opnorm{\X}
\end{align*}
To proceed further, with probability $1-n\exp(-d/2)$, each row of $\widetilde{\X}-\X$ is upper bounded by $2\sigma$. Hence, using a standard tail bound over supremum of $n$ Gaussian random variables (which follows by union bounding) we have 
\[
\E[\infnorm{(\widetilde{\X}-\X)\vct{w}}^2]^{1/2}\le 2\sigma\sqrt{2\log n}
\] holds with the same probability. Furthermore, spectral norm bound on Gaussian random matrix implies that
\begin{align*}
\opnorm{\widetilde{\X}-\X}^2\le \left(2(\sqrt{n}+\sqrt{d})\right)^2\frac{\sigma^2}{d}\le 8(n/d+1)\sigma^2,
\end{align*}
holds with probability at least $1-\e^{-\frac{1}{2}(n+d)}$. Plugging these two probabilistic bounds into the chain of inequalities we conclude that
\begin{align*}
\opnorm{\widetilde{\mtx{\Sigma}}(\widetilde{\X})-\widetilde{\mtx{\Sigma}}(\X)}\lesssim \sigma^2B^2\log n\opnorm{\X}^2 +\sigma^2B^2 (n/d+1)+\sqrt{\log n}\cdot\sigma B^2\opnorm{\X}^2+\sigma B^2\sqrt{n/d+1}\opnorm{\X}
\end{align*}
To establish \eqref{sigma bound}, observe that $B^2\sigma\sqrt{\log n}\|\X\|^2$ dominates over other terms in the regime $\sigma \sqrt{\log n}$ is small. The final bound is a standard application of Davis-Kahan Theorem \cite{yu2014useful} when we use the fact that $\tilde{\bSi}(\X)$ is low-rank.
\end{proof}

The following lemma plugs in the critical quantities of Theorem \ref{gen main} for our mixture model to obtain a generalization bound.
\begin{theorem} [Generalization for Mixture Model]\label{app thm} Consider a dataset $\{\x_i,\y_i\}_{i=1}^n$ generated i.i.d.~from the Gaussian mixture model in Definition \ref{GMM}. Let $\la_{\mtx{M}}=\lambda_{\min}(\bSi(\M))$ where $\M\in\R^{KC\times d}$ is the matrix of cluster centers. Suppose input noise level $\sigma$ obeys
\[
\sigma\lesssim\frac{\la_{\min}}{B^2KC\sqrt{\log n}}
\]
Consider the setup of Theorem \ref{gen main} with quantities $\zeta$ and $\Gamma$. Suppose network width obeys
\[
k\gtrsim \frac{\Gamma^4B^8K^8C^4\log n}{\zeta^4\la_{\min}^4}.
\]
With probability $1-n\e^{-d/2}-KC\exp(-\frac{n}{8KC})-(2K)^{-100}-\delta$, running gradient descent for $T=\frac{2\Gamma K^2C}{\eta\vrn^2n\la_{\min}}$ with learning rate $\eta\leq \frac{1}{\nu^2B^2\|\X\|^2}$, we have that
\[
\err{\W_{T}}\lesssim \sqrt{ \frac{\sigma \sqrt{\log n}B^2KC}{\la_{\min}}}+\frac{\Gamma BK\sqrt{C}}{\sqrt{{n\la_{\min}}}}+12\zeta+5\sqrt{\frac{\log(2/\delta)}{n}}+2\e^{-\Gamma}.
\]
\end{theorem}
\begin{proof} The proof is an application of Lemma \ref{perturb} and Theorem \ref{mix mod}. Let $\Rc'$ be the information space corresponding to noiseless dataset where input samples are identical to cluster centers. Let $\Pb',\Pb$ correspond to the projection matrices to $\Rc$ and $\Rc'$. First, using Lemma \ref{perturb} and the bound on $\sigma$, we have
\[
\|\Pb'-\Pb\|\leq c \frac{\sigma \sqrt{\log n}B^2KC}{\la_{\min}}
\]
for some constant $c>0$. Next we quantify $\Pi_{\Rc}(\y)$ using the fact that (i) $\Pi_{\Rc'}(\y)=\y$ via Theorem \ref{mix mod} as follows
\begin{align}
\tn{\Pi_{\Rc}(\y)}\geq\tn{\Pi_{\Rc'}(\y)}-\tn{\Pi_{\Rc}(\y)-\Pi_{\Rc'}(\y)}\geq \sqrt{n}(1- c\frac{\sigma \sqrt{\log n}B^2KC}{\la_{\min}}).
\end{align}
In return, this implies that 
\[
\tn{\Pi_{\Rcb}(\y)}\lesssim \sqrt{ \frac{n\sigma \sqrt{\log n}B^2KC}{\la_{\min}}}
\]
To proceed, we pick $\alpha_0=\sqrt{\frac{\la_{\min}n}{2KC}}$ and corresponding $\bar{\alpha}=\frac{\alpha_0}{\sqrt[4]{n}\sqrt{K\opnorm{\X}}B}\geq \sqrt{\frac{\la_{\min}}{2B^2K^2C}}$ and apply \eqref{simple bound} to find that, classification error is upper bounded by
\[
\err{\W_{T}}\lesssim \sqrt{ \frac{\sigma \sqrt{\log n}B^2KC}{\la_{\min}}}+\frac{\Gamma BK\sqrt{C}}{\sqrt{{n\la_{\min}}}}+12\zeta+5\sqrt{\frac{\log(2/\delta)}{n}}+2\e^{-\Gamma}.
\]
\end{proof}

\section{Joint input-output optimization}\label{sec combined jacob}
In this section we wish to provide the ingredients necessary to prove a result for the case where both set of input and output weights $\W$ and $\V$ are trained. To this aim, we consider the combined neural net Jacobian associated with input and output layers given by
\begin{align}
\vct{x}\mapsto f(\vct{x};\Vb,\mtx{W}):=\Vb\phi(\mtx{W}\vct{x}).\label{combined func}
\end{align}
Denoting the Jacobian associated with \eqref{combined func} by $\Jc(\Vb,\W)$ we have that
\[
\Jc(\Vb,\W)=[\Jc(\Vb)~\Jc(\W)]\in \R^{\K n\times k(\K+d)}
\]
Here, $\Jc(\W)$ is as before whereas $\Jc(\Vb)$ is the Jacobian with respect to $\Vb$ and is given by
\begin{align}\label{output J}
\Jc(\Vb)=\begin{bmatrix}\Jc(\vb_1)~\Jc(\vb_2)~\dots~\Jc(\vb_\K)\end{bmatrix}.
\end{align}
where $\Jc(\vb_\ell)\in\R^{\K n\times k}$ is so that its $\ell$'th block rows of size $n\times k$ is nonzero for $1\leq \ell\leq \K$ i.e.
\[
\text{$\widetilde{\ell}$th~block row of}~\Jc(\vb_\ell)=\begin{cases}0~\text{if}~\ell\neq \widetilde{\ell}\\\phi(\X\W^T)~\text{else}\end{cases}.
\]
Hence, $\Jc(\Vb)$ is $\K\times \K$ block diagonal with blocks equal to $\phi(\X\W^T)$.
The following theorem summarizes the properties of the joint Jacobian.
\begin{theorem}[Properties of the Combined Input/Output Jacobian] $\Jc(\Vb,\W)$ satisfies the following properties.
\begin{itemize}
\item {\bf{Upper bound:}} $\|\Jc(\Vb,\W)\|\leq B\|\X\|(\tf{\W}+\sqrt{\K k}\infnorm{\Vb})$.
\item {\bf{Row-bound:}} For unit length $\ub$: $\trow{\text{mat}\left(\mathcal{J}^T(\mtx{W})\vct{u}\right)}\leq B\sqrt{\K}\tin{\Vb}\|\X\|$.
\item {\bf{Entry-bound:}} For unit length $\ub$: $\tin{\text{mat}\left(\mathcal{J}^T(\Vb)\vct{u}\right)}\leq B\trow{\W}\|\X\|$.
\item {\bf{Lipschitzness:}} Given inputs $\Vb,\Vb'$ and outputs $\W,\W'$
\[
\|\Jc(\Vb,\W)-\Jc(\Vb',\W')\|\leq B\opnorm{\X}(\sqrt{\K k}\infnorm{\Vb-\Vb'}+\sqrt{\K}\infnorm{\Vb}\fronorm{{\mtx{W}'}-\mtx{W}}+\tf{\W-\W'}).
\]
\end{itemize}
\end{theorem}
\begin{proof} First, we prove results concerning $\Jc(\Vb)$. First, note that
\[
\|\Jc(\Vb)\|\leq \|\phi(\X\W^T)\|\leq B\|\X\|\tf{\W}.
\]
Next, note that for $\ub=[\ub_1~\dots~\ub_\K]\in\R^{\K n}$ we have
\begin{align*}
\tin{\Jc^T(\Vb)\ub}&=\max_{1\leq \ell\leq \K}\tin{\phi(\W\X^T)\ub_\ell}\\
&=\max_{1\leq s\leq k}|\phi(\w_s\X^T)\ub_\ell|\\
&=B\trow{\W}\|\X\|.
\end{align*}
Let $\Jc_1,\Jc_2$ be the Jacobian matrices restricted to $\Vb$ and $\W$ of $\Jc(\Vb,\W)$. To prove Lipschitzness, first observe that
\[
\|\Jc(\Vb,\W)-\Jc(\Vb',\W')\|\leq \|\Jc_1(\Vb,\W)-\Jc_1(\Vb',\W')\|+ \|\Jc_2(\Vb,\W)-\Jc_2(\Vb',\W')\|.
\]
Next, observe that
\[
 \|\Jc_1(\Vb,\W)-\Jc_1(\Vb',\W')\|\leq \|\phi(\X\W^T)-\phi(\X\W'^T)\|\leq B\|\X\|\tf{\W-\W'}.
\]
We decompose $\Jc_2$ via
\begin{align*}
 \|\Jc_2(\Vb,\W)-\Jc_2(\Vb',\W')\|&\leq \|\Jc_2(\Vb,\W)-\Jc_2(\Vb,\W')\|+ \|\Jc_2(\Vb,\W')-\Jc_2(\Vb',\W')\|\\
 &\leq B\sqrt{\K}\infnorm{\Vb}\opnorm{\mtx{X}}\fronorm{{\mtx{W}'}-\mtx{W}}+\|\Jc_2(\Vb,\W')-\Jc_2(\Vb',\W')\|.
\end{align*}
To address the second term, note that, Jacobian is linear with respect to output layer hence
\[
\|\Jc_2(\Vb,\W')-\Jc_2(\Vb',\W')\|=\|\Jc_2(\Vb-\Vb',\W')\|\leq B\sqrt{\K k}\infnorm{\Vb-\Vb'}\opnorm{\X}.
\]
Combining the latter two identities we arrive at
\[
 \|\Jc_2(\Vb,\W)-\Jc_2(\Vb',\W')\|\leq B\opnorm{\X}(\sqrt{\K k}\infnorm{\Vb-\Vb'}+\sqrt{\K}\infnorm{\Vb}\fronorm{{\mtx{W}'}-\mtx{W}}),
\]
completing the proof.
\end{proof}

\end{document}